\documentclass[11pt]{article}

\usepackage{amsfonts}
\usepackage{amsmath}
\usepackage{graphicx}
\usepackage{subfigure}
\usepackage{color}
\usepackage{vmargin}
\usepackage[percent]{overpic}
\usepackage{rotating}
\newtheorem{claimenv}{Lemma}
\newenvironment{proof}{\noindent {\em Proof:}}{\\\hspace*{\fill}\mbox{$\diamond$}}
\usepackage{algorithm}
\usepackage{algorithmic}
\usepackage{bm}
\usepackage{amsmath, amssymb, amsbsy}
\usepackage{contour}
\usepackage[usenames,dvipsnames]{xcolor}
\definecolor{light-gray}{gray}{0.75}

\setpapersize{USletter}
\setmarginsrb{1in}{1.5in}{1in}{0.5in}{0in}{0in}{11pt}{0.4in} 
\newcommand{\gs}[1]{#1_small.pdf}

\newlength{\defbaselineskip}
\contournumber{5}

\begin{document}


\title{Semi-supervised Eigenvectors for 
 Large-scale \\  
Locally-biased Learning%
\footnote{A preliminary version of parts of this paper appeared in the
Proceedings of the 2012 NIPS Conference~\cite{HM12}.}
}

\author{
Toke J. Hansen
\thanks{
Department of Applied Mathematics and Computer Science,
Technical University of Denmark,
{\tt tjha@imm.dtu.dk}.
}
\and
Michael W. Mahoney
\thanks{
Department of Mathematics,
Stanford University.  
{\tt mmahoney@cs.stanford.edu}.
}
}

\date{}
\maketitle


\begin{abstract}
In many applications, one has side information, \emph{e.g.}, labels that are 
provided in a semi-supervised manner, about a specific target region of a 
large data set, and one wants to perform machine learning and data analysis 
tasks ``nearby'' that prespecified target region.  
For example, one might be interested in the clustering structure of a data 
graph near a prespecified ``seed set'' of nodes, or one might be interested
in finding partitions in an image that are near a prespecified ``ground 
truth'' set of pixels.
Locally-biased problems of this sort are particularly challenging for 
popular eigenvector-based machine learning and data analysis tools.
At root, the reason is that eigenvectors are inherently global quantities, 
thus limiting the applicability of eigenvector-based methods in situations 
where one is interested in very local properties of the~data.

In this paper, we address this issue by providing a methodology to construct 
\emph{semi-supervised eigenvectors} of a graph Laplacian, and we illustrate 
how these locally-biased eigenvectors can be used to perform 
\emph{locally-biased machine learning}.
These semi-supervised eigenvectors capture successively-orthogonalized 
directions of maximum variance, conditioned on being well-correlated with an 
input seed set of nodes that is assumed to be provided in a semi-supervised 
manner.
We show that these semi-supervised eigenvectors can be computed quickly as 
the solution to a system of linear equations; and we also describe several 
variants of our basic method that have improved scaling properties.
We provide several empirical examples demonstrating how these 
semi-supervised eigenvectors can be used to perform locally-biased learning; 
and we discuss the relationship between our results and recent machine 
learning algorithms that use global eigenvectors of the graph Laplacian.
\end{abstract}

\section{Introduction}
\label{sxn:intro}

In many applications, one has information about a specific target region of 
a large data set, and one wants to perform common machine learning and data 
analysis tasks ``nearby'' the pre-specified target region.  
In such situations, eigenvector-based methods such as those that have been 
popular in machine learning in recent years tend to have serious 
difficulties.
At root, the reason is that eigenvectors, \emph{e.g.}, of Laplacian matrices 
of data graphs, are inherently \emph{global} quantities, and thus they might 
not be sensitive to very \emph{local} information.
Motivated by this, we consider the problem of finding a set of 
locally-biased vectors---we will call them \emph{semi-supervised 
eigenvectors}---that inherit many of the ``nice'' properties that the 
leading nontrivial global eigenvectors of a graph Laplacian have---for 
example, that capture ``slowly varying'' modes in the data, that are 
fairly-efficiently computable, that can be used for common machine learning 
and data analysis tasks such as kernel-based and semi-supervised learning, 
etc.---so that we can perform what we will call \emph{locally-biased machine 
learning} in a principled manner.

\subsection{Locally-biased Learning}

By \emph{locally-biased machine learning}, we mean that we have a data set, 
\emph{e.g.}, represented as a graph, and that we have information, 
\emph{e.g.}, given in a semi-supervised manner, that certain ``regions'' of 
the data graph are of particular interest.
In this case, we may want to focus predominantly on those regions and 
perform data analysis and machine learning, \emph{e.g.}, classification, 
clustering, ranking, etc., that is ``biased toward'' those pre-specified 
regions.
Examples of this include the~following.
\begin{itemize}
\item
\emph{Locally-biased community identification.}
In social and information network analysis, one might have a small ``seed
set'' of nodes that belong to a cluster or community of 
interest~\cite{andersen06seed,LLDM08_communities_CONF}; 
in this case, one might want to perform link or edge prediction, or one 
might want to ``refine'' the seed set in order to find other nearby members.
\item
\emph{Locally-biased image segmentation.}
In computer vision, one might have a large corpus of images along with a 
``ground truth'' set of pixels as provided by a face detection 
algorithm~\cite{EOK07,MOV12-JMLR,MVM11};
in this case, one might want to segment entire heads from the background for 
all the images in the corpus in an automated manner.
\item
\emph{Locally-biased neural connectivity analysis.}
In functional magnetic resonance imaging applications, one might have small 
sets of neurons that ``fire'' in response to some external experimental 
stimulus~\cite{NPDH06}; 
in this case, one might want to analyze the subsequent temporal dynamics of 
stimulation of neurons that are ``nearby,'' either in terms of connectivity 
topology or functional response, members of the original set.
\end{itemize}
In each of these examples, the data are modeled by a graph---which is either
``given'' from the application domain or is ``constructed'' from feature 
vectors obtained from the application domain---and one has information that
can be viewed as semi-supervised in the sense that it consists of 
exogeneously-specified ``labels'' for the nodes of the graph.
In addition, there are typically a relatively-small number of labels and one
is interested in obtaining insight about the data graph nearby those labels.

These examples present considerable challenges for standard global spectral 
techniques and other traditional eigenvector-based methods.
(Such eigenvector-based methods have received attention in a wide range of 
machine learning and data analysis applications in recent years.
They have been useful, for example, 
in non-linear dimensionality reduction~\cite{BN03,CLLMNWZ05a};
in kernel-based machine learning~\cite{SS01-book};
in Nystr\"{o}m-based learning methods~\cite{WS01,TalRos10}; 
spectral partitioning~\cite{pothen90partition,ShiMalik00_NCut,NJW01_spectral}, 
and so on.)
At root, the reason is that eigenvectors are inherently global quantities, 
thus limiting their applicability in situations where one is interested in 
very local properties of the data.  
That is, very local information can be ``washed out'' and essentially 
invisible to these globally-optimal vectors.
For example, a sparse cut in a graph may be poorly correlated with the 
second eigenvector and thus invisible to a method based only on eigenvector 
analysis. 
Similarly, if one has semi-supervised information about a specific target 
region in the graph, as in the above examples, one might be interested in 
finding clusters near this prespecified local region in a semi-supervised 
manner; but this local region might be essentially invisible to a method 
that uses only global eigenvectors.
Finally, one might be interested in using kernel-based methods to find 
``local correlations'' or to regularize with respect to a ``local 
dimensionality'' in the data, but this local information might be destroyed
in the process of constructing kernels with traditional kernel-based methods.

\subsection{Semi-supervised Eigenvectors}

In this paper, we provide a methodology to construct what we will call
\emph{semi-supervised eigenvectors} of a graph Laplacian; and we illustrate 
how these locally-biased eigenvectors (locally-biased in the sense that they 
will be well-correlated with the input seed set of nodes or that most of 
their ``mass'' will be on nodes that are ``near'' that seed set) inherit 
many of the properties that make the leading nontrivial global eigenvectors 
of the graph Laplacian so useful in applications.
In order to make this method useful, there should ideally be a ``knob'' that 
allows us to interpolate between very local and the usual global 
eigenvectors, depending on the application at hand; 
we should be able to use these vectors in common machine learning pipelines 
to perform common machine learning tasks; and 
the intuitions that make the leading $k$ nontrivial global eigenvectors of 
the graph Laplacian useful should, to the extent possible, extend to the
locally-biased setting.
To achieve this, we will formulate an optimization ansatz that is a variant 
of the usual global spectral graph partitioning optimization problem that 
includes a natural locality constraint as well as an orthogonality 
constraint, and we will iteratively solve this problem.

In more detail, assume that we are given as input a (possibly weighted) data 
graph $G=(V,E)$, an indicator vector $s$ of a small ``seed set'' of nodes, a 
\emph{correlation parameter} $\kappa \in [0,1]$, and a positive integer $k$.
Then, informally, we would like to construct $k$ vectors that satisfy the 
following bicriteria:
first, each of these $k$ vectors is well-correlated with the input seed set;
and 
second, those $k$ vectors describe successively-orthogonalized directions of
maximum variance, in a manner analogous to the leading $k$ nontrivial global 
eigenvectors of the graph Laplacian.
(We emphasize that the seed set $s$ of nodes, the integer $k$, and the 
correlation parameter $\kappa$ are part of the input; and thus they should 
be thought of as being available in a semi-supervised manner.)
Somewhat more formally, our main algorithm, Algorithm~\ref{alg_new_1} in 
Section~\ref{sxn:main-alg}, returns as output $k$ semi-supervised 
eigenvectors; each of these is the solution to an optimization problem of
the form of \textsc{Generalized LocalSpectral} in Figure~\ref{fig:objective},
and thus each ``captures'' (say) $\kappa/k$ of the correlation with the seed 
set.
Our main theoretical result, described in Section~\ref{sxn:main-alg}, states 
that these vectors define successively-orthogonalized directions of maximum
variance, conditioned on being $\kappa/k$-well-correlated with an input seed 
set $s$; and that each of these $k$ semi-supervised eigenvectors can be 
computed quickly as the solution to a system of linear equations.
To extend the practical applicability of this basic result, we will in 
Section~\ref{sxn:main-alg-extensions} describe several heuristic extensions 
of our basic framework that will make it easier to apply the method of 
semi-supervised eigenvectors at larger size scales.
These extensions involve using the so-called Nystr{\"{o}}m method, 
computing one locally-biased eigenvector and iteratively ``peelling off'' 
successive components of interest, as well as performing random walks that 
are ``local'' in a stronger sense than our basic method considers.

Finally, in order to illustrate how the method of semi-supervised 
eigenvectors performs in practice, we also provide a detailed empirical 
evaluation using a wide range of both small-scale as well as larger-scale 
data.
In particular, we consider two small data sets, one consisting of graphs 
generated from a popular network generation model and the other data drawn 
from Congressional roll call voting patterns, in order to illustrate the 
basic method; we consider graphs constructed from the widely-studied MNIST 
digit data, in order to illustrate how the method performs on a data set 
that is widely-known in the machine learning community; and we consider 
two larger data sets, one consisting of Internet graphs and the other 
consisting of graphs constructed from fMRI medical imaging, in order to 
illustrate how the method performs in larger-scale applications.

\subsection{Related Work} 

From a technical perspective, the work most closely related to ours is the 
recently-developed ``local spectral method'' of 
Mahoney \emph{et al.}~\cite{MOV12-JMLR}.
The original algorithm of Mahoney \emph{et al.}~\cite{MOV12-JMLR} introduced 
a methodology to construct a locally-biased version of the 
\emph{leading} nontrivial eigenvector of a graph Laplacian and also showed 
(theoretically and empirically in a social network analysis application) 
that that the resulting vector could be used to partition a graph in a 
locally-biased manner.
From this perspective, our extension incorporates a natural orthogonality 
constraint that successive vectors need to be orthogonal to previous vectors.
Subsequent to the work of \cite{MOV12-JMLR}, \cite{MVM11} applied the 
algorithm of \cite{MOV12-JMLR} to to the problem of finding locally-biased 
cuts in a computer vision application.
Similar ideas have also been applied somewhat differently.
For example, \cite{andersen06seed} use locally-biased random walks, 
\emph{e.g.}, short random walks starting from a small seed set of nodes, to 
find clusters and communities in graphs arising in Internet advertising 
applications; 
\cite{LLDM08_communities_CONF} used locally-biased random walks to 
characterize the local and global clustering structure of a wide range of 
social and information networks; and
\cite{Joa03} developed the Spectral Graph Transducer, which performs 
transductive learning via spectral graph partitioning.

The objectives in both~\cite{Joa03} and~\cite{MOV12-JMLR} are constrained 
eigenvalue problems that can be solved by finding the smallest eigenvalue 
of an asymmetric generalized eigenvalue problem; but in practice this 
procedure can be highly unstable~\cite{Gander1989}. 
The algorithm of~\cite{Joa03} reduces the instabilities by performing all 
calculations in a subspace spanned by the $d$ smallest eigenvectors of the 
graph Laplacian; 
whereas the algorithm of~\cite{MOV12-JMLR} performs a binary search, 
exploiting the monotonic relationship between a control parameter and the 
corresponding Lagrange multiplier.
The form of our optimization problem also has similarities to other work in 
computer vision applications: \emph{e.g.},~\cite{YS01} and~\cite{EOK07} find 
good conductance clusters subject to a set of linear constraints.

In parallel, \cite{BN03} and a large body of subsequent work including 
\cite{CLLMNWZ05a} used (the usual global) eigenvectors of the graph 
Laplacian to perform dimensionality reduction and data representation, in 
unsupervised and semi-supervised settings~\cite{TSL00,RS00,ZBLWS04}.
Typically, these methods construct some sort of local neighborhood 
structure around each data point, and they optimize some sort of global 
objective function to go ``from local to global''~\cite{SWHSL06}.
In some cases, these methods can be understood in terms of data drawn from 
an hypothesized manifold~\cite{BN08}, and more generally they have proven 
useful for denoising and learning in semi-supervised 
settings~\cite{BN04,BNS06}.
These methods are based on spectral graph theory~\cite{Chung:1997}; 
and thus many of these methods have a natural interpretation in terms of 
diffusions and kernel-based 
learning~\cite{SS01-book,KL02,SJ02,chapelle2002,HLMS04}.
These interpretations are important for the usefulness of these global 
eigenvector methods in a wide range of applications.
As we will see, many (but not all) of these interpretations can be ported 
to the ``local'' setting, an observation that was made previously in a 
different context~\cite{CM11_TR}.

Many of these diffusion-based spectral methods also have a natural 
interpretation in terms of spectral ranking~\cite{Vig09_TR}.
``Topic sensitive'' and ``personalized'' versions of these spectral ranking 
methods have also been studied~\cite{haveliwala03_topicpr,JW03}; and these
were the motivation for diffusion-based methods to find locally-biased
clusters in large graphs~\cite{Spielman:2004,andersen06local,MOV12-JMLR}.
Our optimization ansatz is a generalization of the linear equation 
formulation of the PageRank procedure~\cite{PBMW99,MOV12-JMLR,Vig09_TR}; and
its solution involves Laplacian-based linear equation solving, which has 
been suggested as a primitive is of more general interest in large-scale 
data analysis~\cite{Teng10}.

\subsection{Outline of the Paper}

In the next section, Section~\ref{sxn:background}, we will provide notation 
and some background and discuss related work.
Then, in Section~\ref{sxn:main-alg} we will present our main algorithm and
our main theoretical result justifying the algorithm; and in 
Section~\ref{sxn:main-alg-extensions} we will present several extensions of 
our basic method that will help for certain larger-scale applications of the 
method of semi-supervised eigenvectors.
In Section~\ref{sxn:empirical}, we present an empirical analysis, including
both toy data to illustrate how the ``knobs'' of our method work, as well as
applications to realistic machine learning and data analysis problems; 
and in Section~\ref{sxn:conclusion} we present a brief discussion and 
conclustion.

\section{Background and Notation}
\label{sxn:background}

\newcommand{\vol}{\mathrm{vol}}
\newcommand{\defeq}{\stackrel{\textup{def}}{=}}

Let $G=(V,E,w)$ be a connected undirected graph with $n=|V|$ vertices and
$m=|E|$ edges, in which edge $\{i,j\}$ has weight $w_{ij}.$
For a set of vertices $S \subseteq V$ in a graph, the \emph{volume of $S$}
is $\vol(S) \defeq \sum_{i \in S}d_i$, in which case the \emph{volume of the
graph $G$} is $\vol(G) \defeq \vol(V)=2m$.
In the following, $A_G \in \mathbb{R}^{V \times V}$ will denote the
adjacency matrix of $G$, while $D_G \in \mathbb{R}^{V \times V}$ will denote
the diagonal degree matrix of $G$, \emph{i.e.},
$D_G(i,i)=d_i = \sum_{\{i,j\} \in E} w_{ij}$, the weighted degree of vertex $i$.
The Laplacian of $G$ is defined as $L_G \defeq D_G-A_G$.
(This is also called the combinatorial Laplacian, in which case the
normalized Laplacian of $G$ is $\mathcal{L}_G\defeq D_G^{-1/2}L_GD_G^{-1/2}$.)

The Laplacian is the symmetric matrix having quadratic form
$x^T L_G x = \sum_{ij \in E} w_{ij} (x_i - x_j)^2$, for $x \in  \mathbb{R}^V$.
This implies that $L_G$ is positive semidefinite and that the all-one vector
$1 \in  \mathbb{R}^V$ is the eigenvector corresponding to the smallest
eigenvalue $0$.
The generalized eigenvalues of $ L_G x = \lambda_i D_G x$ are 
$0=\lambda_1 < \lambda_2 \leq \cdots \leq \lambda_N$.
We will use $v_2$ to denote smallest non-trivial eigenvector, \emph{i.e.},
the eigenvector corresponding to $\lambda_2$; $v_3$ to denote the next
eigenvector; and so on.
%
We 
will overload notation to 
use 
$\lambda_2(A)$ to denote the smallest non-zero generalized eigenvalue of 
$A$ with respect to $D_G$. 
Finally, for a matrix $A,$ let $A^{+}$ denote its (uniquely defined)
Moore-Penrose pseudoinverse.
For two vectors $x,y \in \mathbb{R}^{n}$, and the degree matrix $D_G$ for a
graph $G$, we define the \emph{degree-weighted inner product} as
$ x^T D_G y \defeq  \sum_{i=1}^{n} x_{i}y_{i}d_{i}$.
In particular, if a vector $x$ has unit norm, then $x^T D_G x = 1$.
Given a subset of vertices $S \subseteq V$, we denote by $1_S$ the indicator
vector of $S$ in $\mathbb{R}^V$ and by $1$ the vector in $\mathbb{R}^V$
having all entries set equal to $1$.

\section{Optimization Approach to Semi-supervised Eigenvectors}
\label{sxn:main-alg}

In this section, we provide our main technical results: a motivation and 
statement of our optimization ansatz; our main algorithm for computing 
semi-supervised eigenvectors; and an analysis that our algorithm computes 
solutions of our optimization ansatz.

\subsection{Motivation for the Program}

Recall the optimization perspective on how one computes the leading 
nontrivial global eigenvectors of the normalized Laplacian $\mathcal{L}_G$ 
or, equivalently, of the leading nontrivial generalized eigenvectors of 
$L_G$.
The first nontrivial eigenvector $v_2$ is the solution to the problem
\textsc{GlobalSpectral} that is presented on the left of 
Figure~\ref{fig:objective}.
Equivalently, although \textsc{GlobalSpectral} is a non-convex optimization 
problem, strong duality holds for it and it's solution may be computed as 
$v_2$, the leading nontrivial generalized eigenvector of $L_G$.
(In this case, the value of the objective is $\lambda_2$, and global 
spectral partitioning involves then doing a ``sweep cut'' over this vector 
and appealing to Cheeger's inequality.)
The next eigenvector $v_3$ is the solution to \textsc{GlobalSpectral},
augmented with the constraint that $x^TD_Gv_2=0$; and in general the 
$t^{th}$ generalized eigenvector of $L_G$ is the solution to 
\textsc{GlobalSpectral}, augmented with the constraints that $x^TD_Gv_i=0$, 
for $i\in\{2,\ldots,t-1\}$.
Clearly, this set of constraints and the constraint $x^TD_G1=0$ can be 
written as $x^TD_GX=0$, where $0$ is a $(t-1)$-dimensional all-zeros vector, 
and where $X$ is an $n \times (t-1)$ orthogonal matrix whose $i^{th}$ column 
equals $v_i$ (where $v_1=1$, the all-ones vector, is the first column of $X$).

\begin{figure*}[t]
\begin{minipage}[t]{0.3\linewidth}
\centering
\textsc{GlobalSpectral}\\
\;
\begin{align*}
\text{minimize} \quad & x^T L_G x \\
\text{s.t}\quad & x^T D_G x = 1 \\
& x^T D_G 1 = 0 
\end{align*}
\end{minipage}
\begin{minipage}[t]{0.3\linewidth}
\centering
\textsc{LocalSpectral}\\
\;
\begin{align*}
\text{minimize} \quad & x^T L_G x \\
\text{s.t}\quad & x^T D_G x = 1 \\
& x^T D_G 1 = 0 \\
& x^T D_G s \geq \sqrt{\kappa}
\end{align*}
\end{minipage}
\hspace{0.5cm}
\begin{minipage}[t]{0.3\linewidth}
\centering
\textsc{Generalized LocalSpectral}
\begin{align*}
\text{minimize} \quad & x^T L_G x \\
\text{s.t}\quad & x^T D_G x = 1 \\
& x^T D_G X = 0 \\
& x^T D_G s \geq \sqrt{\kappa}
\end{align*}
\end{minipage}
\caption{%
Left: The usual \textsc{GlobalSpectral} partitioning optimization problem; 
the vector achieving the optimal solution is $v_2$, the leading nontrivial 
generalized eigenvector of $L_G$ with respect to $D_G$.
Middle: The \textsc{LocalSpectral} optimization problem, which was 
originally introduced in~\cite{MOV12-JMLR}; for $\kappa=0$, this coincides 
with the usual global spectral objective, while for $\kappa > 0$, this 
produces solutions that are biased toward the seed vector $s$.
Right: The \textsc{Generalized LocalSpectral} optimization problem we 
introduce that includes both the locality constraint and a more general 
orthogonality constraint.
Our main algorithm for computing semi-supervised eigenvectors will 
iteratively compute the solution to \textsc{Generalized LocalSpectral} for a 
sequence of $X$ matrices.
In all three cases, the optimization variable is $x \in \mathbb{R}^{n}$.
}
\label{fig:objective}
\end{figure*}

Also presented in Figure~\ref{fig:objective} is \textsc{LocalSpectral}, 
which includes a constraint that the solution be well-correlated with an 
input seed set.
This \textsc{LocalSpectral} optimization problem was introduced 
in~\cite{MOV12-JMLR}, where it was shown that the solution to 
\textsc{LocalSpectral} may be interpreted as a locally-biased version of the 
second eigenvector of the Laplacian.%
\footnote{In \cite{MOV12-JMLR}, the locality constraint was actually a 
quadratic constraint, and thus a somewhat involved analysis was required.
In retrospect, given the form of the solution, and in light of the 
discussion below, it is clear that the quadratic part was not ``real,'' and 
thus we present this simpler form of \textsc{LocalSpectral} here.
This should make the connections with our \textsc{Generalized LocalSpectral} 
objective more immediate.}
In particular, although \textsc{LocalSpectral} is not convex, it's solution 
can be computed efficiently as the solution to a set of linear equations 
that generalize the popular Personalized PageRank procedure; 
in addition, by performing a sweep cut and appealing to a variant of 
Cheeger's inequality, this locally-biased eigenvector can be used to perform 
locally-biased spectral graph partitioning~\cite{MOV12-JMLR}.

\subsection{Our Main Algorithm}

We will formulate the problem of computing semi-supervised vectors in terms 
of a primitive optimization problem of independent interest.
Consider the \textsc{Generalized LocalSpectral} optimization problem, as 
shown in Figure~\ref{fig:objective}.
For this problem, we are given a graph $G=(V,E)$, with associated Laplacian 
matrix $L_G$ and diagonal degree matrix $D_G$; 
an indicator vector $s$ of a small ``seed set'' of nodes; 
a \emph{correlation parameter} $\kappa \in [0,1]$; and
an $n \times \nu$ constraint matrix $X$ that may be assumed to be an 
orthogonal matrix.
We will assume (without loss of generality) that $s$ is properly normalized
and orthogonalized so that $s^T D_{G} s =1$ and $s^T D_{G} 1 =0$.
While $s$ can be a general unit vector orthogonal to $1$, it may be helpful
to think of $s$ as the indicator vector of one or more vertices in $V$,
corresponding to the target region of the graph.

In words, the problem \textsc{Generalized LocalSpectral} asks us to find a 
vector $x \in \mathbb{R}^{n}$ that minimizes the variance $x^TL_Gx$ subject 
to several constraints: that $x$ is unit length; that $x$ is orthogonal to 
the span of $X$; and that $x$ is $\sqrt{\kappa}$-well-correlated with 
the input seed set vector $s$.
In our application of \textsc{Generalized LocalSpectral} to the computation 
of semi-supervised eigenvectors, we will iteratively compute the solution 
to \textsc{Generalized LocalSpectral}, updating $X$ to contain the 
already-computed semi-supervised eigenvectors.
That is, to compute the first semi-supervised eigenvector, we let $X=1$, 
\emph{i.e.}, the $n$-dimensional all-ones vector, which is the trivial 
eigenvector $L_G$, in which case $X$ is an $n \times 1$ matrix; and
to compute each subsequent semi-supervised eigenvector, we let the columns 
of $X$ consist of $1$ and the other semi-supervised eigenvectors found in 
each of the previous iterations.

To show that \textsc{Generalized LocalSpectral} is efficiently-solvable, 
note that it is a quadratic program with only one quadratic constraint and
one linear equality constraint.%
\footnote{Alternatively, note that it is an example of an constrained 
eigenvalue problem~\cite{Gander1989}.
We thank the numerous individuals who pointed this out to us subsequent to 
our dissemination of the original version of this paper.}
In order to remove the equality constraint, which will simplify the problem, 
let's change variables by defining the $n \times (n-\nu)$ matrix $F$ as
$$
\{ x: X^TD_Gx = 0 \} = \{ x: x=F\hat x \} .
$$
That is, $F$ is a span for the null space of $X^T$; and we will take $F$ to 
be an orthogonal matrix.
In particular, this implies that $F^TF$ is an $(n-\nu)\times(n-\nu)$ 
Identity and $FF^T$ is an $n \times n$ Projection.
Then, with respect to the $\hat x$ variable, \textsc{Generalized LocalSpectral} 
becomes 
\begin{equation}
\begin{aligned}
  & \underset{y}{\text{minimize}}
  & & \hat x^T F^T L_G F y \\ 
  & \text{subject to}
  & & \hat x^T F^T D_G F \hat x = 1, \\
  & & & \hat x^T F^T D_G s \geq \sqrt{\kappa}  .
\end{aligned}
\label{eqn:toke-vp-2}
\end{equation}
In terms of the variable $x$, the solution to this optimization problem is 
of the form 
\begin{eqnarray}
\nonumber
x^{*} &=& c F \left( F^T \left( L_G - \gamma D_G \right) F \right)^{+} F^T D_Gs  \\
      &=& c \left( FF^T \left( L_G - \gamma D_G \right) FF^T \right)^{+} D_Gs   ,
\label{eqn:xstar}
\end{eqnarray}
for a normalization constant $c \in (0,\infty)$ and for some $\gamma$ that 
depends on $\sqrt{\kappa}$.
The second line follows from the first since $F$ is an $n \times (n-\nu)$ 
orthogonal matrix.
This so-called ``S-proceudre'' is described in greater detail in 
Chapter 5 and Appendix B of~\cite{Boyd04}.
The significance of this is that, although it is a non-convex optimization 
problem, the \textsc{Generalized LocalSpectral} problem can be solved by 
solving a linear equation, in the form given in Eqn.~(\ref{eqn:xstar}).

Returning to our problem of computing semi-supervised eigenvectors, recall 
that, in addition to the input for the \textsc{Generalized LocalSpectral} 
problem, we need to specify a positive integer $k$ that indicates the number 
of vectors to be computed.
In the simplest case, we would assume that we would like the correlation to 
be ``evenly distributed'' across all $k$ vectors, in which case we will 
require that each vector is $\sqrt{\kappa/k}$-well-correlated with the input 
seed set vector $s$; but this assumption can easily be relaxed, and thus  
Algorithm~\ref{alg_new_1} is formulated more generally as taking a 
$k$-dimensional vector $\kappa = [\kappa_1,\ldots, \kappa_k]^T$ of 
correlation coefficients as input.

To compute the first semi-supervised eigenvector, we will let $X=1$, the 
all-ones vector, in which case the first nontrivial semi-supervised 
eigenvector is
\begin{equation}
x_1^{*} = c \left( L_G - \gamma_1 D_G \right)^{+} D_Gs   ,
\label{eqn:first-step}
\end{equation}
where $\gamma_1$ is chosen to saturate the part of the correlation constraint
along the first direction. 
(Note that the projections $FF^T$ from Eqn.~(\ref{eqn:xstar}) are not 
present in Eqn.~(\ref{eqn:first-step}) since by design $s^TD_G1=0$.)
That is, to find the correct setting of $\gamma_1$, it suffices to perform a 
binary search over the possible values of $\gamma_1$ in the interval 
$(-\vol(G), \lambda_2(G))$ until the correlation constraint is satisfied, 
that is, until $(s^T D_{G} x_1)^2$ is sufficiently close to $\kappa_1$.

To compute subsequent semi-supervised eigenvectors, \emph{i.e.}, at steps 
$t = 2,\ldots, k$ if one ultimately wants a total of $k$ semi-supervised 
eigenvectors, then one lets $X$ be the $n \times t$ matrix of the form
\begin{align}
X = [1 , x^{*}_1 , \ldots , x^{*}_{t-1} ] ,
\end{align}
where $ x^{*}_1 , \ldots ,  x^{*}_{t-1}$ are successive semi-supervised eigenvectors; and the 
projection matrix $FF^T$ is of the form 
$$
FF^T = I - D_GX(X^T D_G D_G X)^{-1}X^T D_G,
$$
due to the the degree-weighted inner norm.

Then, by Eqn.~(\ref{eqn:xstar}), the $t^{th}$ semi-supervised 
eigenvector takes the form
\begin{eqnarray}
x_{t}^{*} 
    &=& c \left( FF^T (L_G-\gamma_tD_G) FF^T \right)^{+} D_Gs  . \notag 
\end{eqnarray}

\begin{algorithm}                      
\caption{Main algorithm to compute semi-supervised eigenvectors}          
\label{alg_new_1}                           
\begin{algorithmic}[1]                    
\REQUIRE $L_G, D_G, s, \kappa = [\kappa_1,\ldots, \kappa_k]^T,\epsilon$  
such that $s^T D_G 1 = 0$, $s^T D_G s = 1$, $\kappa^T 1 \leq 1$
\STATE $X=[1]$ 
\FOR{$t= 1$ to $k$}
\STATE $FF^T \leftarrow I - D_GX(X^T D_G D_G X)^{-1}X^T D_G$
\STATE $\top \leftarrow \lambda_2$ where $FF^T L_G FF^T v_2 = \lambda_2 FF^T D_G FF^T v_2$
\STATE $\bot \leftarrow -\text{vol}(G)$
\REPEAT  
\STATE $\gamma_t \leftarrow (\bot+\top)/2$ (Binary search over $\gamma_t$)
\STATE $x_t \leftarrow (FF^T( L_G - \gamma_t D_G) FF^T)^{+} F F^T D_G s$
\STATE Normalize $x_t$ such that $x_t^T D_G x_t=1$
\STATE \textbf{if} $( x_t^T D_G s)^2>\kappa_t$ \textbf{then} $\bot \leftarrow \gamma_t$ \textbf{else}  $\top \leftarrow \gamma_t$ \textbf{end if}
\UNTIL{$\|( x_t^T D_G s)^2-\kappa_t\| \leq \epsilon$ \textbf{or} $\|(\bot+\top)/2 -\gamma_t \| \leq \epsilon$}
\STATE Augment $X$ with $x_t^*$ by letting $X = [X,x_t^*]$.
\ENDFOR
\end{algorithmic}
\end{algorithm}


In more detail, Algorithm~\ref{alg_new_1} presents pseudo-code for our main 
algorithm for computing semi-supervised eigenvectors.
The algorithm takes as input a graph $G=(V,E)$, a seed set $s$ (which could 
be a general vector $s\in\mathbb{R}^{n}$, subject for simplicity to the
normalization constraints $s^T D_G 1 = 0$ and $s^T D_G s = 1$, but which is 
most easily thought of as an indicator vector for the local ``seed set'' of 
nodes), a number $k$ of vectors we want to compute, and a vector of locality 
parameters $(\kappa_1,\ldots,\kappa_k)$, where $\kappa_i \in [0,1]$ and 
$\sum_{i=1}^{k} \kappa_i = 1$ (where, in the simplest case, one could 
choose $\kappa_i = \kappa/k$, $\forall i$, for some $\kappa \in [0,1]$).
Several things should be noted about our implementation of our main 
algorithm.
First, as we will discuss in more detail below, we compute the projection 
matrix $FF^T$ only \emph{implicitly}.  
Second, a na\"{i}ve approach to Eqn.~(\ref{eqn:xstar}) does not immediately 
lead to an efficient solution, since $ D_G s$ will not be in the span of 
$( F F^T( L_G - \gamma D_G) F F^T)$, thus leading to a large residual. 
By changing variables so that $ x = F F^T y$, the solution becomes 
$$
x_t^{*} \propto FF^T (FF^T( L_G - \gamma_t D_G) FF^T)^{+} FF^T D_G s . 
$$
Since $FF^T$ is a projection matrix, this expression is equivalent to 
\begin{align}
x_t^{*} \propto \left (FF^T( L_G - \gamma_t D_G) FF^T \right)^{+} FF^T D_G s \label{eq:semisupeigs} .
\end{align}
Third, regarding the solution $x_i$, we exploit that 
$FF^T( L_G - \gamma_i D_G)FF^T$ is an SPSD matrix, and we apply the conjugate 
gradient method, rather than computing the explicit pseudoinverse. 
That is, in the implementation we never explicitly represent the dense 
matrix $FF^T$, but instead we treat it as an operator and we 
simply evaluate the result of applying a vector to it on either side. 
Fourth, we use that $\lambda_2$ can never decrease (here we refer to 
$\lambda_2$ as the smallest non-zero eigenvalue of the modified matrix), 
so we only recalculate the upper bound for the binary search when an 
iteration saturates without satisfying 
$\|( x_t^T D_G s)^2-\kappa_t\| \leq \epsilon$. Estimating the bound is critical for the semi-supervised eigenvectors to be able to interpolate all the way to the global eigenvectors of the graph, so in Section \ref{sec:bounding} we return to a discussion on efficient strategies for computing the leading nontrivial eigenvalue of $L_G$ projected down onto the space perpendicular to the previously computed solutions.

From this discussion, it should be clear that Algorithm~\ref{alg_new_1} 
solves the semi-supervised eigenvector problem by solving in an iterative 
manner optimization problems of the form of 
\textsc{Generalized LocalSpectral}; and
that the running time of Algorithm~\ref{alg_new_1} boils down to solving 
a sequence of linear equations.

\subsection{Discussion of Our Main Algorithm}\label{sec:discussion}

There is a natural ``regularization'' interpretation underlying our 
construction of semi-supervised eigenvectors.
To see this, recall that the first step of our algorithm can be computed as 
the solution of a set of linear equations 
\begin{equation}
x^{*} = c \left( L_G - \gamma D_G \right)^{+} D_Gs   ,
\label{eq:mahoney}
\end{equation}
for some normalization constant $c$ and some $\gamma$ that can be 
determined by a binary search over $(-\vol(G), \lambda_2(G))$; and that
subsequent steps compute the analogous quantity, subject to 
additional constraints that the solution be orthogonal to the 
previously-computed vectors.
The quantity $\left( L_G - \gamma D_G \right)^{+}$ can be interpreted as a 
``regularized'' version of the pseudoinverse of $L$, where 
$\gamma\in(-\infty,\lambda_2(G))$ serves as the regularization parameter.
This interpretation has recently been made precise:
\cite{MO11-implementing} show that running a PageRank computation---as well 
as running other diffusion-based procedures---\emph{exactly} optimizes a
regularized version of the \textsc{GlobalSpectral} (or 
\textsc{LocalSpectral}, depending on the input seed vector) problem; 
and \cite{PM11} provide a precise statistical framework justifying this.

The usual interpretation of PageRank involves ``random walkers'' who 
uniformly (or non-uniformly, in the case of Personalized PageRank) 
``teleport'' with a probability $\alpha\in(0,1)$.
As described in~\cite{MOV12-JMLR},  choosing $\alpha\in(0,1)$ corresponds to choosing 
$\gamma \in (-\infty,0)$.
By rearranging Eqn.~(\ref{eq:mahoney}) as 
\begin{eqnarray*}
x^{*} &=& c \left( (D_G-A_G) - \gamma D_G \right)^{+} D_Gs   \\
 &=& \frac{c}{1- \gamma} \left( D_G- \frac{1}{1- \gamma} A_G \right)^{+} D_Gs   \\
 &=& \frac{c}{1- \gamma} D_G^{-1} \left( I- \frac{1}{1- \gamma} A_G D_G^{-1}  \right)^{+} D_Gs,  
\end{eqnarray*}
we recognize $A_GD_G^{-1} $ as the standard random walk matrix, and it becomes immediate that the solution based on random walkers,
\begin{eqnarray*}
x^{*} = \frac{c}{1- \gamma} D_G^{-1} \left ( I+ \sum_{i=1}^\infty \left ( \frac{1}{1-\gamma} D_G^{-1}A_G \right)^i \right) D_G s  ,
\end{eqnarray*}
is divergent for $\gamma>0$.
Since $\gamma=\lambda_2(G)$ corresponds to no regularization and
$\gamma\rightarrow-\infty$ corresponds to heavy regularization, viewing this 
problem in terms of solving a linear equation is formally more powerful than 
viewing it in terms of random walkers.
That is, while all possible values of the regularization parameter---and in 
particular the (positive) value $\lambda_2(\cdot)$---are achievable 
algorithmically by solving a linear equation, only values in $(-\infty,0)$ 
are achievable by running a PageRank~diffusion.
In particular, if the optimal value of $\gamma$ that saturates the 
$\kappa$-dependent locality constraint is negative, then running the 
PageRank diffusion could find it; otherwise, the ``best'' one could do will 
still not saturate the locality constraint, in which case some of the 
intended correlation is ``unused.''

An important technical and practical point has to do with the precise manner 
in which the $i^{th}$ vector is well-correlated with the seed set $s$.
In our formulation, this is captured by a \emph{locality parameter} 
$\gamma_i$ that is chosen (via a binary search) to ``saturate'' the 
correlation condition, \emph{i.e.}, so that the $i^{th}$ vector is 
$\kappa/k$-well-correlated with the input seed set.
As a general rule, successive $\gamma_i$s need to be chosen that successive 
vectors are \emph{less} well-localized around the input seed set.
(Alternatively, depending on the application, one could choose this parameter 
so that successive $\gamma_i$s are equal; but this will involve 
``sacrificing'' some amount of the $\kappa/k$ correlation, which will lead 
to the last or last few eigenvectors being very poorly-correlated with the 
input seed set.
These tradeoffs will be described in more detail below.)
Informally, if $s$ is a seed set consisting of a small number of nodes that
are ``nearby'' each other, then to maintain a given amount of correlation, 
we must ``view'' the graph over larger and larger size scales as we compute
more and more semi-supervised eigenvectors.
More formally, we need to let the value of the regularization parameter 
$\gamma$ at the $i^{th}$ round, we call it $\gamma_i$, vary for each 
$i\in\{1,\ldots,k\}$.
That is, $\gamma_i$ is not pre-specified, but it is chosen via a binary 
search over the region $( -\vol(G) ,\lambda_2(\cdot))$, where 
$\lambda_2(\cdot)$ is the leading nontrivial eigenvalue of $L_G$ projected 
down onto the space perpendicular to the previously-computed vectors (which
is in general larger than $\lambda_2(G)$).
In this sense, our semi-supervised eigenvectors are both 
``locally-biased'', in a manner that depends on the input seed vector and
correlation parameter, and ``regularized'', in a manner that depends on the 
local graph structure.

 \begin{figure}[!hbt]
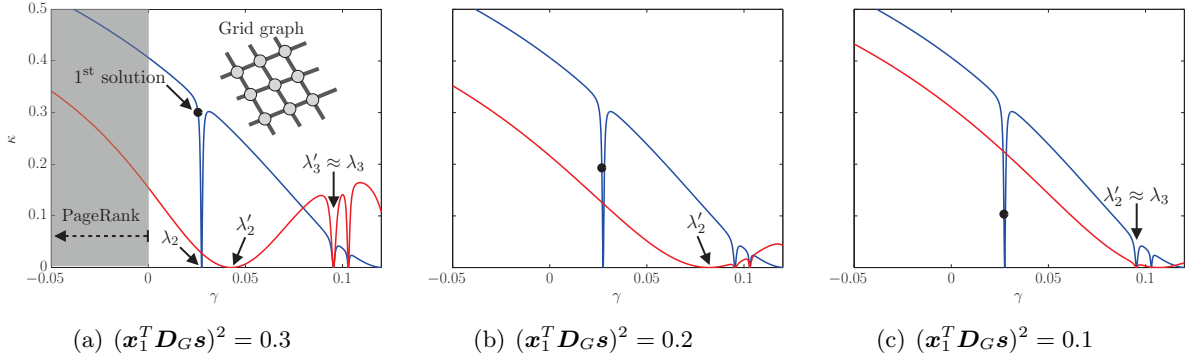

  \centering
\subfigure[$(\bm x_1^T \bm D_{G} \bm s)^2=0.3$]{
\begin{minipage}[b]{0.3\linewidth}
\centering
  \includegraphics[width=5.7cm]{\gs{sseigs_conceptual_a}}\\
\end{minipage}
\label{fig:sseigs_conceptualA}
}
\subfigure[$(\bm x_1^T \bm D_{G} \bm s)^2=0.2$]{
\begin{minipage}[b]{0.3\linewidth}
\centering
  \includegraphics[width=5.7cm]{\gs{sseigs_conceptual_b}}\\
\end{minipage}
\label{fig:sseigs_conceptualB}
}
\subfigure[$(\bm x_1^T \bm D_{G} \bm s)^2=0.1$]{
\begin{minipage}[b]{0.3\linewidth}
\centering
  \includegraphics[width=5.7cm]{\gs{sseigs_conceptual_c}}\\
\end{minipage}
\label{fig:sseigs_conceptualC}
}
  \caption{
Interplay between the $\gamma$ parameter and the correlation $\kappa$ that
a semi-supervised eigenvector has with a seed $\bm s$ on a two-dimensional
grid. In Figure \ref{fig:sseigs_conceptualA}-\ref{fig:sseigs_conceptualC},
we vary the locality parameter for the leading semi-supervised 
eigenvector, which in each case leads to a value of $\gamma$ which is 
marked by the black dot on the blue curve. 
This allows us to illustrate the influence on the relationship between 
$\gamma$ and $\kappa$ on the next semi-supervised eigenvector.
Figure~\ref{fig:sseigs_conceptualA} also highlights the range ($\gamma<0$)
in which Personalized PageRank can be used for computing solutions to 
semi-supervised eigenvectors.
}\label{fig:sseigs_conceptual}
\end{figure}


To illustrate the previous discussion, Figure \ref{fig:sseigs_conceptual} 
considers the two-dimensional grid. 
In each subfigure, the blue curve shows the correlation with a single seed
node as a function of $\gamma$ for the leading semi-supervised 
eigenvector, and the black dot illustrates the value of $\gamma$ for three
different values of the locality parameter $\kappa$. 
This relationship between $\kappa$ and $\gamma$ is in general non-convex, 
but it is monotonic for $\gamma \in (-\vol(G), \lambda_2(G))$. 
The red curve in each subfigure shows the decay for the second 
semi-supervised eigenvector.
Recall that it is perpendicular to the first semi-supervised eigenvector,
that the decay is monotonic for 
$\gamma \in (-\vol(G), \lambda_2'(G))$, and that 
$\lambda_2 < \lambda_2' \leq \lambda_3$.  
In Figure~\ref{fig:sseigs_conceptualA}, the first semi-supervised 
eigenvector is not ``too'' close to $\lambda_2$, and so $\lambda_2'$ 
(\emph{i.e.}, the second eigenvalue of the next semi-supervised 
eigenvector) increases just slightly. 
In Figure~\ref{fig:sseigs_conceptualB}, we consider a locality parameter 
that leads to a value of $\gamma$ that is closer to $\lambda_2$, thereby 
increasing the value of $\lambda_2'$.
Finally, in Figure~\ref{fig:sseigs_conceptualC}, the locality parameter is
such that the leading semi-supervised eigenvector almost coincides with 
$\bm v_2$; this results in $\lambda_2' \approx \lambda_3$, as required if 
we were to compute the global eigenvectors.

\subsection{Bounding the Binary Search}\label{sec:bounding}
For the following derivations it is more convenient to consider the normalized graph Laplacian, in which case we define the first solution as
\begin{align}
y_1 = c\left( \mathcal{L}_G - \gamma_1 I \right)^{+} D_G^{1/2}s \label{eq:sol_y}
\end{align}
where $x_1^{*}=D_G^{-1/2}y_1$. This approach is convenient since the projection operator with null space defined by previous solutions can be expressed as $FF^T=I-YY^T$, assuming that $Y^TY=1$. That is, $Y$ is of the form
$$
Y = [ D_G^{1/2} , y^{*}_1 , \ldots , y^{*}_{t-1} ] ,
$$
where $y^{*}_i$ are successive solutions to Eqn. (\ref{eq:sol_y}).
In the following the type of projection operator will be implicit from the context, \emph{i.e.}, when working with the combinatorial graph Laplacian $FF^T = I - D_GX(X^T D_G D_G X)^{-1}X^T D_G$, whereas for the normalized graph Laplacian $FF^T=I-YY^T$. 

For the normalized graph Laplacian $\mathcal{L}_G$, the eigenvalues of $\mathcal{L}_G v = \lambda v$ equal the eigenvalues of the generalized eigenvalue problem $L_G v=\lambda D_G v$.
The binary search employed in Algorithm \ref{alg_new_1} uses a monotonic relationship between the $\gamma \in ( -\vol(G) ,\lambda_2(\cdot))$ parameter and the correlation with the seed $x^TD_Gs$, that can be deduced from the KKT-conditions \cite{MOV12-JMLR}. Note, that if the upper bound for the binary search $\top=\lambda_2(FF^T \mathcal{L}_GFF^T)$ is not determined with sufficient precision, the search will (if we underestimate $\top$) fail to satisfy the constraint, or (if we overestimate $\top$) fail to converge because the monotonic relationship no longer hold. 

By Lemma \ref{proof1} in Appendix \ref{app:proofs} it follows that $\lambda_2(FF^T \mathcal{L}_GFF^T)=\lambda_2(\mathcal{L}_G + \omega YY^T)$ when $\omega \rightarrow \infty$. Since the latter term is a sum of two PSD matrices, the value of the upper bound can only increase as stated by Lemma \ref{proof2} in Appendix \ref{app:proofs}. This is an important property, 
because if we do not recalculate $\top$, the previous value is guaranteed to be an underestimate, meaning that the objective will remain convex. Thus, it may be more efficient to first recompute $\top$ when the binary search fails to satisfy $( x^T D_G s)^2=\kappa$, meaning that $\top$ must be recomputed to increase the search range.



We compute the value for the upper bound of the binary search by transforming the problem in such a way that we can determine the greatest eigenvalue of a new system (fast and robust), and from that, deduce the new value of $\top=\lambda_2(FF^T \mathcal{L}_GFF^T)$.  
We do so by expanding the expression as
\begin{align*}
FF^T\mathcal{L}_GFF^T&=FF^T\left (I-D_G^{-1/2}A_GD_G^{-1/2} \right)FF^T\\
&=FF^T-FF^TD_G^{-1/2}A_GD_G^{-1/2}FF^T\\
&=I-\left (FF^TD_G^{-1/2}A_GD_G^{-1/2}FF^T+YY^T \right ).
\end{align*}
Since all columns of $Y$ will be eigenvectors of $FF^T\mathcal{L}_GFF^T$ with zero eigenvalue, these will all be eigenvectors of $FF^TD_G^{-1/2}A_GD_G^{-1/2}FF^T+YY^T$ with eigenvalue $1$. Hence, the largest algebraic eigenvalue $\lambda_{\text{LA}}(FF^TD_G^{-1/2}A_GD_G^{-1/2}FF^T)$ can be used to compute the upper bound for the binary search as 
\begin{align}
\top=\lambda_2(FF^T \mathcal{L}_GFF^T)=1-\lambda_{\text{LA}}(FF^TD_G^{-1/2}A_GD_G^{-1/2}FF^T).\label{eq:upperbound}
\end{align}
The reason for not considering the largest magnitude eigenvalue, is that $A_G$ may be indefinite. Finally, with respect to our implementation we emphasize that $FF^T$ is used as a projection operator, and not represented explicitly.

\section{Extension of Our Main Algorithm and Implementation Details}
\label{sxn:main-alg-extensions}
%
In this section, we present two variants of our main algorithm that are 
more well-suited for very large-scale applications; the first uses a 
column-based low-rank approximation, and the second uses random walk 
ideas.
In Section \ref{sec:nystrom}, we describe how to use the Nystr\"{o}m 
method, which constructs a low-rank approximation to the kernel matrix by 
sampling columns, to construct a general solution for semi-supervised 
eigenvectors, where the low-rankness is exploited for very efficient 
computation.
Then, in Section \ref{sec:reid}, we describe a ``Push-peeling heuristic,''
based on the efficient Push algorithm by~\cite{andersen06local}.
The basic idea is that if, rather than iteratively computing 
locally-biased semi-supervised eigenvectors using the procedure described 
in Algorithm~\ref{alg_new_1}, we instead compute solutions to 
\textsc{LocalSpectral} and then construct the semi-supervised eigenvectors
by ``projecting away'' pieces of these solutions, then we can take 
advantage of local random walks that have improved algorithmic properties.

%


\subsection{A Nystr{\"{o}}m-based Low-rank Approach}\label{sec:nystrom}
%
%
Here we describe the use of the recently-popular Nystr{\"{o}}m method to 
speed up the computation of semi-supervised eigenvectors.
We do so by considering how a low-rank decomposition can be exploited to 
yield solutions to the \textsc{Generalized LocalSpectral} objective in 
Figure~\ref{fig:objective}, where the running time largely depends on a 
matrix-vector product.
These methods are most appropriate when the kernel matrix is reasonably 
well-approximated by a low-rank matrix~\cite{drineas2005nystrom,gittens2012revisiting,Williams00theeffect}.

Given some low-rank approximation $\mathcal{L}_G \approx I- V \Lambda V^T$, we apply the Woodbury matrix identity, and we derive an explicit solution for the leading semi-supervised eigenvector 
\begin{align*}
y_1&\approx c\left ((1-\gamma) I - V \Lambda V^T\right)^+D_G^{1/2}s\\
&\approx c\left ( \frac{1}{1-\gamma}I + \frac{1}{(1-\gamma)^2}V \left(\Lambda^{-1} -\frac{1}{1-\gamma}I  \right)^{-1}V^T \right ) D_G^{1/2}s \\
&\approx\frac{c}{1-\gamma}\left(I+V \Sigma V^T\right)D_G^{1/2}s,
\end{align*}
where $\Sigma_{ii}=\frac{1}{\frac{1-\gamma}{\lambda_i}-1}$. 
In order to compute efficiently the subsequent semi-supervised eigenvectors we must accommodate for the projection operator $FF^T=I-YY^T$, while yet exploiting the explicit closed-form inverse 
$(\mathcal{L}_G-\gamma I)^+ \approx \frac{1}{1-\gamma}\left(I+V \Sigma V^T\right)$.
However, the projection operator complicates the expression, since the previous solution can be spanned by multiple global eigenvectors, so leveraging from the low-rank decomposition is more difficult for the inverse $(FF^T(\mathcal{L}_G-\gamma I)FF^T)^+$. 

Conveniently, we can decouple the projection operator by treating the orthogonality constraint using a Lagrangian approach, such that the solution can be expressed as
\begin{align*}
y_t=c \left(\mathcal{L}_G-\gamma I+\omega YY^T \right)^+D_G^{1/2}s, 
\end{align*}
where $\omega\geq 0$ denotes the associated Lagrange multiplier, and where the sign is deduced from the KKT conditions. Applying the Woodbury matrix identity is now straightforward  
\renewcommand{\P}{{P_\gamma}}
\begin{align}
\left (\P+\omega YY^T \right)^+ &= \P^+-\omega \P^+ Y\left (I+\omega Y^T\P^+Y\right)^+Y^T\P^+, \notag
\intertext{where for notational convenience we have introduced $P_\gamma=\mathcal{L}_G-\gamma I$. By decomposing $Y^T\P^+Y$ with an eigendecomposition $USU^T$ the equation simplifies as follows}
\left (\P+\omega YY^T \right)^+&= \P^+-\omega \P^+ Y\left(I+\omega USU^T\right)^+Y^T\P^+ \notag \\
&=\P^+ -\P^+ Y U\Omega U^T Y^T\P^+, \notag
\intertext{where $\Omega_{ii}=\frac{1}{\frac{1}{\omega}+S_{ii}}$. Note how this result gives a well defined way of controlling the amount of ``orthogonality'', and by Lemma \ref{proof1} in Appendix \ref{app:proofs}, we get exact orthogonality in the limit of $\omega \rightarrow \infty$, in which case the expression simplifies to}
\left (\P+\omega YY^T \right)^+&= \P^+- \P^+Y (Y^T\P^+Y)^+ Y^T\P^+. \notag
\intertext{Using the explicit expression for $\P^+$, the solution now only involves matrix-vector products and the inverse of a small matrix}
y_t&=c \left (\P^+- \P^+Y (Y^T\P^+Y)^+ Y^T\P^+ \right )D_G^{1/2}s \label{eq:lagrange_exact}.
\end{align}

To conclude this section, let us also consider how we can optimize the efficiency of the calculation of $\lambda_2(FF^T \mathcal{L}_GFF^T)$ used for bounding the binary search in Algorithm \ref{alg_new_1}. 
According to Eqn. (\ref{eq:upperbound}) the bound can be calculated efficiently as $\top=1-\lambda_{\text{LA}}(FF^TD_G^{-1/2}A_GD_G^{-1/2}FF^T)$. However, by substituting with $D_G^{-1/2}A_GD_G^{-1/2}\approx V \Lambda V^T$, we can exploit low-rankness since
\begin{align*}
\top=1-\lambda_\text{LA}(FF^TV \Lambda V^TFF^T) = 1-\lambda_\text{LA}(\Lambda^{1/2} V^TFF^TV \Lambda^{1/2} ),
\end{align*}
where the latter is a much smaller system.

\subsection{A Push-peeling Heuristic}\label{sec:reid}
Here we present a variant of our main algorithm that exploits 
the connections between diffusion-based procedures 
and eigenvectors, allowing semi-supervised eigenvectors to be efficiently computed for large networks. 
This is most well-known for the leading nontrivial eigenvectors of the
graph Laplacian~\cite{Chung:1997}; but recent work has exploited these 
connections in the context of performing locally-biased spectral graph
partitioning~\cite{Spielman:2004,andersen06local,MOV12-JMLR}.
In particular, we can compute the locally-biased vector using the first step
of Algorithm~\ref{alg_new_1}, or alternatively we can compute it using a 
locally-biased random walk of the form used 
in~\cite{Spielman:2004,andersen06local}. Here we present a heuristic that works by peeling off components from a solution to the PageRank problem, and by exploiting the regularization interpretation of $\gamma$, we can from these components obtain the subsequent semi-supervised eigenvectors. 

Specifically, we focus on the Push algorithm by~\cite{andersen06local}. This algorithm approximates the solution to PageRank very efficiently, 
by exploiting the local modifications that occur when the seed is highly concentrated. This makes our algorithm very scalable and applicable for large-scale data, since only the local neighborhood near the seed set will be touched by the algorithm. In comparison, by solving the linear system of equations we explicitly touch all nodes in the graph, even though most spectral rankings will be below the computational precision~\cite{Vig11_TR}.

We adapt a similar notation as in~\cite{andersen06local} and start by defining the usual PageRank vector $\text{pr}(\alpha,s_\text{pr})$ as the unique solution of the linear system
 \begin{align}
\text{pr}(\alpha,s_\text{pr})&=\alpha s_\text{pr} + (1-\alpha)A_GD_G^{-1}\text{pr}(\alpha,s_\text{pr})\label{eq:pagerank},
\intertext{where $\alpha$ is the teleportation parameter, and $s_\text{pr}$ is the sparse starting vector. 
For comparison, the push algorithm by~\cite{andersen06local} computes an approximate PageRank vector $\text{pr}_\epsilon(\alpha',s_\text{pr})$ for a slightly different system}
\text{pr}_\epsilon(\alpha',s_\text{pr})&=\alpha' s_\text{pr} + (1-\alpha')W\text{pr}_\epsilon(\alpha',s_\text{pr}),\notag
\end{align}
where $W=\frac{1}{2}(I+A_GD_G^{-1})$ and not the usual random walk matrix $AD_G^{-1}$ as used in Eqn. (\ref{eq:pagerank}). However, these equations are only superficially different, and equivalent up to a change of the respective teleportation parameter. Thus, it is straightforward to verify that these teleportation parameters and the $\gamma$ parameter of Eqn. (\ref{eq:mahoney}) are related as 
\begin{align*}
\alpha=\frac{2\alpha'}{1+\alpha'}\Leftrightarrow \alpha'=\frac{\alpha}{2-\alpha}\Leftrightarrow \alpha'=\frac{\gamma}{\gamma -2}, 
\end{align*}
and that the leading semi-supervised eigenvector for $\gamma \in (-\infty,0)$ can be approximated as 
\begin{align*}
x_1^{*}\approx \frac{c}{-\gamma} D_G^{-1} \text{pr}_\epsilon \left(\frac{\gamma}{\gamma-2},D_Gs \right).
\end{align*}
To generalize subsequent semi-supervised eigenvectors to this diffusion based framework, we need to accommodate for the projection operator such that subsequent solutions can be expressed in terms of graph diffusions. By requiring distinct values of $\gamma$ for all semi-supervised eigenvectors, we may use the solution for the leading semi-supervised eigenvector and then systematically ``peel off'' components, thereby obtaining the solution of one of the consecutive semi-supervised eigenvectors. 
By Lemma \ref{claim:peeling}, in Appendix \ref{app:proofs} the general solution in Eqn. (\ref{eq:semisupeigs}) can be approximated by
\begin{align}
x_t^{*}  \approx c \left ( I -XX^TD_G \right)( L_G - \gamma_t D_G)^{+} D_G s,\label{eq:approxpeeling}
\end{align}
under the assumption that all $\gamma_k$ for $1<k\leq t$ are sufficiently apart. If we think about $\gamma_k$ as being distinct eigenvalues of the generalized eigenvalue problem $L_Gx_k=\gamma_k D_G x_k$, then it is clear that Eqn. (\ref{eq:approxpeeling}), correctly computes the sequence of generalized eigenvectors. This is explained by the fact that $( L_G - \gamma_t D_G)^{+} D_G s$ can be interpreted as the first step of the Rayleigh quotient iteration, where $\gamma_t$ is the estimate of the eigenvalue, and $D_G s$ is the estimate of the eigenvector. Given that the estimate of the eigenvalue is right, this algorithm will in the initial step compute the corresponding eigenvector, and the operator $\left ( I -XX^TD_G \right)$ will be superfluous, as the global eigenvectors are already orthogonal in the degree-weighted norm. To quantify the failure modes of the approximation, let us consider what happens when $\gamma_2$ starts to approach $\gamma_1$. What constitutes the second solution for a particular value of $\gamma_2$ is the perpendicular component with respect to the projection onto the solution given by $\gamma_1$. As $\gamma_2$ approaches $\gamma_1$, this perpendicular part  diminishes and the solution becomes ill-posed. Fortunately, we can easily detect such issues during the binary search in Algorithm \ref{alg_new_1}, and in general the approximation has turned out to work very well in practice as our experimental results in Section \ref{sxn:empirical} show. 

In terms of the approximate PageRank vector $\text{pr}_\epsilon(\alpha',s_\text{pr})$ , the general approximate solution takes the following form
\begin{align}
x_t^{*}\approx c \left ( I-XX^TD_G \right) D_G^{-1} \text{pr}_\epsilon \left(\frac{\gamma_t}{\gamma_t-2},D_Gs \right)\label{eq:pushslow}.
\end{align}
As already stated in Section \ref{sec:discussion}, the impact of using a diffusion based procedure is that we cannot interpolate all the way to the global eigenvectors, and that the main challenge is that the solutions do not become too localized. The $\epsilon$ parameter of the Push algorithm controls the threshold for propagating mass away from the seed set and into the adjacent nodes in the graph. If the threshold is too high, the solution will be very localized and make it difficult to find more than a few semi-supervised eigenvectors, as characterized by Lemma \ref{claim:corrsum} in Appendix \ref{app:proofs}, because the leading ones will then span the entire space of the seed set. 
As the choice of $\epsilon$ is important for the applicability of our algorithm, we will in Section \ref{sxn:empirical} investigate the influence of this parameter on large data graphs. 

To conclude this section, we consider an important implementation detail that have been omitted so far. In the work of~\cite{MOV12-JMLR} the seed vector was defined to be perpendicular to the all-ones vector, and for the sake of consistency we have chosen to define it in the same way. The impact of projecting the seed set to a space that is perpendicular to the all-ones vector is that the resulting seed vector is no longer sparse, making the use of the Push algorithm in Eqn. (\ref{eq:pushslow}) inefficient.
The seed vector can, however, without loss of generality, be defined as $s\propto D_G^{-1/2} \left (I-v_0v_0^T \right )s_0$ where $s_0$ is the sparse seed, and $v_0\propto \text{diag} \left (D_G^{1/2} \right )$ is the leading eigenvector of the normalized graph Laplacian (corresponding to the all-ones vector of the combinatorial graph Laplacian).
If we substitute with this expression for the seed in Eqn. (\ref{eq:pushslow}), it follows by plain algebra (see Appendix \ref{app:sparsediffusion}) that
\begin{align}
x_t^{*}&\approx c \left (I-XX^TD_G \right )\left ( D_G^{-1}    \text{pr}_\epsilon \left(\frac{\gamma_t}{\gamma_t-2},D_G^{1/2}s_0 \right)  - D_G^{-1/2} v_0 v_0^Ts_0 \right )\label{eq:pushfast}.
\end{align}
Now the Push algorithm is only defined on the sparse seed set making the the expression very scalable. Finally, the Push algorithm maintains a queue of high residual nodes that are yet to be processed. The order in which nodes are processed influences the overall running time, and in ~\cite{Vig11_TR} preliminary experiments showed that a FIFO queue resulted in the best performance for large values of $\gamma$,  as compared to a priority queue that scales logarithmically. For this reason we have chosen to use a FIFO queue in our implementation.


\section{Empirical Results}
\label{sxn:empirical}

In this section, we provide a detailed empirical evaluation of the method of 
semi-supervised eigenvectors and how it can be used for locally-biased 
machine learning.
Our goal is two-fold: 
first, to illustrate how the ``knobs'' of the method work; and
second, to illustrate the usefulness of the method in real applications.
To do this, we consider several classes of data.
\begin{itemize}
\item
\textbf{Toy data.}
In Section~\ref{sxn:empirical-toy}, we consider one-dimensional examples of 
the popular ``small world'' model~\cite{watts98collective}.
This is a parameterized family of models that interpolates between 
low-dimensional grids and random graphs; and, as such, it allows us to 
illustrate the behavior of the method and its various parameters in a 
controlled setting.
\item
\textbf{Congressional voting data.}
In Section~\ref{sxn:empirical-congress}, we consider roll call voting 
data from the United States Congress that are based on~\cite{PR91}.
This is an example of realistic data set that has relatively-simple global 
structure but nontrivial local structure that varies with 
time~\cite{CM11_TR}; and thus it allows us to illustrate the method in 
a realistic but relatively-clean setting.
\item
\textbf{Handwritten image data.}
In Section~\ref{sxn:empirical-digits}, we consider data from the 
MNIST digit data set~\cite{mnistlecun}.
These data  have been widely-studied in machine learning and related areas 
and they have substantial ``local heterogeneity.''
Thus, these data allow us to illustrate how the method may be used to 
perform locally-biased versions of common machine learning tasks such as 
smoothing, clustering, and kernel construction.
\item
\textbf{Functional brain imaging data.}
In Section~\ref{sxn:empirical-fmri}, we consider functional magnetic 
resonance imaging (fMRI) data.
Single subject fMRI data is characterized by high dimensionality and relatively few samples, in contrast to the MNIST data 
that consist of many samples and a relatively low dimensionality. 
We demonstrate how our semi-supervised eigenvectors 
can be applied to construct a data-driven spatially-biased basis by incorporating \emph{a priori}
knowledge from a functional brain atlas~\cite{Eickhoff2005}.
\item
\textbf{Large-scale network data.}
In Section~\ref{sec:largescalenetworkdata}, we consider large-scale network data, and demonstrate significant performance improvements of the
push-peeling heuristic compared to solving the same equations using a conjugate gradient solver.
These improvements are demonstrated on datasets from the DIMACS implementation challenge, as well as on large web-crawls with more then 3 billion non-zeros in the adjacency matrix~\cite{BCSU3,BRSLLP,BoVWFI}.  
\end{itemize}

\subsection{Small-world Data}
\label{sxn:empirical-toy}

The first data sets we consider are networks constructed from the so-called 
small-world model. 
This model can be used to demonstrate how semi-supervised eigenvectors focus 
on specific target regions of a large data graph to capture slowest modes of 
local variation; and it can also be used to illustrate how the ``knobs'' of 
the method work, \emph{e.g.}, how $\kappa$ and $\gamma$ interplay, in a 
practical setting.
In Figure~\ref{fig:smallw}, we plot the usual global eigenvectors, as well 
as locally-biased semi-supervised eigenvectors, around illustrations of 
non-rewired and rewired realizations of the small-world graph, \emph{i.e.}, 
for different values of the rewiring probability $p$ and for different 
values of the locality parameter $\kappa$.

\begin{figure*}[hbt!]
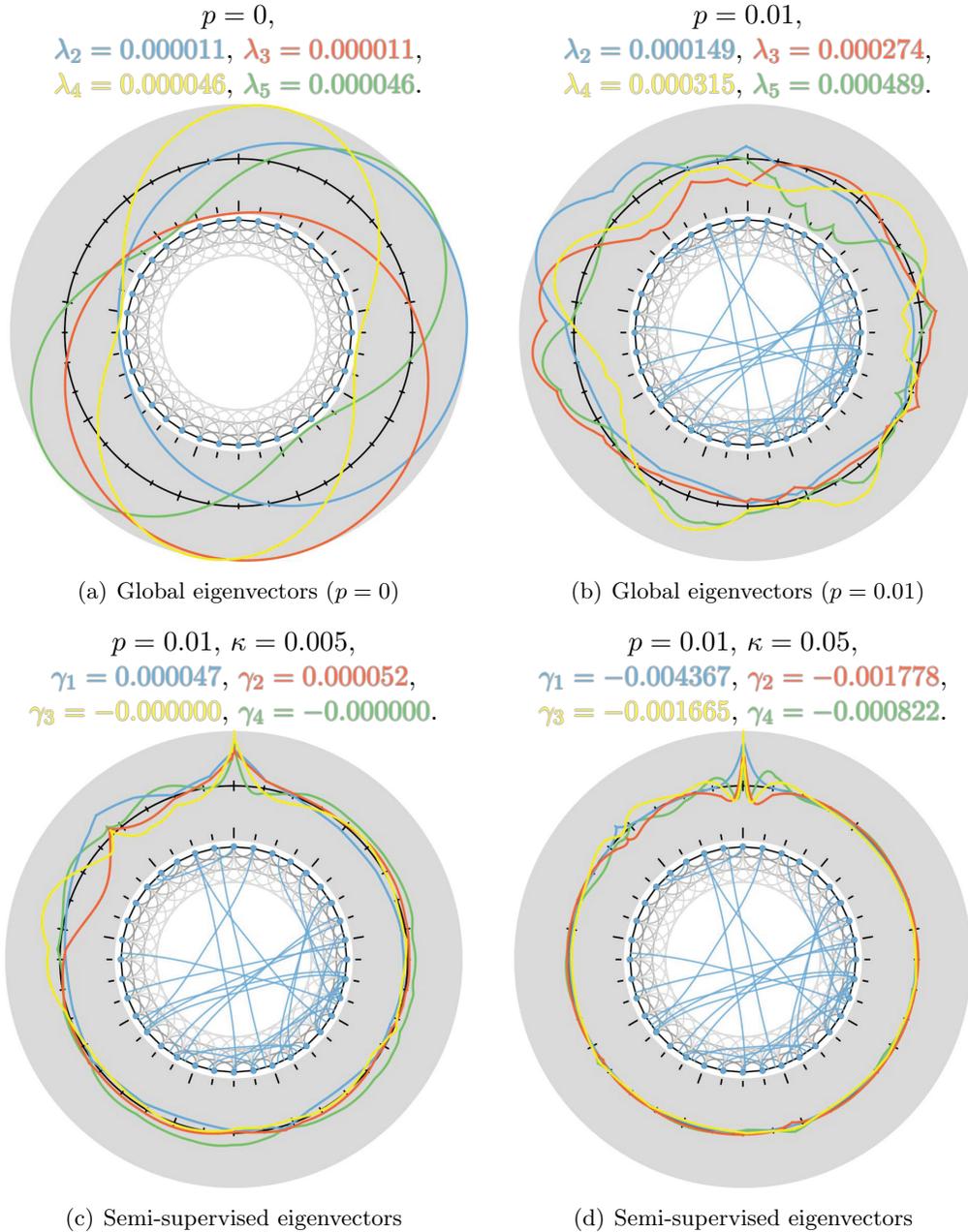

\centering
\subfigure[Global eigenvectors ($p=0$)]{
\begin{minipage}[b]{0.40\linewidth}
\centering
\input{circos-n=3600,p=0,k=0.tex}
\includegraphics[scale=0.2]{\gs{circos-n=3600,p=0,k=0}}\\
\end{minipage}
\label{fig:smallwA}
}
\subfigure[Global eigenvectors ($p=0.01$)]{
\begin{minipage}[b]{0.40\linewidth}
\centering
\input{circos-n=3600,p=0_01,k=0.tex}
\includegraphics[scale=0.2]{\gs{circos-n=3600,p=0_01,k=0}}\\
\end{minipage}
\label{fig:smallwB}
}
\subfigure[Semi-supervised eigenvectors]{
\begin{minipage}[b]{0.40\linewidth}
\centering
\input{circos-n=3600,p=0_01,k=0_005.tex} 
\includegraphics[scale=0.2]{\gs{circos-n=3600,p=0_01,k=0_005}}\\
\end{minipage}
\label{fig:smallwC}
}
\subfigure[Semi-supervised eigenvectors]{
\begin{minipage}[b]{0.40\linewidth}
\centering
\input{circos-n=3600,p=0_01,k=0_05.tex}
\includegraphics[scale=0.2]{\gs{circos-n=3600,p=0_01,k=0_05}}\\
\end{minipage}
\label{fig:smallwD}
}
\caption[*]{Illustration of small-world graphs with rewiring probability of 
$p=0$ or $p=0.01$ and with different values of the $\kappa$ parameter.
For each subfigure, the data consist of $3600$ nodes, each connected to it's 
$8$ nearest-neighbors. 
In the center of each subfigure, we show the nodes 
(\textcolor[rgb]{0.4196,0.6824,0.8392}{blue}) and edges (black and light gray 
are the local edges, and \textcolor[rgb]{0.4196,0.6824,0.8392}{blue} 
are the randomly-rewired edges).
We wrap around the plots (black x-axis and gray background), 
visualizing the $4$ smallest semi-supervised eigenvectors.
Eigenvectors are color coded as 
\textcolor[rgb]{0.4196,0.6824,0.8392}{\contour{light-gray}{blue}}, 
\textcolor[rgb]{0.9843,0.4157,0.2902}{\contour{light-gray}{red}}, 
\textcolor[rgb]{1.0000,0.9569,0.3137}{\contour{light-gray}{yellow}}, and 
\textcolor[rgb]{0.4549,0.7686,0.4627}{\contour{light-gray}{green}}, 
starting with the one having the smallest eigenvalue. 
}
\label{fig:smallw}
\end{figure*}

To start, in Figure~\ref{fig:smallwA} that we show a graph with no 
randomly-rewired edges ($p=0$) and a parameter $\kappa$ such that the global 
eigenvectors are obtained. 
This yields a symmetric graph with eigenvectors corresponding to orthogonal 
sinusoids, \emph{i.e.}, the first two capture the slowest mode of variation 
and correspond to a sine and cosine with equal random phase-shift (up to a 
rotational ambiguity). 
In Figure~\ref{fig:smallwB}, random edges have been added with probability 
$p=0.01$ and the parameter $\kappa$ is still chosen such that the global 
eigenvectors---now of the rewired graph---are obtained.
Note the many small kinks in the eigenvectors at the location of the 
randomly added edges. 
Note also the slow mode of variation in the interval on the top left; a 
normalized-cut based on the leading global eigenvector would extract this 
region, since the remainder of the ring is more well-connected due to the 
random rewiring.

In Figure~\ref{fig:smallwC}, we see the same graph realization as in 
Figure~\ref{fig:smallwB}, except that the semi-supervised eigenvectors 
have a seed node at the top of the circle, \emph{i.e.}, at ``12 o-clock,'' 
and the locality parameter $\kappa_t=0.005$, which corresponds to 
moderately well-localized eigenvectors. 
As with the global eigenvectors, the locally-biased semi-supervised 
eigenvectors are of successively-increasing (but still localized) variation.
Note also that the neighborhood around ``11 o-clock'' contains more 
mass, \emph{e.g.}, when compared with the same parts of the circle in 
Figure~\ref{fig:smallwB} or with other parts of the circle in 
Figure~\ref{fig:smallwC}, even though it is not very near the seed node in 
the original graph geometry.  
The reason for this is that this region is well-connected with the seed via 
a randomly added edge, and thus it is close in the modified graph topology. 
Above this visualization, we also show the value of $\gamma_t$ that saturates 
$\kappa_t$, \emph{i.e.}, $\gamma_t$ is the Lagrange multiplier that defines 
the effective locality $\kappa_t$. 
Not shown is that if we kept reducing $\kappa_t$, then $\gamma_t$ would tend 
towards $\lambda_{t+1}$, and the respective semi-supervised eigenvectors 
would tend towards the global eigenvectors that are illustrated in 
Figure~\ref{fig:smallwB}. 
Finally, in Figure~\ref{fig:smallwD}, the desired locality is increased 
to $\kappa=0.05$ (which has the effect of decreasing the value of 
$\gamma_t$), making the semi-supervised eigenvectors more localized in the 
neighborhood of the seed. 
It should be clear that, in addition to being determined by the locality 
parameter, we can think of $\gamma$ as a regularizer biasing the global 
eigenvectors towards the region near the seed set.
That is, variation in eigenvectors that are near the initial seed (in the 
modified graph topology) are most important, while variation that is far 
away from the initial seed matters much less.


\subsection{Congressional Voting Data}
\label{sxn:empirical-congress}

The next data set we consider is a network constructed from a time 
series of roll call voting patterns from the United States Congress that are 
based on~\cite{PR91}.
This is a particularly well-structured social network for which there is a 
great deal of meta-information, and it has been studied recently with 
graph-based methods~\cite{multiplex_Mucha,WPFMP09_TR,CM11_TR}.
Thus, it permits a good illustration of the method of semi-supervised 
eigenvectors in a real application~\cite{Poole05}.
This data set is known to have nontrivial time-varying structure at 
different time steps, and we will illustrate how the method of 
semi-supervised eigenvectors can perform locally-biased classification with 
a traditional kernel-based algorithm.

In more detail, 
we evaluate our method by considering the known Congress data-set containing 
the roll call voting patterns in the U.S Senate across time. 
We considered Senates in the $70^{th}$ Congress through the $110^{th}$ 
Congress, thus covering the years $1927$ to $2008$. 
During this time, the U.S went from $48$ to $50$ states, hence the number of
senators in each of these $41$ Congresses was roughly the same. 
We constructed an $N \times N$ adjacency matrix, with $N = 4196$ ($41$ 
Congresses each with $\approx 100$ Senators) where $A_{ij}\in [0,1]$ 
represents the extent of voting agreement between legislators $i$ and $j$, 
and where identical senators in adjacent Congresses are connected with an 
inter-Congress connection strength. 
We then considered the Laplacian matrix of this graph, constructed in the 
usual way~\cite{CM11_TR}.

\begin{figure*}[hbt!] 
\centering
\includegraphics[scale=0.4]{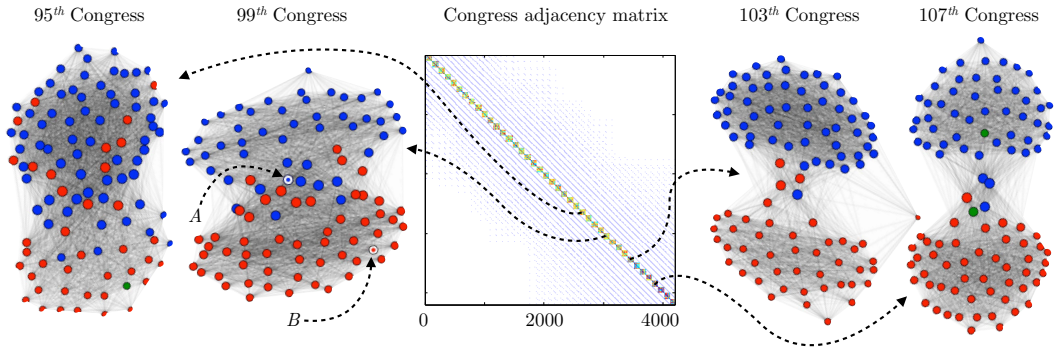}
\caption{Shows the Congress adjacency matrix, along with four of the 
individual Congresses. Nodes are scaled according to their degree, blue 
nodes correspond to Democrats, red to Republicans, and green to 
Independents.
}
\label{fig:cong-conceptual}
\end{figure*}

Figure \ref{fig:cong-conceptual} visualizes the adjacency matrix, along with 
four of the individual Congresses, color coded by party.
This illustrates that these data should be viewed---informally---as a 
structure (depending on the specific voting patterns of each Congress) 
evolving along a one-dimensional temporal axis, confirming the results 
of~\cite{CM11_TR}.
Note that the latter two Congresses are significantly better described by
a simple two-clustering than the former two Congresses, and an examination 
of the clustering properties of each of the $40$ Congresses reveals 
significant variation in the local structure of individual Congresses, in a 
manner broadly consistent with \cite{Poole05} and \cite{PR91}.
In particular, the more recent Congresses are significantly more polarized.

\begin{figure*}[hbt!]
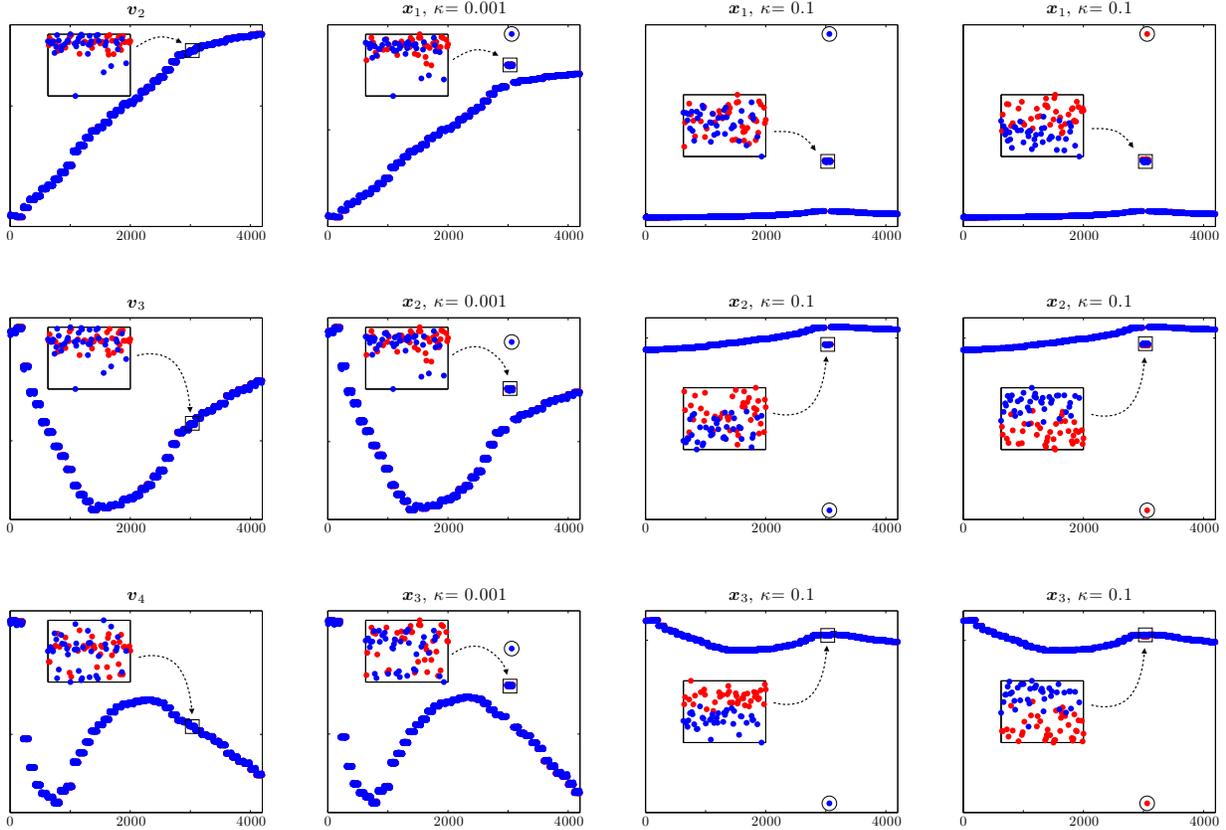

\begin{minipage}[b]{0.21\linewidth}
\centering
\includegraphics[scale=0.22]{\gs{v1-semi-0_0-crop}}
\end{minipage}
\hspace{0.5cm}
\begin{minipage}[b]{0.21\linewidth}
\centering
\includegraphics[scale=0.22]{\gs{v1-semi-0_001-crop}}
\end{minipage}
\hspace{0.5cm}
\begin{minipage}[b]{0.21\linewidth}
\centering
\includegraphics[scale=0.22]{\gs{v1-semi-0_1-crop}}
\end{minipage}
\hspace{0.5cm}
\begin{minipage}[b]{0.21\linewidth}
\centering
\includegraphics[scale=0.22]{\gs{v1-semi2-0_1-crop}}
\end{minipage}
\newline 
\medskip
\newline
\begin{minipage}[b]{0.21\linewidth}
\centering
\includegraphics[scale=0.22]{\gs{v2-semi-0_0-crop}}
\end{minipage}
\hspace{0.5cm}
\begin{minipage}[b]{0.21\linewidth}
\centering
\includegraphics[scale=0.22]{\gs{v2-semi-0_001-crop}}
\end{minipage}
\hspace{0.5cm}
\begin{minipage}[b]{0.21\linewidth}
\centering
\includegraphics[scale=0.22]{\gs{v2-semi-0_1-crop}}
\end{minipage}
\hspace{0.5cm}
\begin{minipage}[b]{0.21\linewidth}
\centering
\includegraphics[scale=0.22]{\gs{v2-semi2-0_1-crop}}
\end{minipage}
\newline 
\medskip
\newline
\begin{minipage}[b]{0.21\linewidth}
\centering
\includegraphics[scale=0.22]{\gs{v3-semi-0_0-crop}}
\end{minipage}
\hspace{0.5cm}
\begin{minipage}[b]{0.21\linewidth}
\centering
\includegraphics[scale=0.22]{\gs{v3-semi-0_001-crop}}
\end{minipage}
\hspace{0.5cm}
\begin{minipage}[b]{0.21\linewidth}
\centering
\includegraphics[scale=0.22]{\gs{v3-semi-0_1-crop}}
\end{minipage}
\hspace{0.5cm}
\begin{minipage}[b]{0.21\linewidth}
\centering
\includegraphics[scale=0.22]{\gs{v3-semi2-0_1-crop}}
\end{minipage}
\caption{First column: The leading three nontrivial global eigenvectors. 
Second column: The leading three semi-supervised eigenvectors seeded 
(circled node) in an articulation point between the two parties in the 
$99^{th}$ Congress (see Figure \ref{fig:cong-conceptual}), for correlation 
$\kappa = 0.001$.  Third column: Same seed as previous column, but for a 
correlation of $\kappa=0.1$. 
Notice the localization on the third semi-supervised eigenvector. 
Fourth column: Same correlation as the previous column, but for another seed 
node well within the cluster of Republicans. 
Notice the localization on all three semi-supervised eigenvectors.
}
\label{fig:cong-eigenvects}
\end{figure*}

The first vertical column of Figure~\ref{fig:cong-eigenvects} illustrates 
the first three global eigenvectors of the full data set, illustrating 
fluctuations that are sinusoidal and consistent with the one-dimensional 
temporal scaffolding.
Also shown in the first column are the values of that eigenfunction for
the members of the $99^{th}$ Congress, illustrating that there is \emph{not} 
a good separation based on party affiliations.
The next three vertical columns of Figure~\ref{fig:cong-eigenvects} 
illustrate various localized eigenvectors computed by starting with a seed 
node in the $99^{th}$ Congress.
For the second column, we visualize the semi-supervised eigenvectors for a 
very low correlation ($\kappa=0.001$), which corresponds to only a weak 
localization---in this case one sees eigenvectors that look very similar to 
the global eigenvectors, and the elements of the eigenvector on that 
Congress do not reveal partitions based on the party~cuts.

The third and fourth column of Figure~\ref{fig:cong-eigenvects} illustrate 
the semi-supervised eigenvectors for a much higher correlation 
($\kappa=0.1$), meaning a much stronger amount of locality.
In particular, the third column starts with the seed node marked $A$ in
Figure \ref{fig:cong-conceptual}, which is at the articulation point between 
the two parties, while the fourth column starts with the seed node marked 
$B$, which is located well within the cluster of Republicans.
In both cases the eigenvectors are much more strongly localized on the
$99^{th}$ Congress near the seed node, and in both cases one observes the
partition into two parties based on the elements of the localized 
eigenvectors.
Note, however, that when the initial seed is at the articulation point 
between two parties then the situation is much noisier: in this case, this
``partitionability'' is seen only on the third semi-supervised eigenvector, 
while when the initial seed is well within one party then this is seen on 
all three eigenvectors.
Intuitively, when the seed set is strongly within a good cluster, then that
cluster tends to be found with semi-supervised eigenvectors (and we will 
observe this again below).
This is consistent with the diffusion interpretation of eigenvectors.
This is also consistent with~\cite{CM11_TR}, who observed that the 
properties of eigenvector localization depended on the local structure of 
the data around the seed node, as well as the larger scale structure around 
that local cluster.

\begin{figure*}[hbt]
\begin{minipage}[b]{0.3\linewidth}
\centering
\includegraphics[scale=0.34]{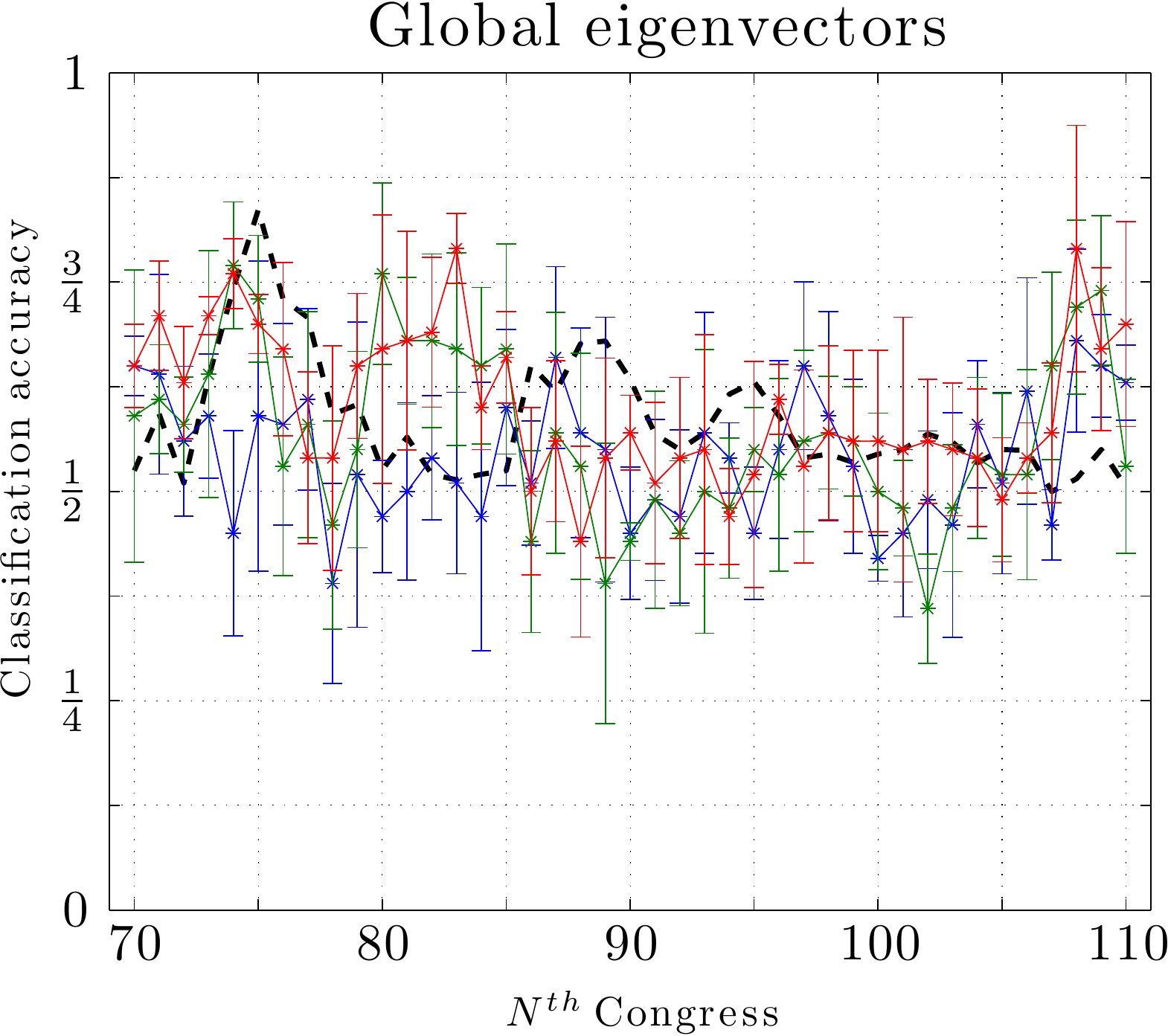}
\end{minipage}
\hspace{0.5cm}
\begin{minipage}[b]{0.3\linewidth}
\centering
\includegraphics[scale=0.34]{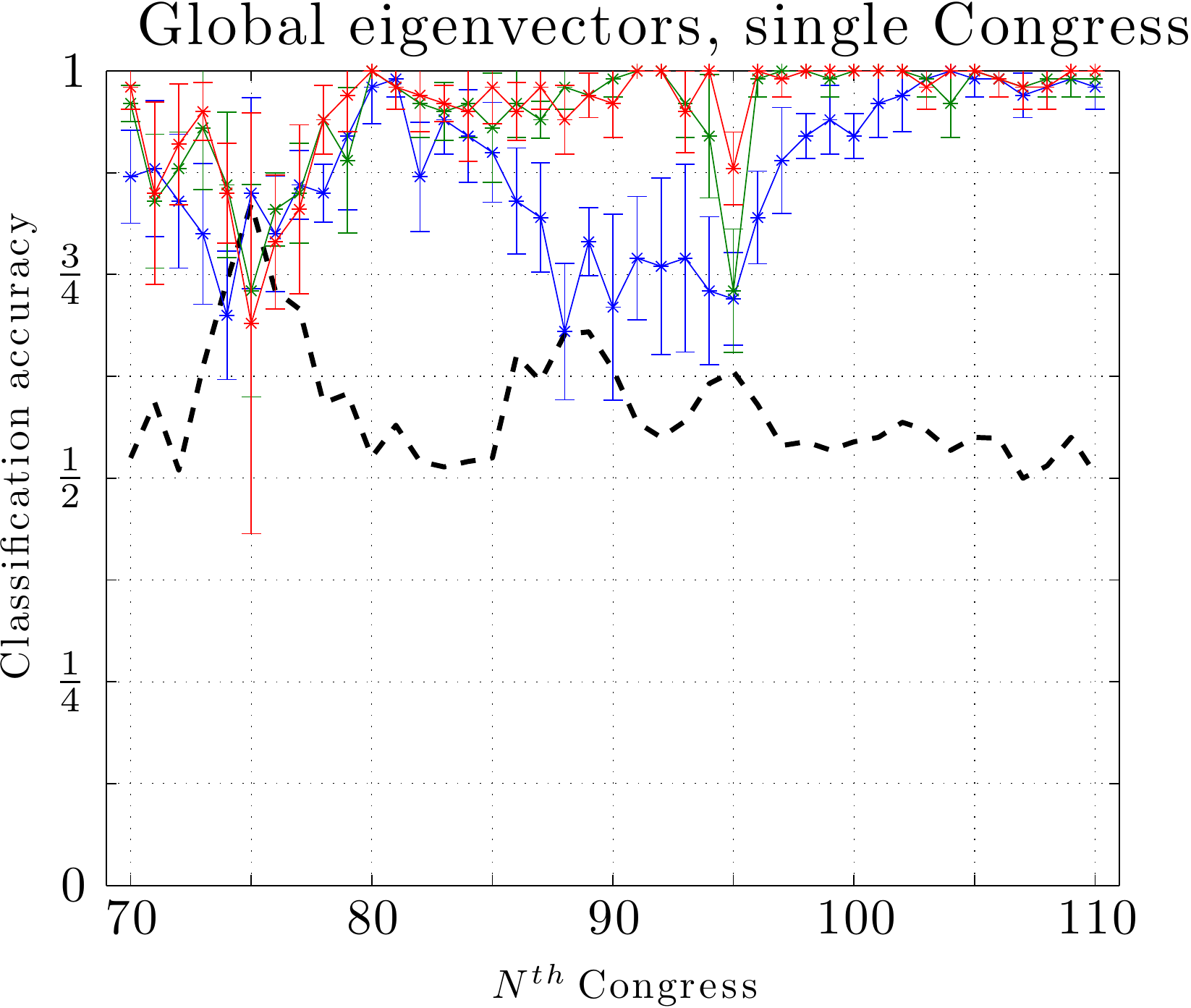}
\end{minipage}
\hspace{0.5cm}
\begin{minipage}[b]{0.3\linewidth}
\centering
\includegraphics[scale=0.34]{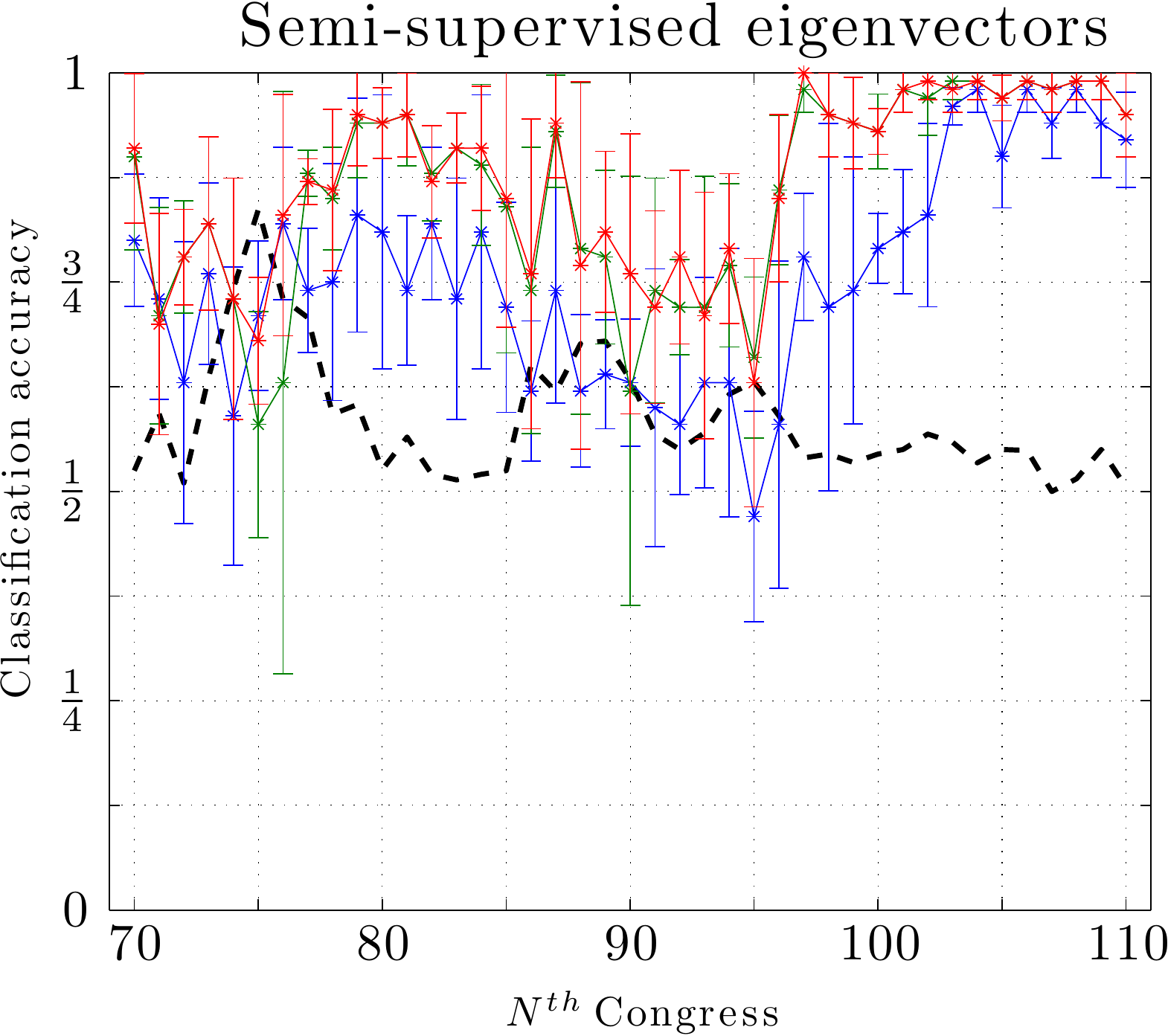}
\end{minipage}
\caption{Classification accuracy measured in individual Congresses. For each 
Congress we perform $5$-fold cross validation based on $\approx 80$ samples 
and leave out the remaining $20$ samples to estimate an unbiased test error. 
Error bars are obtained by resampling and they correspond to $1$ standard 
deviation. For each approach we consider features based on the $1^{st}$ 
(blue), $2^{nd}$ (green), and $3^{rd}$ (red) smallest eigenvector(s), 
excluding the all-one vector. We also plot the probability of the most 
probable class as a baseline measure (black) as some Congresses are very 
imbalanced.
}
\label{fig:cong-classification}
\end{figure*}

To illustrate how these structural properties manifest themselves in a
more traditional machine learning task, we also consider the classification 
task of discriminating between Democrats and Republicans in single 
Congresses, \emph{i.e.}, we measure to what extent we can extract local 
discriminative features. 
To do so, we apply $L_2$-regularized $L_2$-loss support vector 
classification with a linear kernel, where features are extracted using the 
global eigenvectors of the entire data set, global eigenvectors from a 
single Congress (best case measure), and our semi-supervised eigenvectors.
Figure~\ref{fig:cong-classification} illustrates the classification accuracy 
for $1$, $2$, and $3$ eigenvectors. 
As reported by~\cite{CM11_TR}, locations that exhibit discriminative 
information are best found on low-order eigenvectors of this data, 
explaining why the classifier based global eigenvectors performs poorly. 
In the classifier based on global eigenvectors in the single Congress we 
exploit \textit{a priori} knowledge to extract the relevant data, that in a 
usual situation would be impossible. 
Hence, this is simply to define a baseline point of reference for the best case 
classification accuracy. 
The classifier based on semi-supervised eigenvectors is seeded using a few 
training samples and performs in-between the two other approaches.
Compared to our point of reference, Congresses in the range $88$ to $96$ do 
worse with the semi-supervised eigenvectors; whereas for Congresses after 
$100$ the semi-supervised approach almost performs on par, even for a single 
single eigenvector.
This is consistent with the visualization in 
Figure~\ref{fig:cong-conceptual} illustrating that earlier Congresses are 
less cleanly separable, as well as with empirical evidence indicating 
heterogeneity due to Southern Democrats in earlier Congresses and the recent 
increase in party polarization in more recent Congresses, as described 
in~\cite{Poole05} and~\cite{PR91}.


\subsection{MNIST Digit Data}
\label{sxn:empirical-digits}

The next data set we consider is the well-studied MNIST data set containing 
$60,000$ training digits and $10,000$ test digits ranging from $0$ to $9$; 
and, with these data, we demonstrate the use of semi-supervised eigenvectors 
as a feature extraction preprocessing step in a traditional machine learning 
setting. 
We construct the full $70,000\times 70,000$ $k$-NN graph, with $k=10$ and 
with edge weights given by 
$w_{ij}=\exp({-\frac{4}{\sigma_i^2}\| x_i- x_j\|^2})$, where $\sigma_i^2$ 
is the Euclidian distance of the $i^{th}$ node to it's nearest neighbor; and 
from this we define the graph Laplacian in the usual way.
We then evaluate the semi-supervised eigenvectors in a transductive learning 
setting by disregarding the majority of labels in the entire training data. 
We use a few samples from each class to seed our semi-supervised 
eigenvectors as well as a few others to train a downstream classification 
algorithm. 
For this evaluation, we use the Spectral Graph Transducer (SGT) 
of~\cite{Joa03}; and we choose to use it for two main reasons. 
First, the transductive classifier is inherently designed to work on a 
subset of global eigenvectors of the graph Laplacian, making it ideal for 
validating that the localized basis constructed by the semi-supervised 
eigenvectors can be more informative when we are solely interested in the 
``local heterogeneity'' near a seed set. 
Second, using the SGT based on global eigenvectors is a good point of 
comparison, because we are only interested in the effect of our subspace 
representation.
(If we used one type of classifier in the local setting, and another in the 
global, the classification accuracy that we measure would obviously be 
confounded.) 
As in~\cite{Joa03}, we normalize the spectrum of both global and 
semi-supervised eigenvectors by replacing the eigenvalues with some 
monotonically increasing function. 
We use $\lambda_i=\frac{i^2}{k^2}$, \emph{i.e}., focusing on ranking among 
smallest cuts; see \cite{chapelle2002}. 
Furthermore, we fix the regularization parameter of the SGT to $c=3200$, and 
for simplicity we fix $\gamma=0$ for all semi-supervised eigenvectors, 
implicitly defining the effective $\kappa = [\kappa_1,\ldots, \kappa_k]^T$. 
Clearly, other correlation distributions $\kappa$ and other values of 
$\gamma$ parameter may yield subspaces with even better discriminative 
properties (which is an issue to which we will return in
Section~\ref{sec:kappa_impact} in greater detail).

 \begin{table*}[ht]
\scriptsize
\centering
\begin{tabular}{ccccccccccccccc}
\hline
& \multicolumn{6}{c}{\#Semi-supervised eigenvectors for SGT} & & \multicolumn{6}{c}{\#Global eigenvectors for SGT} \\
\cline{2-7} \cline{9-14}
Labeled points & 1 & 2 & 4 & 6 & 8 & 10 &   & 1 & 5 & 10 & 15 & 20 & 25 \\
\hline 
$1:1\;\:$ & 0.39 & 0.39 & 0.38 & 0.38 & 0.38 & 0.36 &      & 0.50 & 0.48 & 0.36 & 0.27 & 0.27 & 0.19\\
$1:10$ & 0.30 & 0.31 & 0.25 & 0.23 & 0.19 & 0.15 &    &0.49 & 0.36 & 0.09 & 0.08 & 0.06 & 0.06\\
$5:50$ & 0.12 & 0.15 & 0.09 & 0.08 & 0.07 & 0.06 &     &0.49 & 0.09 & 0.08 & 0.07 & 0.05 & 0.04\\
$10:100$ & 0.09 & 0.10 & 0.07 & 0.06 & 0.05 & 0.05 &     & 0.49 & 0.08 & 0.07 & 0.06 & 0.04 & 0.04\\
$50:500$ & 0.03 & 0.03 & 0.03 & 0.03 & 0.03 & 0.03 &     & 0.49 &  0.10 & 0.07 & 0.06 & 0.04 & 0.04\\
\hline
\end{tabular} 
\label{tab:auc_irm}
\caption{Classification error for discriminating between $4$s and $9$s for 
the SGT based on, respectively, semi-supervised eigenvectors and global 
eigenvectors. 
The first column from the left encodes the configuration, \emph{e.g.}, 
1:10 interprets as 1 seed and 10 training samples from each class (total of 
22 samples---for the global approach these are all used for training). 
When the seed is well-determined and the number of training samples 
moderate (50:500), then  a single semi-supervised eigenvector is sufficient; 
whereas for less data, we benefit from using multiple semi-supervised 
eigenvectors. All experiments have been repeated 10 times.}
\label{tab:performance}
\end{table*}

\subsubsection{Discriminating between pairs of digits}

Here, we consider the task of discriminating between two digits; and, in 
order to address a particularly challenging task, we work with \emph{$4$s} 
and \emph{$9$s}.
(This is particularly challenging since these two classes tend to overlap 
more than other combinations since, \emph{e.g.}, a closed $4$ can resemble a
$9$ more than an open $4$.)
Hence, we expect that the class separation axis will not be evident in the 
leading global eigenvector, but instead it will be ``buried'' further down 
the spectrum; and we hope to find a ``locally-biased class separation axis''
with locally-biased semi-supervised eigenvectors.
Thus, this example will illustrate how semi-supervised eigenvectors can 
represent relevant heterogeneities in a local subspace of low dimensionality. 
See Table~\ref{tab:performance}, which summarizes our classification results 
based on, respectively, semi-supervised eigenvectors and global eigenvectors, 
when we use the SGT.
See also Figure~\ref{fig:good_seed} and Figure~\ref{fig:bad_seed}, which 
illustrate two realizations for the 1:10 configuration.
In these two figures, the training samples are fixed; and, to demonstrate 
the influence of the seed, we have varied the seed nodes.
In particular, in Figure~\ref{fig:good_seed} the seed nodes $s_+$ and $s_-$ 
are located well-within the respective classes; while in 
Figure~\ref{fig:bad_seed}, they are located much closer to the boudary 
between the two classes.
As intuitively expected, when the seed nodes fall well within the classes to 
be differentiated, the classification is much better than when the seed 
nodes are located closer to the boundary between the two classes.
See the caption in these figures for further details.

\begin{figure*}[hbt!]
\begin{minipage}[t]{0.24\linewidth}
\centering
\vspace{0pt} 
\begin{overpic}[scale=0.3]
{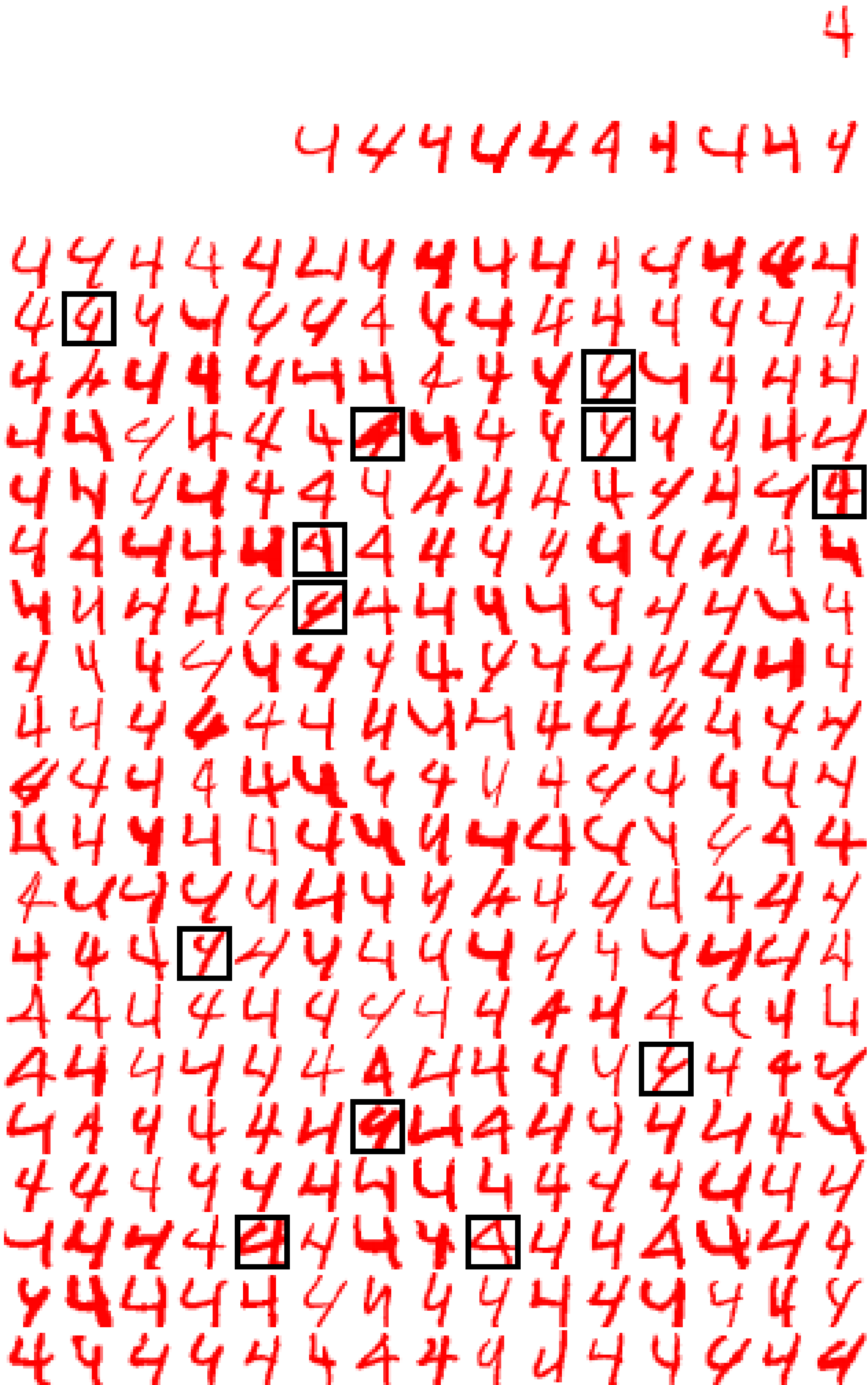}
\put(43,96.5){\scriptsize $s_+=\{$} 
\put(62,96.5){\scriptsize $\}$}
\put(-5,0){\scriptsize \begin{sideways} $\xleftarrow{\hspace*{1.5cm}}$ Test data $\xrightarrow{\hspace*{1.5cm}}$ \end{sideways}}
\put(6,88.5){\scriptsize $l_+=\{$}
\put(62,88.5){\scriptsize $\}$}
\end{overpic}
\end{minipage}
\begin{minipage}[t]{0.24\linewidth}
\centering
\vspace{0pt} 
\begin{overpic}[scale=0.3]
{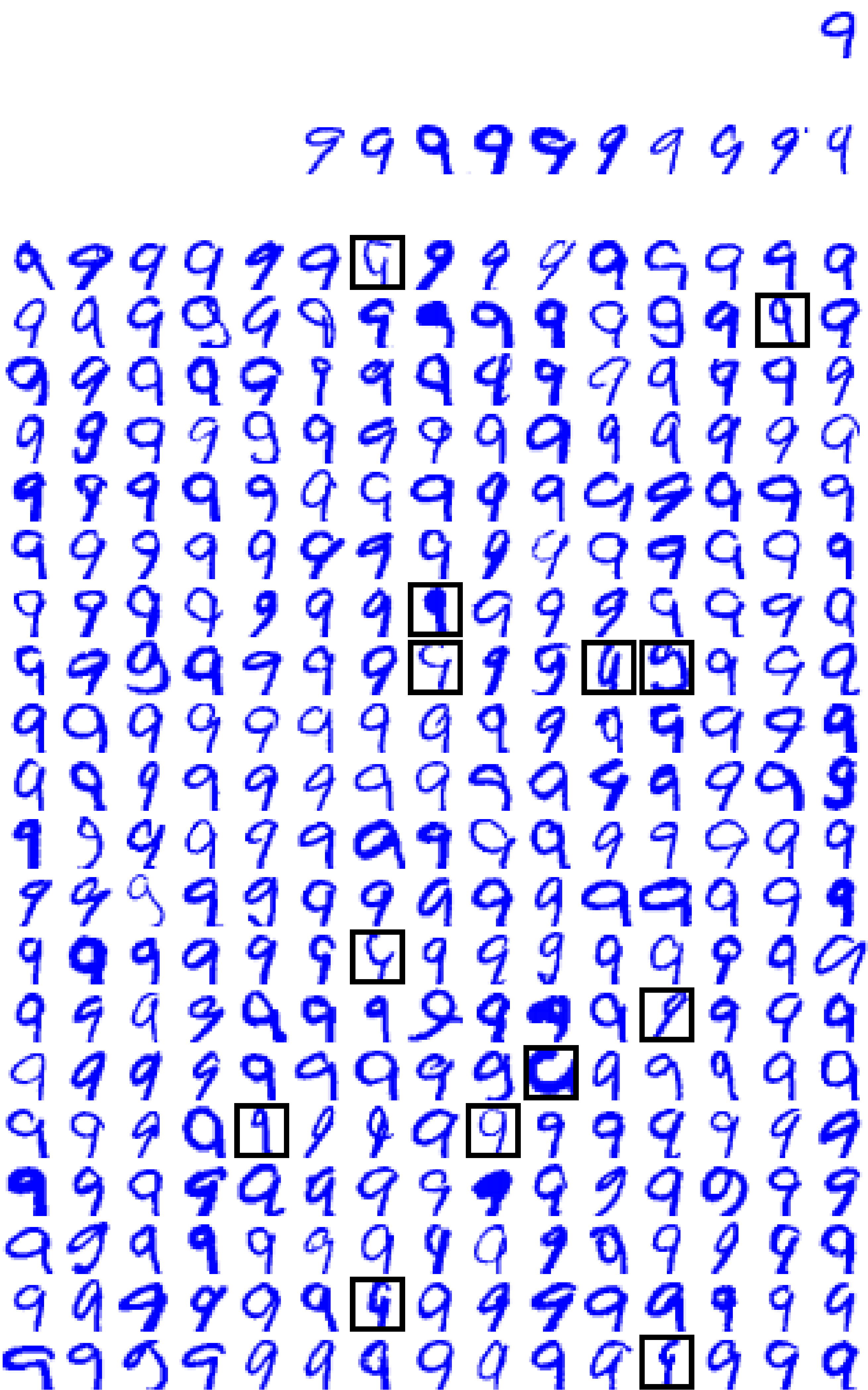}
\put(43,96.5){\scriptsize $s_-=\{$} 
\put(62,96.5){\scriptsize $\}$}
\put(6,88.5){\scriptsize $l_-=\{$}
\put(62,88.5){\scriptsize $\}$}
\end{overpic}
\end{minipage}
\begin{minipage}[t]{0.49\linewidth}
\centering
\vspace{0pt} 
\includegraphics[scale=0.38]{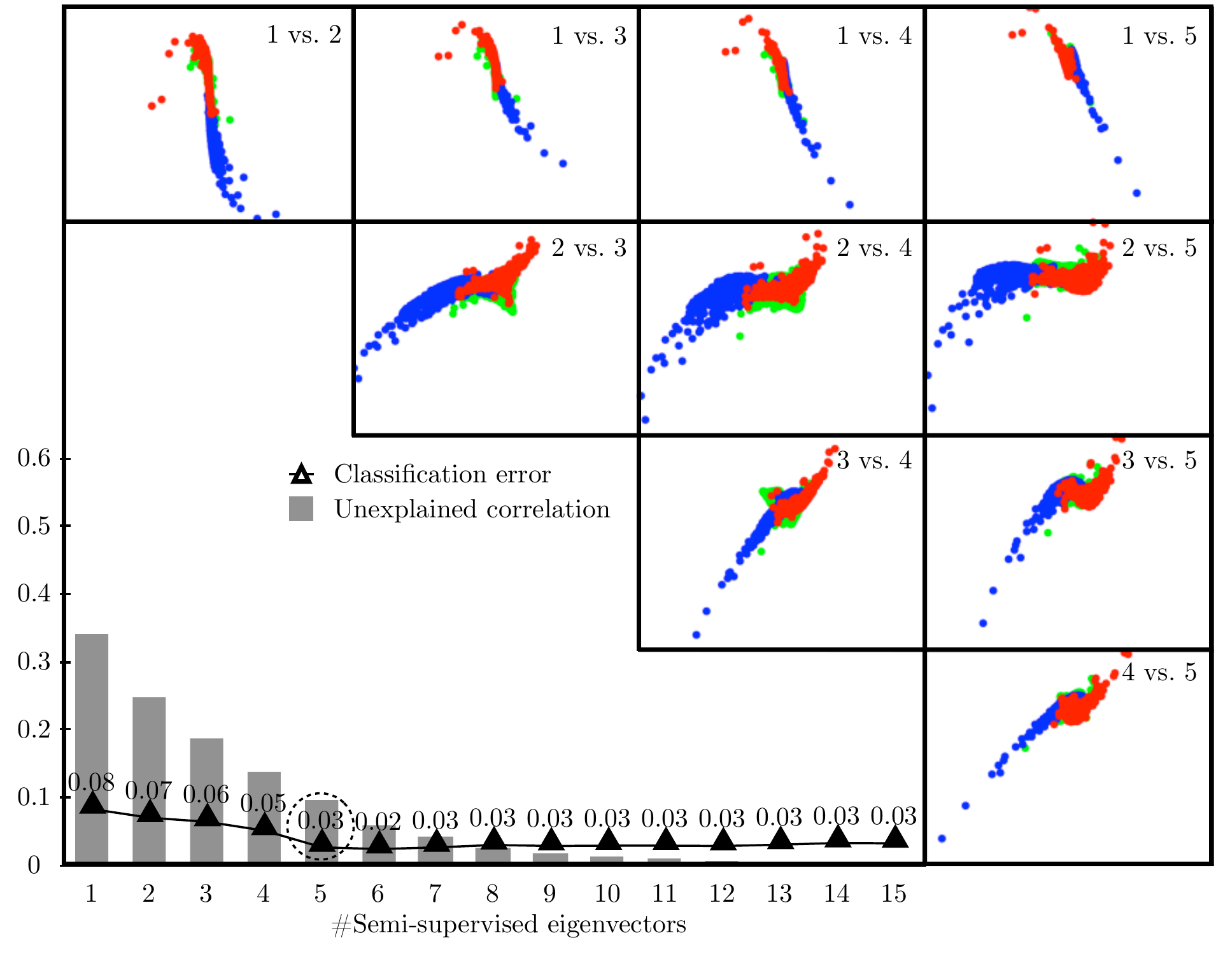}
\end{minipage}
\caption{Discrimination between $4$s and $9$s.  
Left: Shows a subset of the classification results for the SGT based on 5 semi-supervised eigenvectors seeded in $s_+$ and $s_-$, and trained using samples $l_+$ and $l_-$. Misclassifications are marked with black frames. 
Right: Visualizes all test data spanned by the first 5 semi-supervised eigenvectors, by plotting each component as a function of the others. Red (blue) points correspond to $4$ ($9$), whereas green points correspond to remaining digits. As the seed nodes are good representatives, we note that the eigenvectors provide a good class separation. We also plot the error as a function of local dimensionality, as well as the unexplained correlation, \emph{i.e.}, initial components explain the majority of the correlation with the seed (effect of $\gamma=0$). The particular realization based on the leading 5 semi-supervised eigenvectors yields an error of $\approx 0.03$ (dashed circle).}
\label{fig:good_seed} 
\end{figure*}

\begin{figure*}[hbt!]
\begin{minipage}[t]{0.24\linewidth}
\centering
\vspace{0pt} 
\begin{overpic}[scale=0.30]
{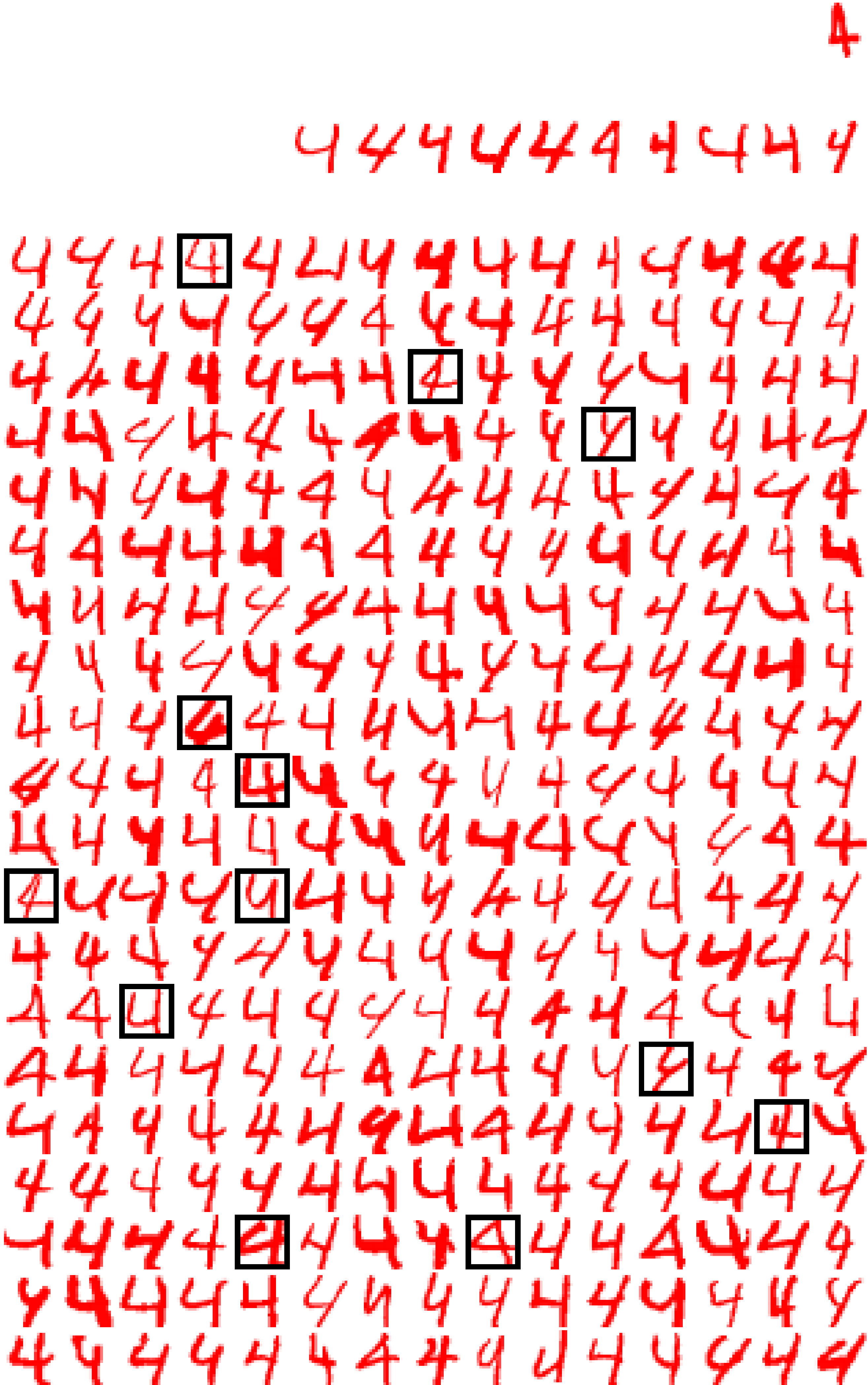}
\put(43,96.5){\scriptsize $s_+=\{$} 
\put(62,96.5){\scriptsize $\}$}
\put(-5,0){\scriptsize \begin{sideways} $\xleftarrow{\hspace*{1.5cm}}$ Test data $\xrightarrow{\hspace*{1.5cm}}$ \end{sideways}}
\put(6,88.5){\scriptsize $l_+=\{$}
\put(62,88.5){\scriptsize $\}$}
\end{overpic}
\end{minipage}
\begin{minipage}[t]{0.24\linewidth}
\centering
\vspace{0pt} 
\begin{overpic}[scale=0.30]
{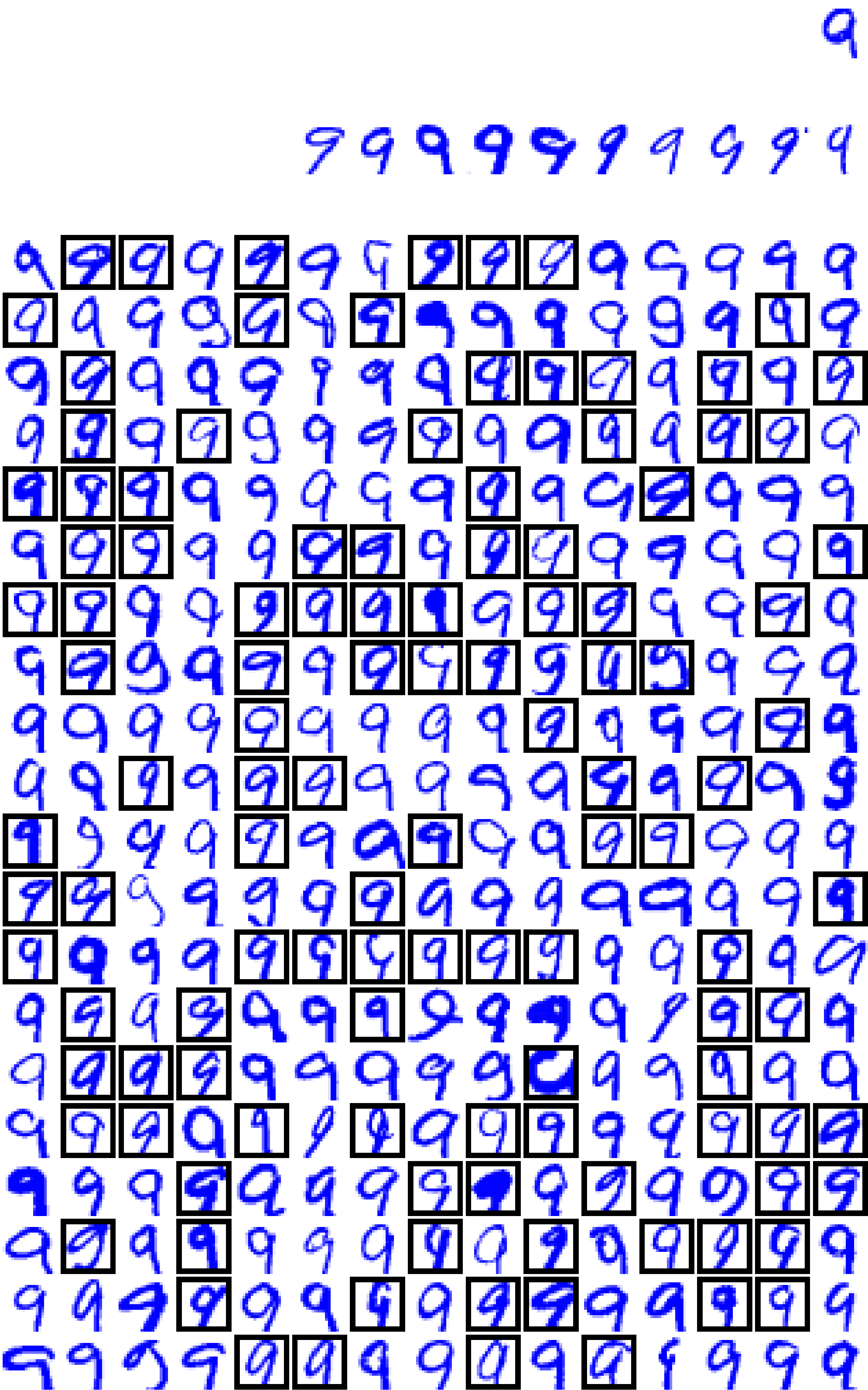}
\put(43,96.5){\scriptsize $s_-=\{$} 
\put(62,96.5){\scriptsize $\}$}
\put(6,88.5){\scriptsize $l_-=\{$}
\put(62,88.5){\scriptsize $\}$}
\end{overpic}
\end{minipage}
\begin{minipage}[t]{0.49\linewidth}
\centering
\vspace{0pt} 
\includegraphics[scale=0.38]{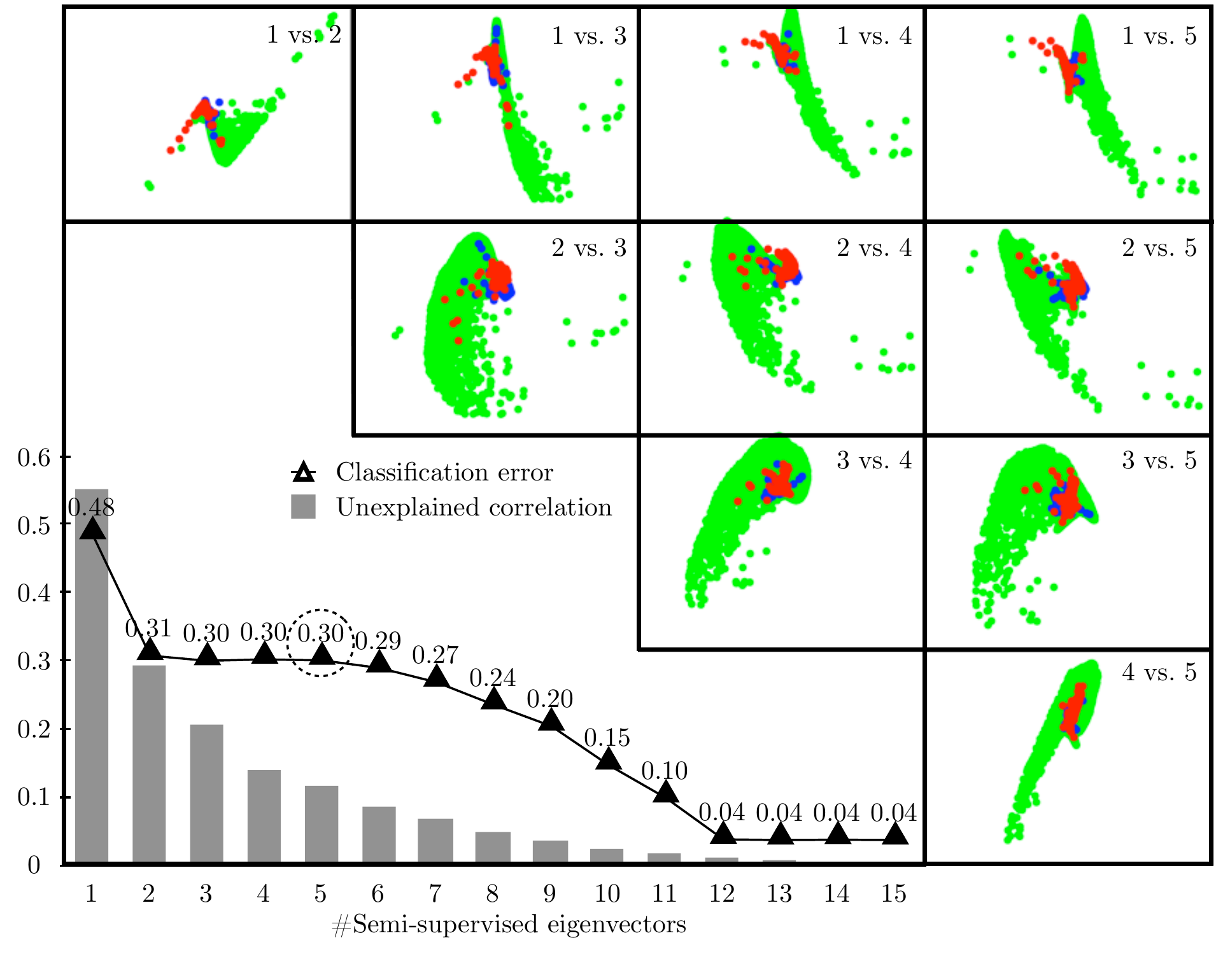}
\end{minipage}
\caption{Discrimination between $4$s and $9$s.  See the general description in Figure \ref{fig:good_seed}. Here we illustrate an instance where the $s_+$ shares many similarities with $s_-$, \emph{i.e.}, $s_+$ is on the boundary of the two classes. This particular realization achieves a classification error of $\approx 0.30$ (dashed circle). In this constellation we first discover localization on low order semi-supervised eigenvectors ($\approx 12$ eigenvectors), which is comparable to the error based on global eigenvectors (see Table \ref{tab:performance}), \emph{i.e.}, further down the spectrum we recover from the bad seed and pickup the relevant mode of~variation.}
\label{fig:bad_seed}
\end{figure*}
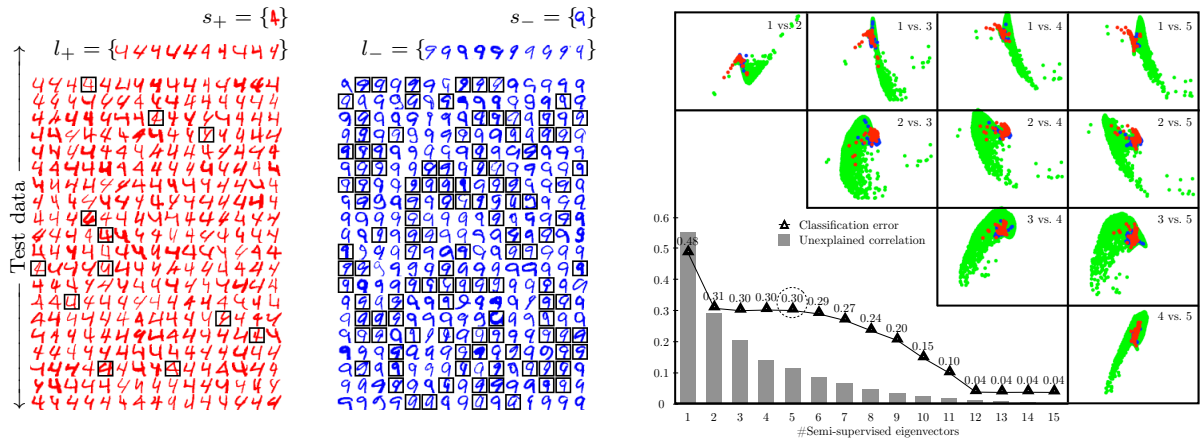

\subsubsection{Effect of choosing the $\kappa$ correlation/locality parameter}
\label{sec:kappa_impact}

Here, we discuss the effect of the choice of the correlation/locality 
parameter $\kappa$ at different steps of Algorithm~\ref{alg_new_1}, 
\emph{e.g.}, how $\{\kappa_t\}_{t=1}^{k}$ should be distributed among the 
$k$ components.
For example, will the downstream classifier benefit the most from a uniform 
distribution or will there exist some other nonuniform distribution that is 
better? 
Although this will be highly problem specific, one might hope that in 
realistic applications the classification performance is not too sensitive 
to the actual choice of distribution.  
To investigate the effect in our example of discriminating between $4$s and 
$9$s, we consider $3$ semi-supervised eigenvectors for various $\kappa$ 
distributions.
Our results are summarized in Figure~\ref{fig:dirichlet_kappa}.

Figures~\ref{fig:dirichlet_kappaA}, \ref{fig:dirichlet_kappaB}, 
and~\ref{fig:dirichlet_kappaC} show, for the global eigenvectors and for 
semi-supervised eigenvectors, where the $\kappa$ vector has been chosed to 
be very nonuniform and very uniform, the top three (global or 
semi-supervised) eigenvectors plotted against each other as well as the ROC 
curve for the SGT classifier discriminating between $4$s and $9$s; and 
Figure~\ref{fig:dirichlet_kappaD} shows the test error as the $\kappa$ vector 
is varied over the unit simplex.
In more detail, red (respectively, blue) corresponds to $4$s (respectively, 
$9$s), and green points are the remaining digits; and, for 
Figures~\ref{fig:dirichlet_kappaB} and~\ref{fig:dirichlet_kappaC}, the 
semi-supervised eigenvectors are seeded using 50 samples from each target 
class ($4$s vs. $9$s) and having a non-uniform distribution of $\kappa$, as 
specified. 
As seen from the visualization of the semi-supervised eigenvectors in 
Figures~\ref{fig:dirichlet_kappaB} and~\ref{fig:dirichlet_kappaC}, the 
classes are much better separated than by using the global eigenvectors, 
which are shown in Figure~\ref{fig:dirichlet_kappaA}.
For example, this is supported by the Area Under the Curve (AUC) and Error 
Rate (ERR), being the point on the Receiver Operating Characteristic (ROC) 
curve that corresponds to having an equal probability of miss-classifying a 
positive or negative sample, which is a fair estimate as the classes in the 
MNIST data set is fairly balanced. 
For Figure~\ref{fig:dirichlet_kappaC}, where we use a uniform distribution 
of $\kappa$, the classifier performs slightly better than in 
Figure~\ref{fig:dirichlet_kappaB}, which uses the non-uniform $\kappa$ 
distribution (but both semi-supervised approaches are significantly better 
than the using the global eigenvectors). 
For Figure~\ref{fig:dirichlet_kappaD}, we see the test error on the simplex 
defined by $\kappa$. 
To obtain this plot we sampled 500 different $\kappa$ distributions 
according to a uniform Dirichlet distribution.
With the exception of one extreme very nonuiform corner, the classification 
accuracy is not too sensitive to the choice of $\kappa$ distribution.
Thus, if we think of the semi-supervised eigenvectors as a 
locally-regularized version of the global eigenvectors, the desired 
discriminative properties are not too sensitive to the details of the 
locally-biased regularization.


\begin{figure*}[hbt!]
\subfigure[Global eigenvectors]{
  \begin{minipage}[b]{0.49\linewidth}
  \centering
  \includegraphics[scale=0.35]{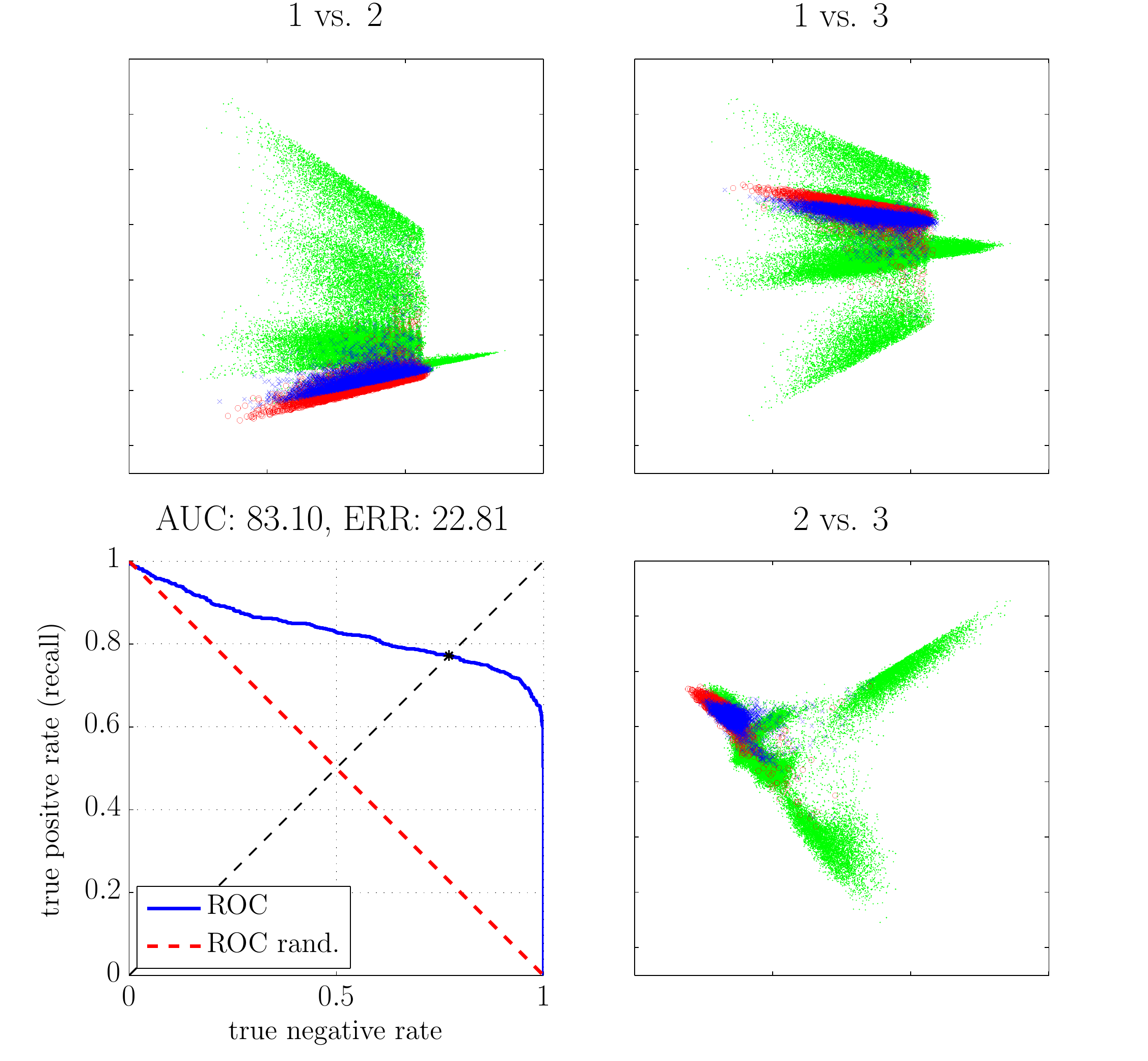}
  \end{minipage}
  \label{fig:dirichlet_kappaA}
}
\subfigure[
  Semi-supervised, with $\kappa^{(1)} = ( \frac{1}{10}, \frac{1}{10}, \frac{8}{10} )$
  ]{
  \begin{minipage}[b]{0.49\linewidth}
  \centering
  \includegraphics[scale=0.35]{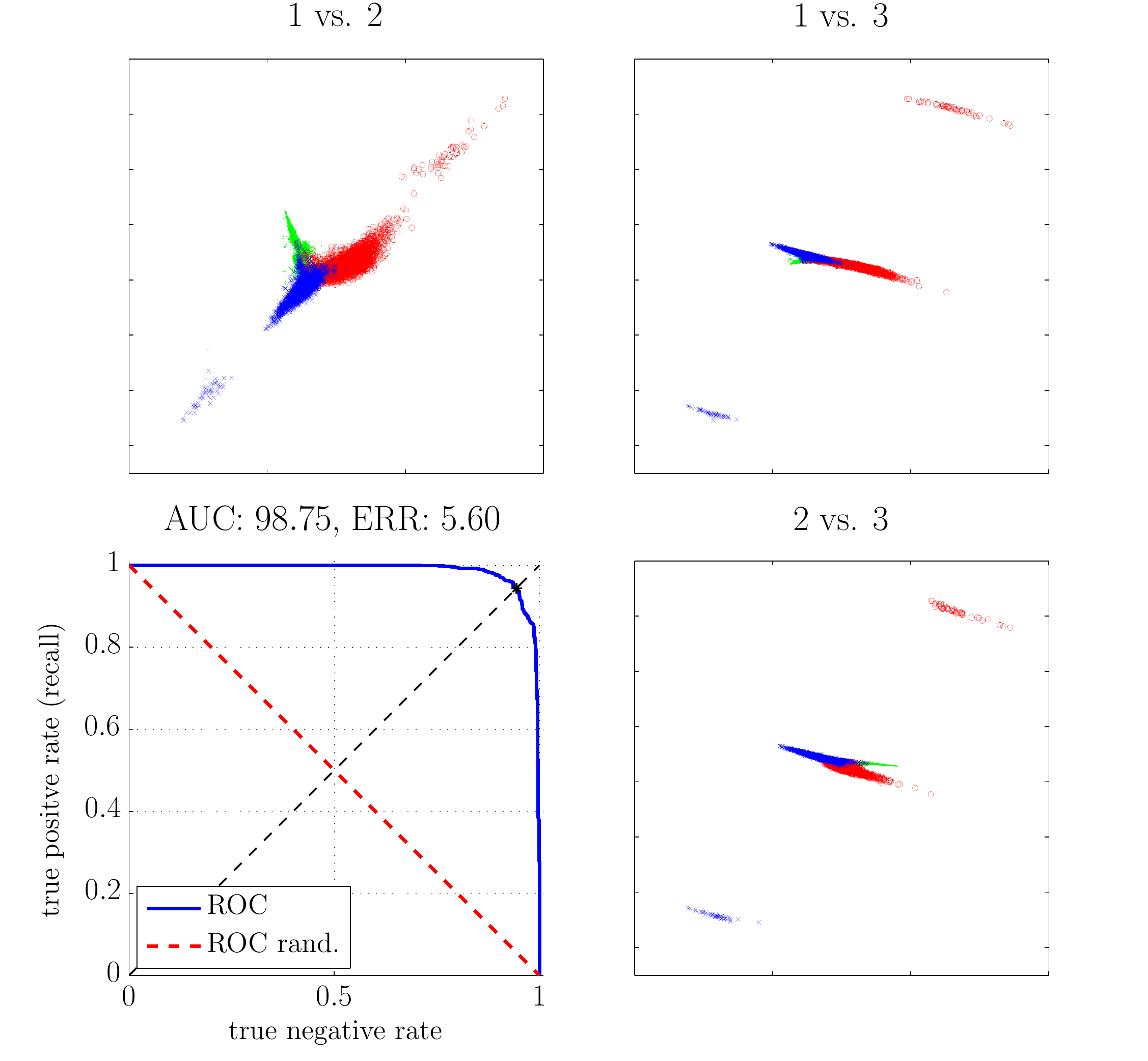}
  \end{minipage}
  \label{fig:dirichlet_kappaB}
}
\subfigure[
  Semi-supervised, with $\kappa^{(2)} = ( \frac{1}{3}, \frac{1}{3}, \frac{1}{3} )$
  ]{
  \begin{minipage}[b]{0.49\linewidth}
  \centering
  \includegraphics[scale=0.35]{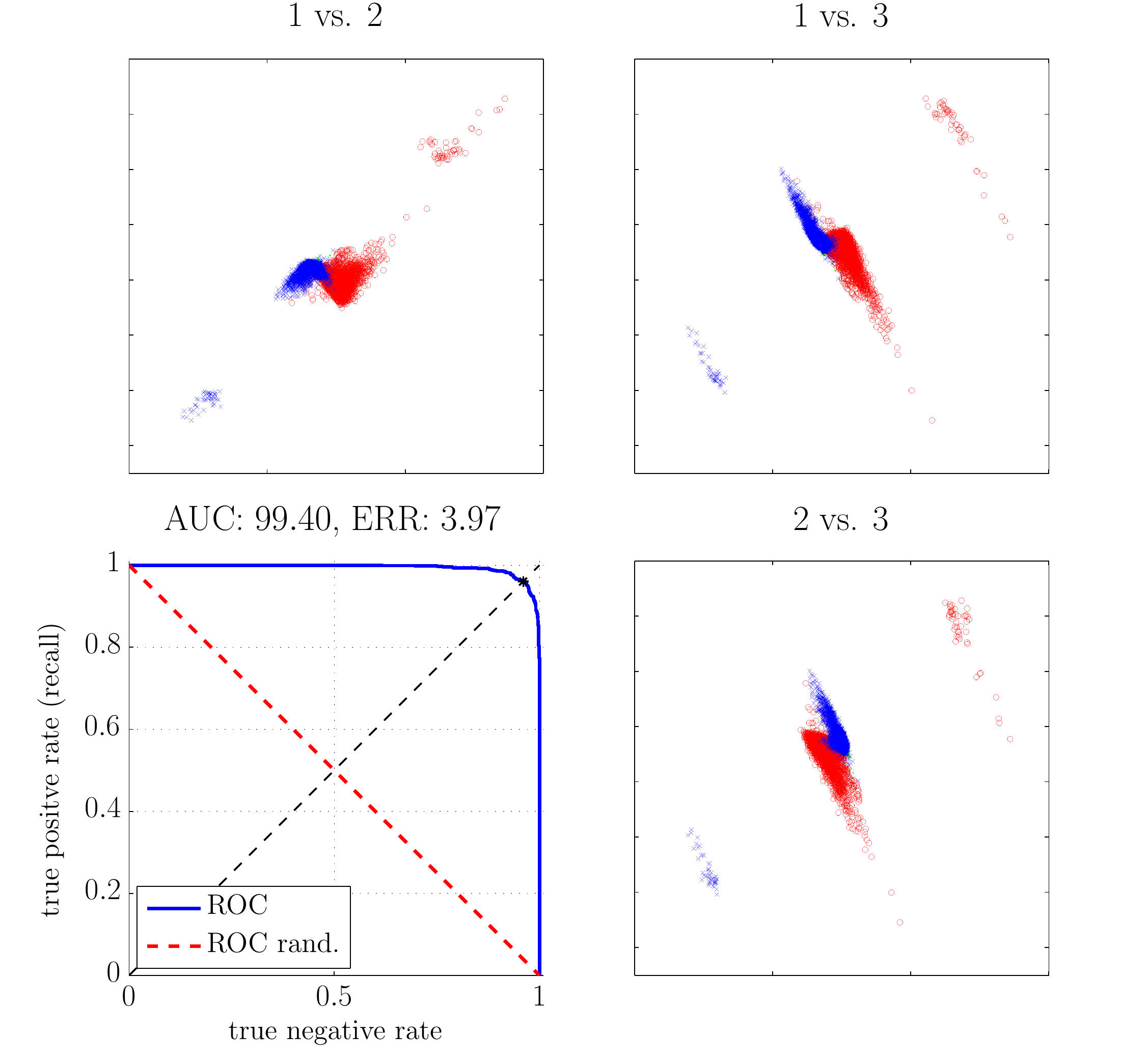}
  \end{minipage}
  \label{fig:dirichlet_kappaC}
}
\subfigure[Test error on the $\kappa$ simplex]{
  \begin{minipage}[b]{0.49\linewidth}
  \centering
  \includegraphics[scale=1]{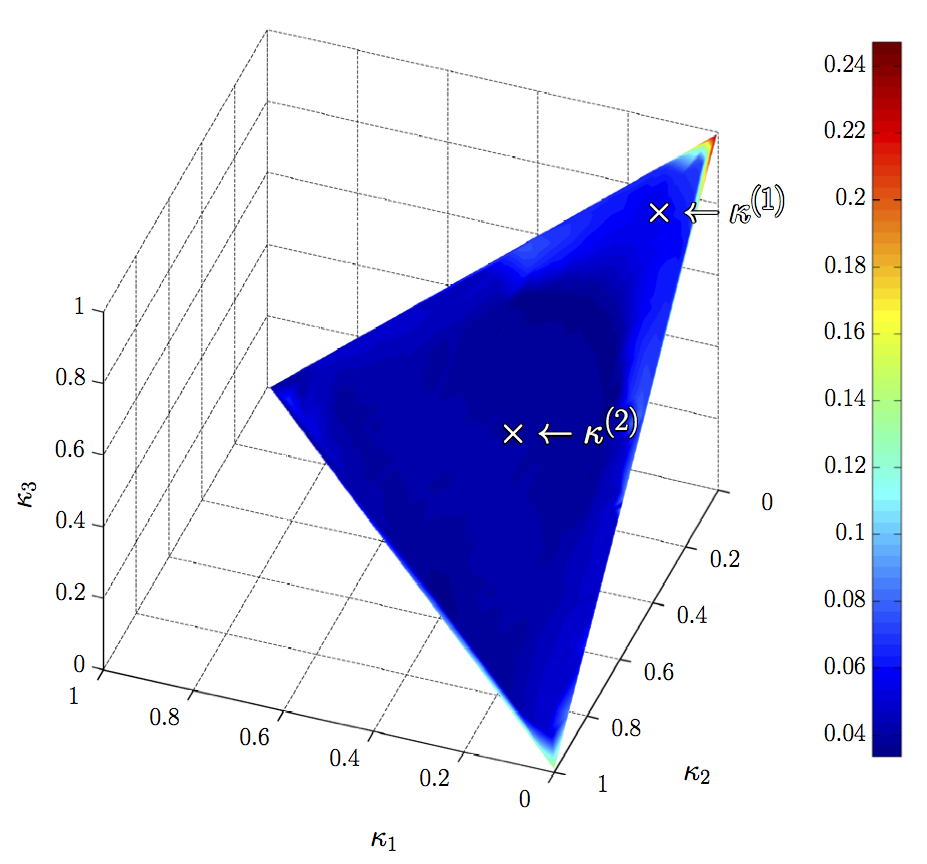}
  \end{minipage}
  \label{fig:dirichlet_kappaD}
}
\caption{The effect of varying the correlation/locality parameter $\kappa$ 
on the classification accuracy.
\ref{fig:dirichlet_kappaA}, \ref{fig:dirichlet_kappaB}, 
\ref{fig:dirichlet_kappaC} show the top three (global or semi-supervised)
eigenvectors plotted against each other as well as the ROC curve for the 
SGT classifier discriminating between $4$s and $9$s; and 
\ref{fig:dirichlet_kappaD} shows the test error as the $\kappa$ vector 
is varied over the unit simplex.
}
\label{fig:dirichlet_kappa}
\end{figure*}

\subsubsection{Effect of approximately computing semi-supervised eigenvectors}\label{sec:mnistpeeling}

Here, we discuss of the push-peeling procedure from 
Section \ref{sxn:main-alg-extensions} that is designed to compute efficient 
approximations to the semi-supervised eigenvectors by using local random 
walks to compute an approximation to personalized PageRank vectors. 
Consider Figure~\ref{fig:peeling}, which shows results for two values of 
the $\epsilon$ parameter (\emph{i.e.}, the parameter in the push algorithm 
that implicitly determines how many nodes will be touched).  
Again we construct the full $70,000\times 70,000$ $k$-NN graph, with $k=10$ and 
with edge weights given by 
$w_{ij}=\exp({-\frac{4}{\sigma_i^2}\| x_i- x_j\|^2})$, where $\sigma_i^2$ 
is the Euclidian distance of the $i^{th}$ node to it's nearest neighbor; and 
from this we define the graph Laplacian in the usual way.
Using this representation we compute $3$ semi-supervised eigenvectors 
seeding using 50 samples from each class ($4$s vs. $9$s). 
However, in this case, we fix the regularization parameter vector as 
$\gamma=[-0.0150,-0.0093,-\text{vol}(G)]$; and note that choosing these specific values correspond to the solutions visualized in Figure \ref{fig:dirichlet_kappaC} when the 
equations are solved exactly. 
Figure~\ref{fig:peelingA} shows the results for $\epsilon=0.001$. 
This approximation gives us sparse solutions, and the histogram in the 
second row illustrates the digits that are assigned a nonzero value in the 
respective semi-supervised eigenvector.
In particular, note that most of the mass of the eigenvector is distributed 
on $4$s and $9$s; but, for this choose of $\epsilon$, only few digits of 
interest ($\approx2.8243\%$, meaning, in particular, that not all of the $4$s
and $9$s) have been touched by the algorithm.
This results in the lack of a clean separation between the two classes as 
one sweeps along the leading semi-supervised eigenvector, as illustrated in 
the first row; the very uniform correlation distribution 
$\kappa=[0.8789,0.0118,0.1093]$; and the high classification error, as shown 
in the ROC curve in the bottom panel. 

Consider, next, Figure~\ref{fig:peelingB}, which shows the results for 
$\epsilon=0.0001$, \emph{i.e.}, where the locality parameter $\epsilon$ has 
been reduced by an order of magnitude.
In this case, the algorithm reproduces the solution by touching only 
$\approx 25.177 \%$ of the nodes in the graph, \emph{i.e.}, basically all of 
the $4$s and $9$s and only a few other digits.
This leads to a much cleaner separation between the two classes as one 
sweeps over the leading semi-supervised eigenvector; a much more uniform 
distribution over $\kappa$; and a classification accuracy that is much 
better and is similar to what we saw in Figure~\ref{fig:dirichlet_kappaC}. 
This example illustrates that this push-peeling approximation provides a 
principled manner to generalize the concept of semi-supervised eigenvectors 
to large-scale settings, where it will be infeasible to touch all nodes of 
the graph.

\begin{figure*}[hbt!]
\subfigure[Locality parameter $\epsilon=0.001$]{
\begin{minipage}[b]{0.49\linewidth}
\centering
Test data sorted along $x_1$.
\includegraphics[scale=0.35]{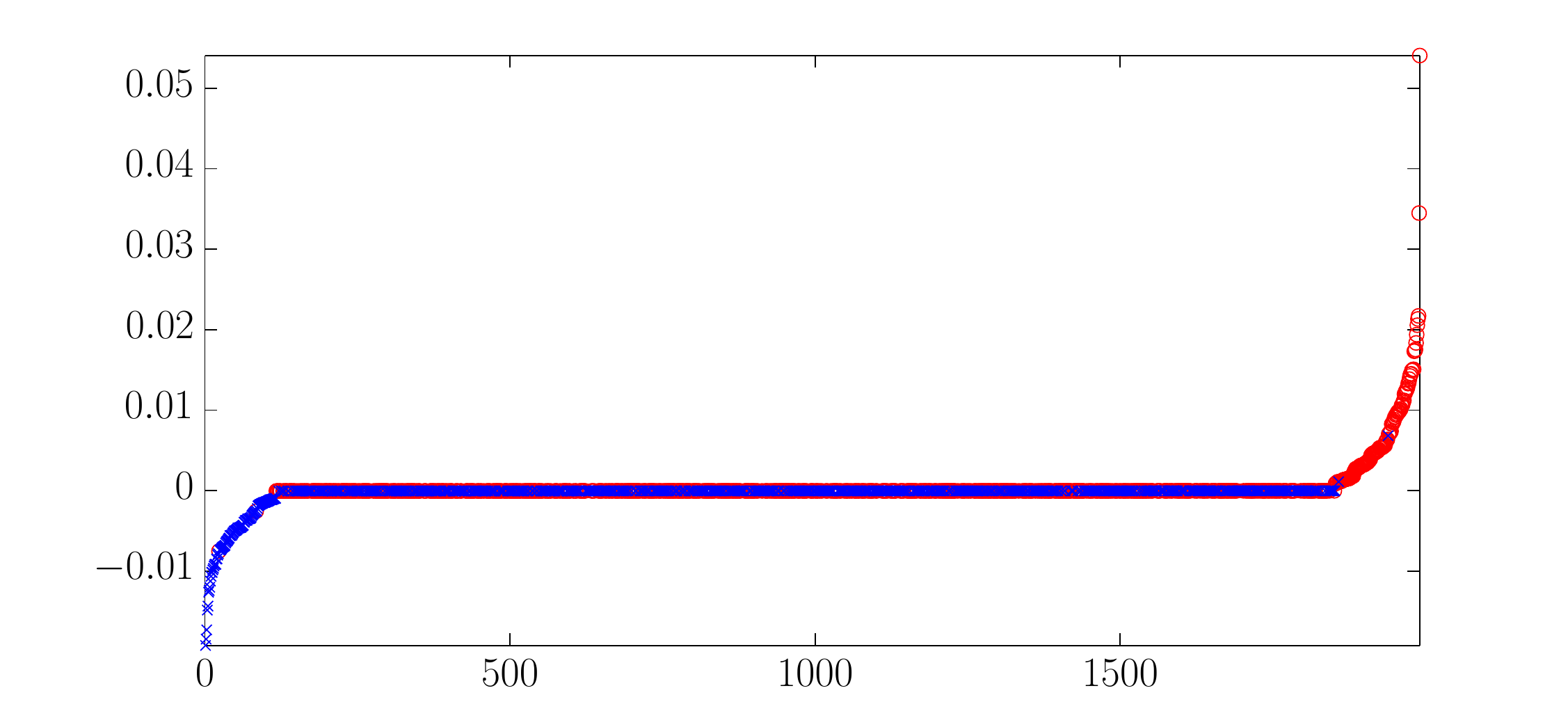}
Touched $\approx 2.824 \%$ of nodes. 
\includegraphics[scale=0.35]{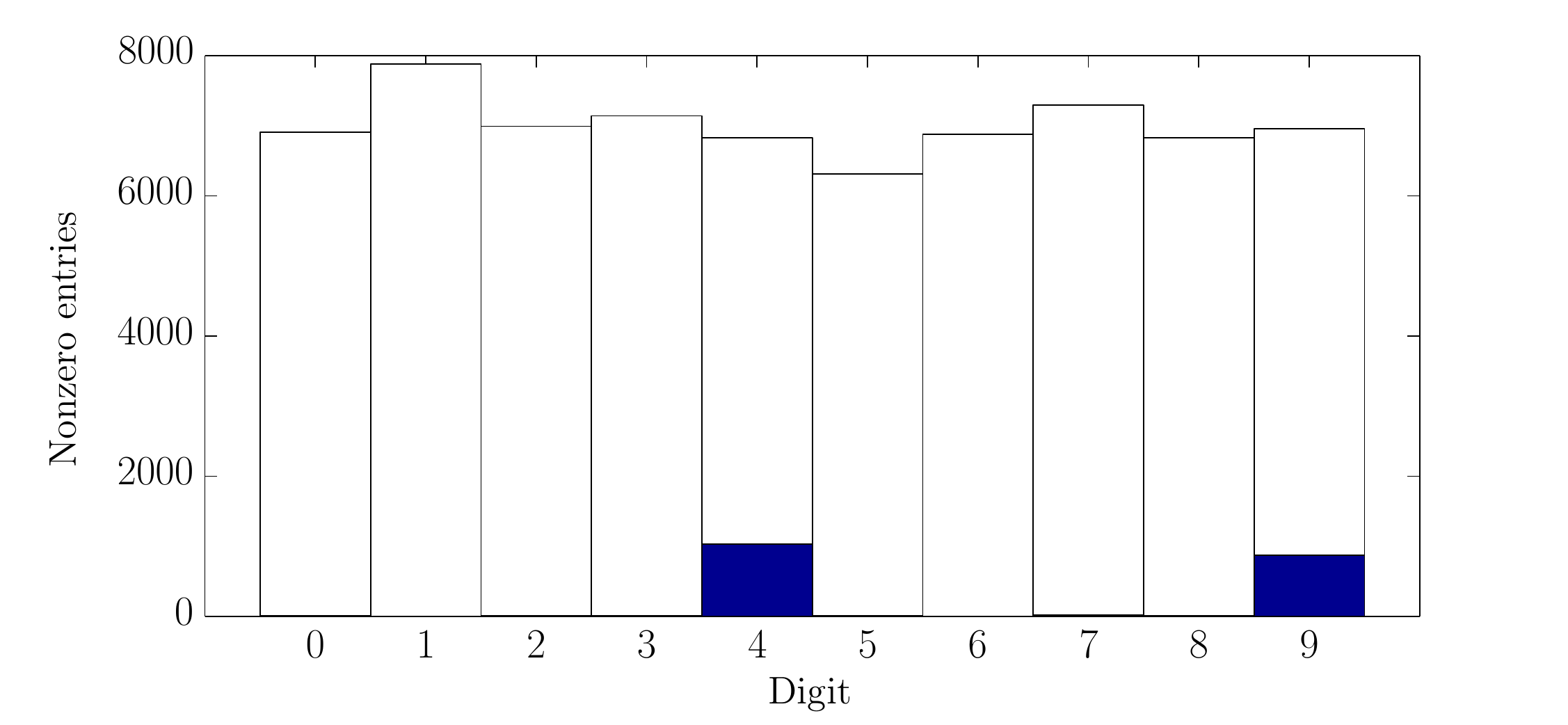}
$\kappa=[0.8789,0.0118,0.1093]$
\includegraphics[scale=0.35]{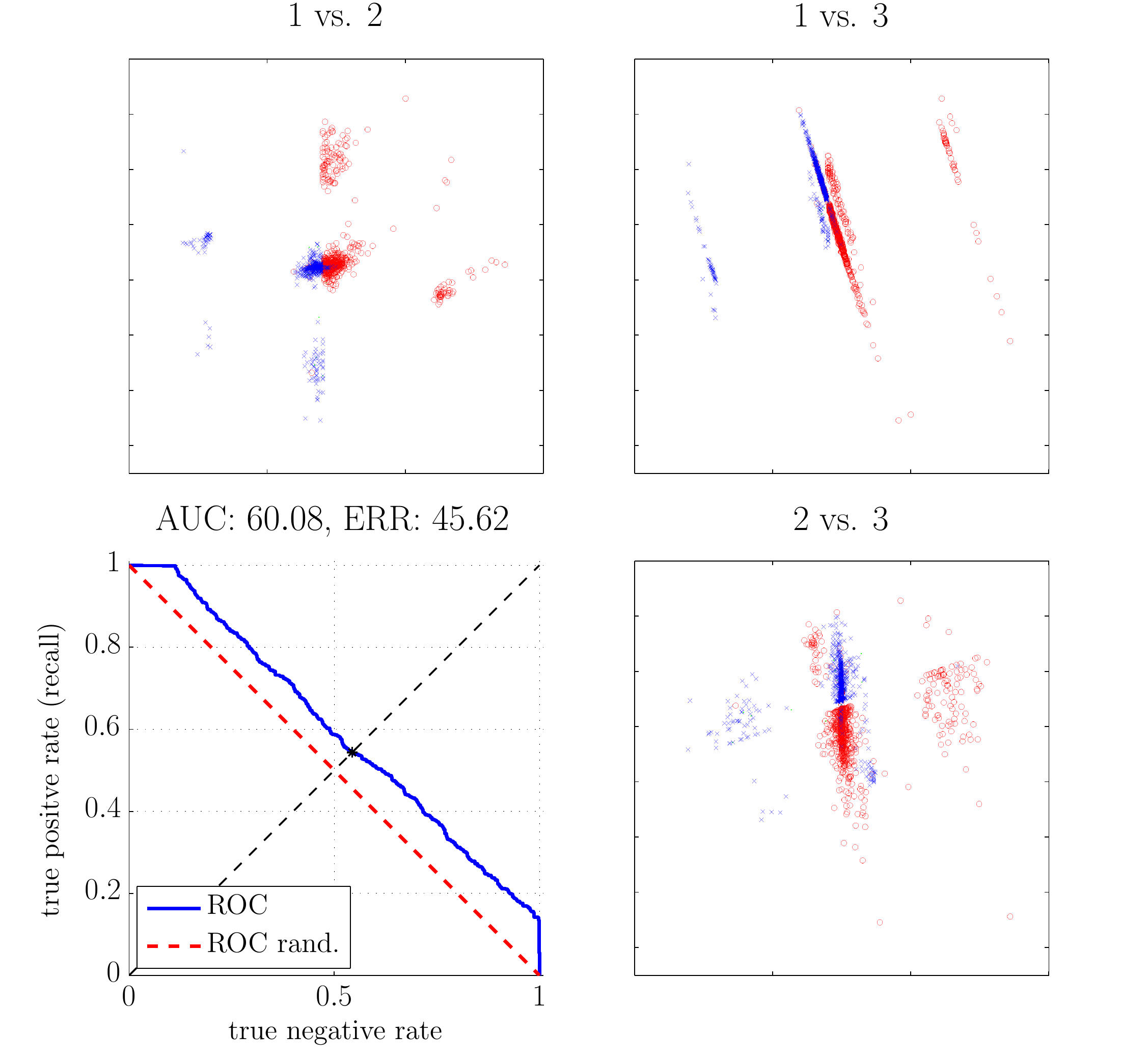}
\end{minipage}
\label{fig:peelingA}
}
\subfigure[Locality parameter $\epsilon=0.0001$]{
\begin{minipage}[b]{0.49\linewidth}
\centering
Test data sorted along $x_1$.
\includegraphics[scale=0.35]{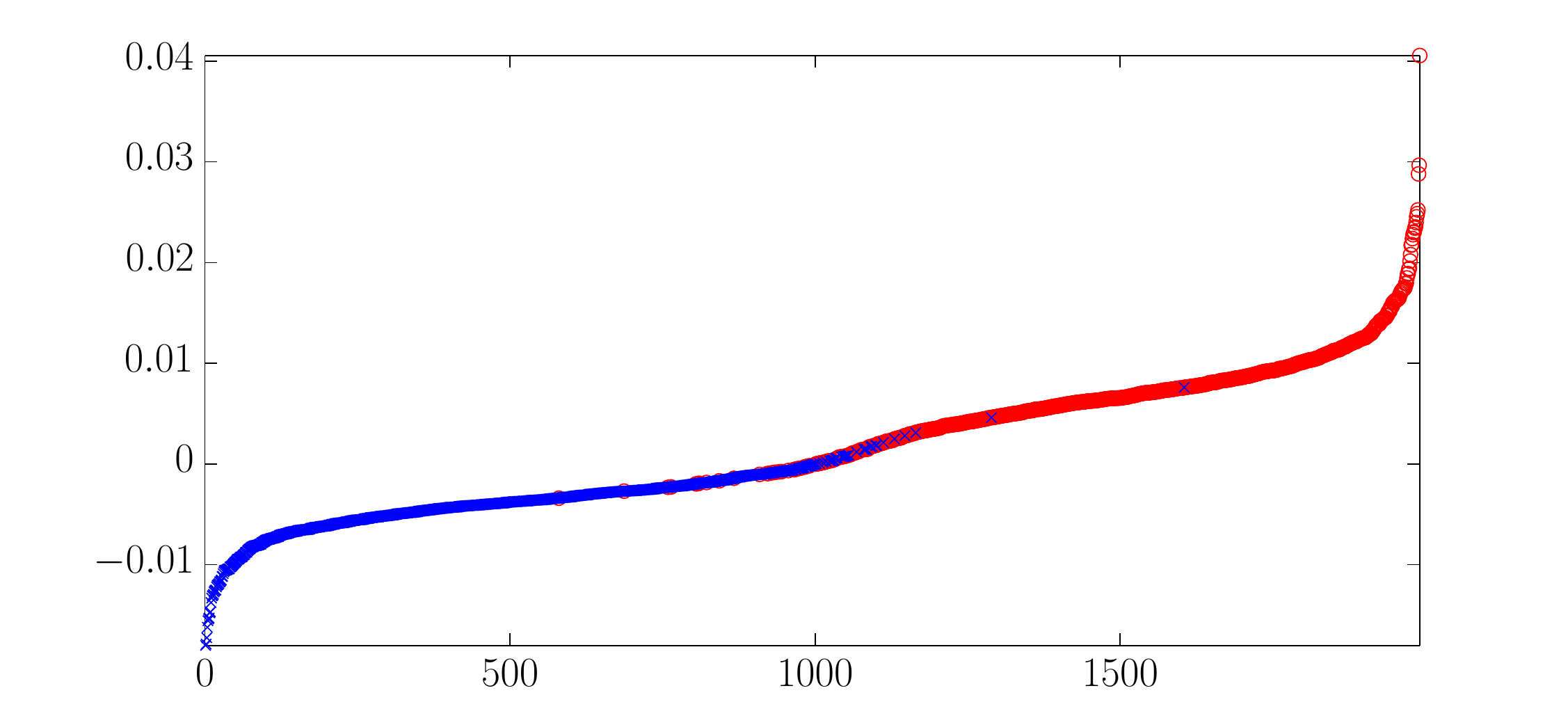}
Touched $\approx 25.177 \%$ of nodes. 
\includegraphics[scale=0.35]{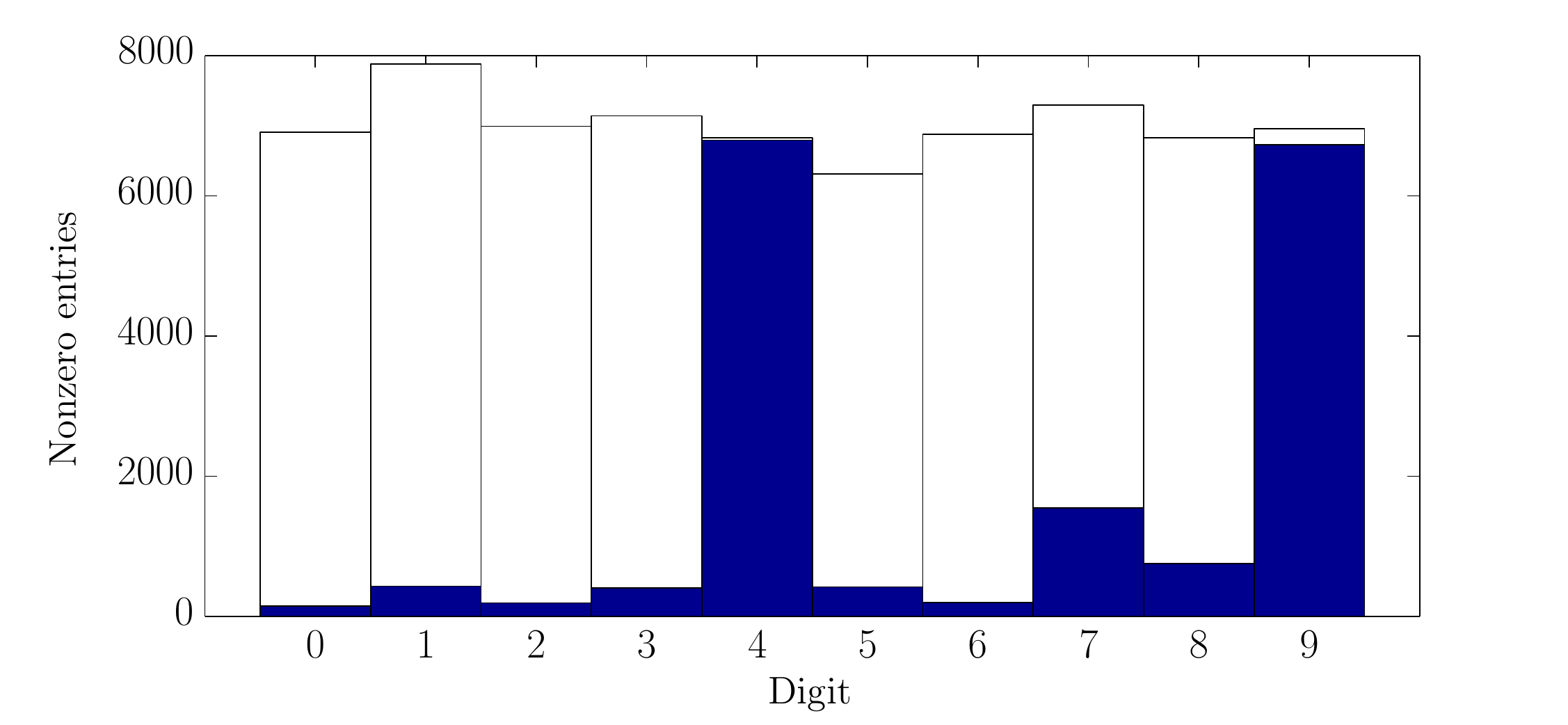}
$\kappa=[0.3333,0.3334,0.3333]$
\includegraphics[scale=0.35]{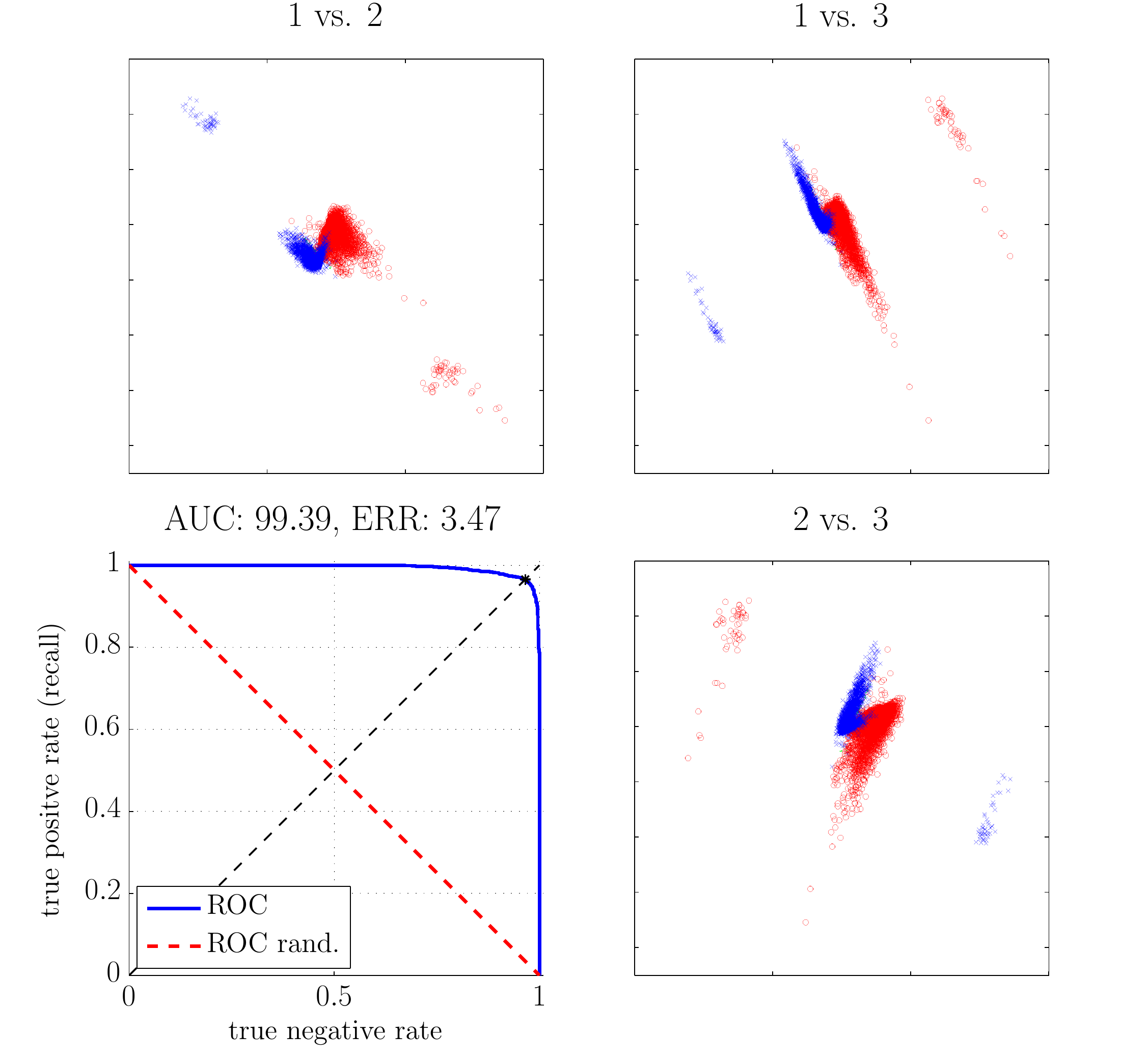}
\end{minipage}
\label{fig:peelingB}
}
\caption{Illustration of the push-peeling procedure to compute $3$ 
semi-supervised eigenvectors for $\gamma=[-0.0150,-0.0093,-\text{vol}(G) ]$. 
\ref{fig:peelingA} shows results for $\epsilon=0.001$; and
\ref{fig:peelingB} shows results for $\epsilon=0.0001$. 
First row shows the entries in the leading semi-supervised eigenvector 
corresponding to test points, color-coded and sorted according to magnitude; 
second row shows the distribution of digits touched in the full graph when 
executing the push algorithm; and
bottom panels provide visualizations similar to the ones in 
Figure~\ref{fig:dirichlet_kappa} (and shown above these is the correlation 
vector $\kappa$ obtained for the fixed choice of $\gamma$.
}
\label{fig:peeling}
\end{figure*}

\subsubsection{Effect of low-rank Nystr\"{o}m approximation}
Here we discuss the use of the low-rank Nystr\"{o}m approximation which is commonly used in large-scale kernel-based machine learning. The memory requirements for representing the explicit kernel matrix, that we here take to be our graph, scales with  
 $\mathcal{O}(N^2)$, whereas inverting the matrix scales with $\mathcal{O}(N^3)$, which, in large-scale settings, is infeasible. The Nystr\"{o}m technique subsamples the dataset to approximate the kernel matrix, and the memory requirements scales with $\mathcal{O}(nN)$ and runs in $\mathcal{O}(n^2N)$, where $n$ is size of the subsample.
For completeness we include the derivation of the Nystr{\"{o}}m approximation for the normalized graph Laplacian in Appendix \ref{app:nystrom}. 

In the beginning of Section \ref{sxn:empirical-digits} we constructed the $70,000\times 70,000$ $k$-nearest neighbor graph, with $k=10$ and 
with edge weights given by $w_{ij}=\exp({-\frac{4}{\sigma_i^2}\| x_i- x_j\|^2})$. Such a sparse construction reduces the effect of ``hubs'', as well as being fairly insensitive to the choice of kernel parameter, as the 10 nearest neighbors are likely to be very close in the Euclidian norm. Because the Nystr{\"{o}}m method will approximate the dense kernel matrix, the choice of kernel parameter is more important, so in the following we will consider the interplay between this parameter, as well as the rank parameter $n$ of the Nystr{\"{o}}m approximation.
Moreover, to allow us to compare a rank-$n$ Nystr{\"{o}}m approximation with the full rank-$N$ kernel matrix, we choose to subsample the dataset for all of the following experiments, due to the  $\mathcal{O}(N^2)$ memory requirements.
Thus, to provide a baseline, Figure \ref{fig:subsample} shows results based 
on a $k$-nearest neighbor graph constructed from $5\%$ and $10\%$ percent of the 
training data, where in both cases we used $10\%$ for the test data.
For both cases, when compared with the results of 
Figure \ref{fig:dirichlet_kappaC}, the classification quality is degraded, and so we emphasize that the goal of the following results are not to outperform the results reported in Figure \ref{fig:dirichlet_kappaC}, but to be comparable with this baseline.

\begin{figure*}[hbt!]
\subfigure[$5\%$ of MNIST]{
\begin{minipage}[b]{0.49\linewidth}
\centering
\includegraphics[scale=0.35]{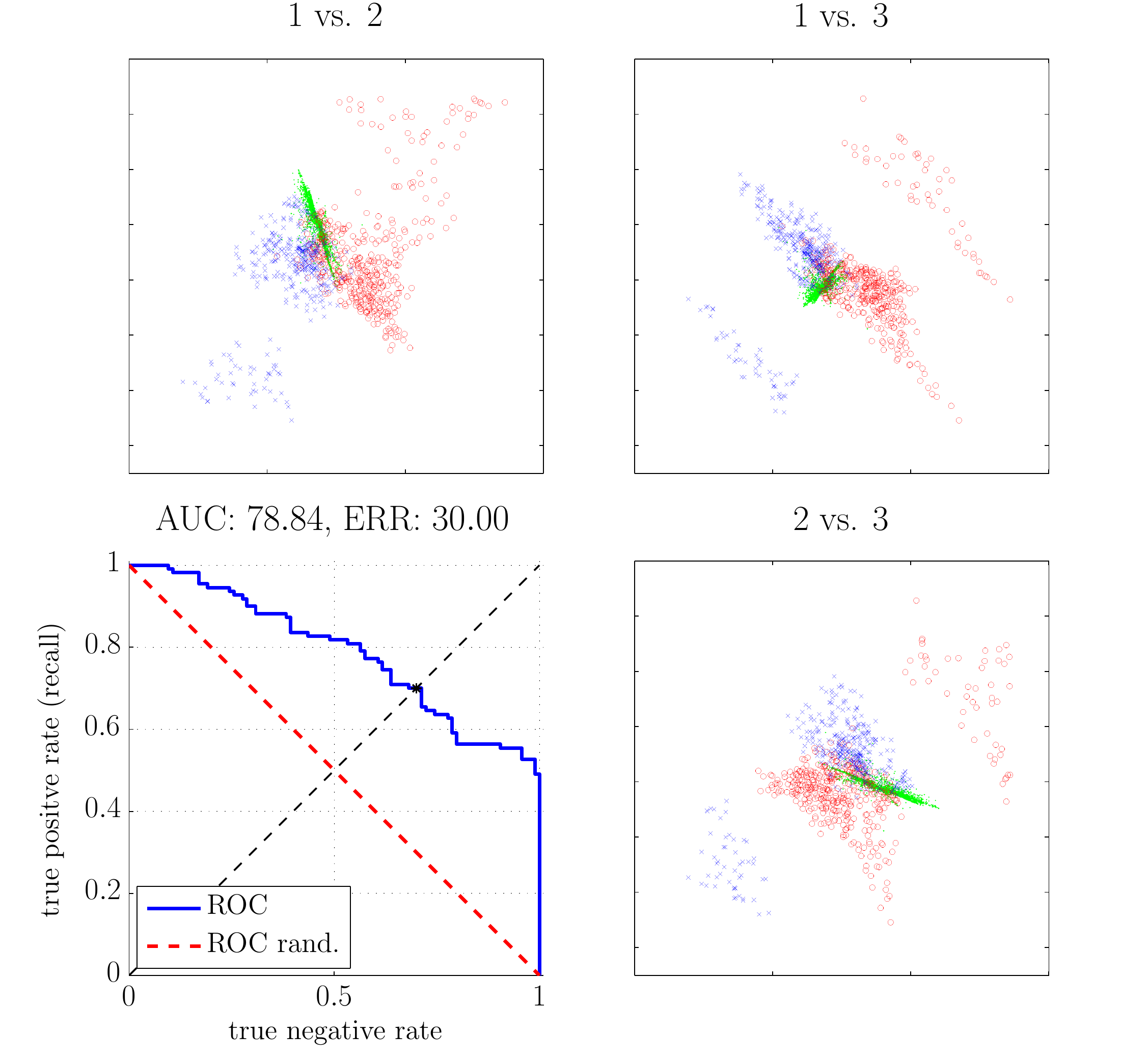}
\end{minipage}
\label{fig:subsampleA}
}
\subfigure[$10\%$ of MNIST]{
\begin{minipage}[b]{0.49\linewidth}
\centering
\includegraphics[scale=0.35]{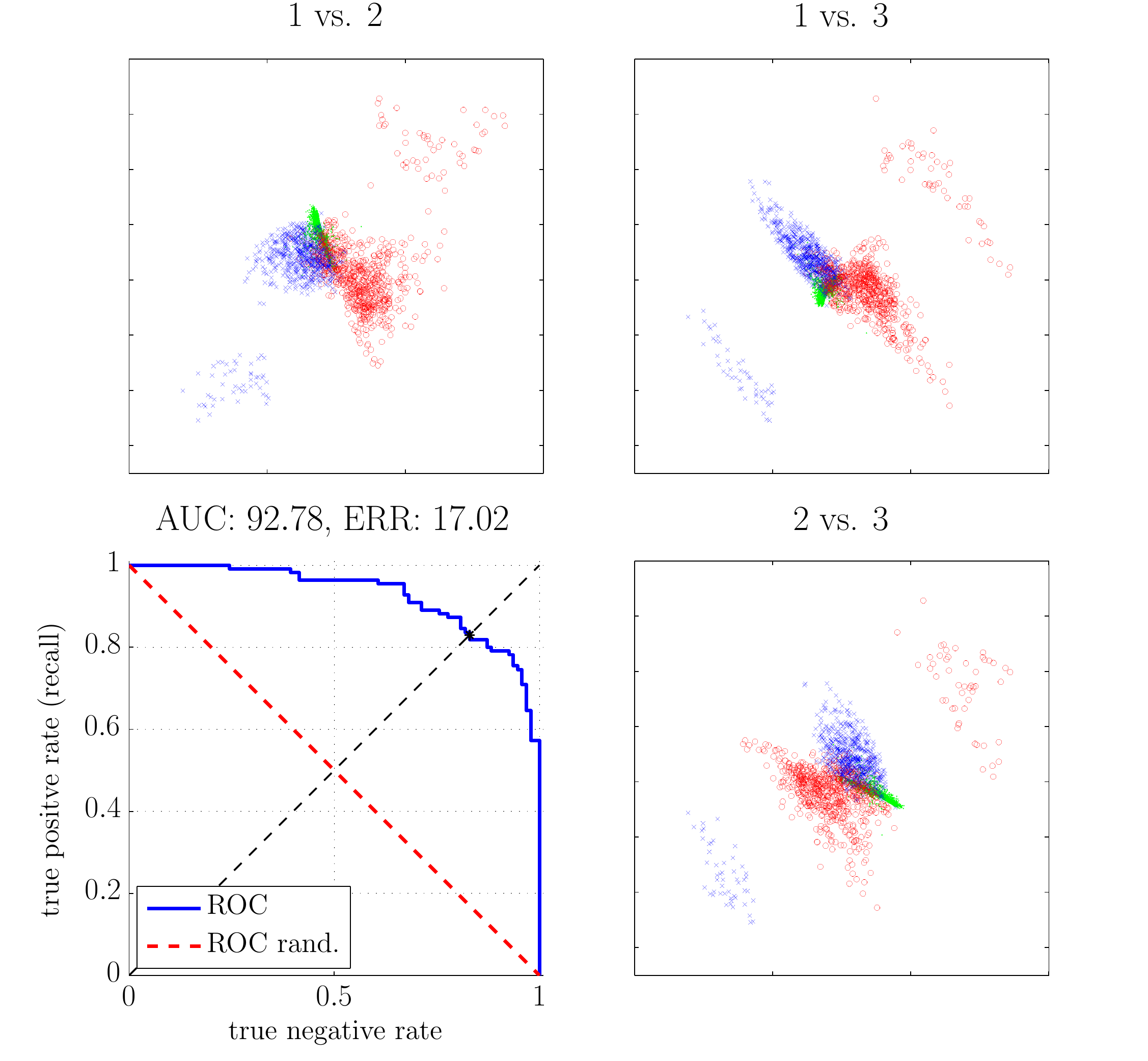}
\end{minipage}
\label{fig:subsampleB}
}
\caption{Example of the impact of subsampling the data set down to 
$5\%$ (in \ref{fig:subsampleA}) and $10\%$ (in \ref{fig:subsampleB}) of 
the original size. 
Remaining parameters are the same as in Figure \ref{fig:dirichlet_kappaC},
which shows the result to which these two plots should be compared.
}
\label{fig:subsample}
\end{figure*}

In light of this baseline, Figure~\ref{fig:subsample_thorough} provides a thorough analysis for 
the choices of $\sigma_i^2$ that we used.
Figures~\ref{fig:subsample_thoroughA} and~\ref{fig:subsample_thoroughB} show
the classification error when using the global eigenvectors, for various rank 
approximations based on the Nystr\"{o}m method as well as the exact method
(corresponding to $\text{rank}=n$). 
Interestingly, these two plots are very dissimilar in terms of their
behavior as a function of the number of components.
In particular, the plot in Figure~\ref{fig:subsample_thoroughB} shows that the 
low rank approximations for a given set of 
components outperform the high rank approximations, and the exact representation 
fails to reduce the error beyond $0.4$ for any of the 
considered set of components. This may seem counterintuitive,
but the reason for this type of behavior is that the relevant global eigenvectors, 
for $\sigma_i^2=200$, are located far from the end of the spectrum (if we visualized 
more components for $\text{rank}=n$ the classification error would
 eventually drop). 
For the same reason, the low rank approximations improve more rapidly than 
the high rank approximations, as the latter approximate the lower part of the 
spectrum better, and these turn out to have poor discriminative properties.
In contrast, the results shown in Figure~\ref{fig:subsample_thoroughA} 
provide good class separation in the lower part of the spectrum, resulting 
in the high rank approximations to reduce the error most rapidly. 

Finally, Figures~\ref{fig:subsample_thoroughC} 
and~\ref{fig:subsample_thoroughD} show the classification error for the 
SGT trained using the semi-supervised eigenvectors.
(Note that the scale of the x-axis is much smaller in these subfigures.)
For both kernel widths (in both Figures~\ref{fig:subsample_thoroughC} 
and~\ref{fig:subsample_thoroughD}), the ordering of the approximations are 
similar, \emph{i.e.}, the semi-supervised eigenvectors constructed from the 
$\text{rank}=n$ approximation performs the best. 
Moreover, the gap between the $\text{rank}=400$ and $\text{rank}=n$ is 
largest for $\sigma_i^2=200$, again suggesting this approximation is of 
insufficient rank to model the relevant local heterogeneities deep down in 
the spectrum; whereas for $\sigma_i^2=80$, the $\text{rank}=400$ the 
approximation comes very close to the exact  representation, suggesting 
that local structures are well modeled near the end of the spectrum. 

To summarize these results, the method of semi-supervised eigenvectors 
successfully extracts relevant local structures to perform locally-biased 
classification, even when they are located far from the end of the spectrum.
Moreover, in both cases we considered, the classification error is reduced 
significantly by using only a few locally-biased components.
This contrasts with the global eigenvectors, where for $\sigma_i^2=80$ at 
least $20$ eigenvectors are needed in order to obtain similar performance; 
and for $\sigma_i^2=200$, the classification error remains high even for 
$200$ eigenvectors in case of $\text{rank}=n$. 

\begin{figure*}[hbt!]
\subfigure[Global eigenvectors, $\sigma_i^2=80$]{
\begin{minipage}[b]{0.49\linewidth}
\centering
\includegraphics[scale=0.4]{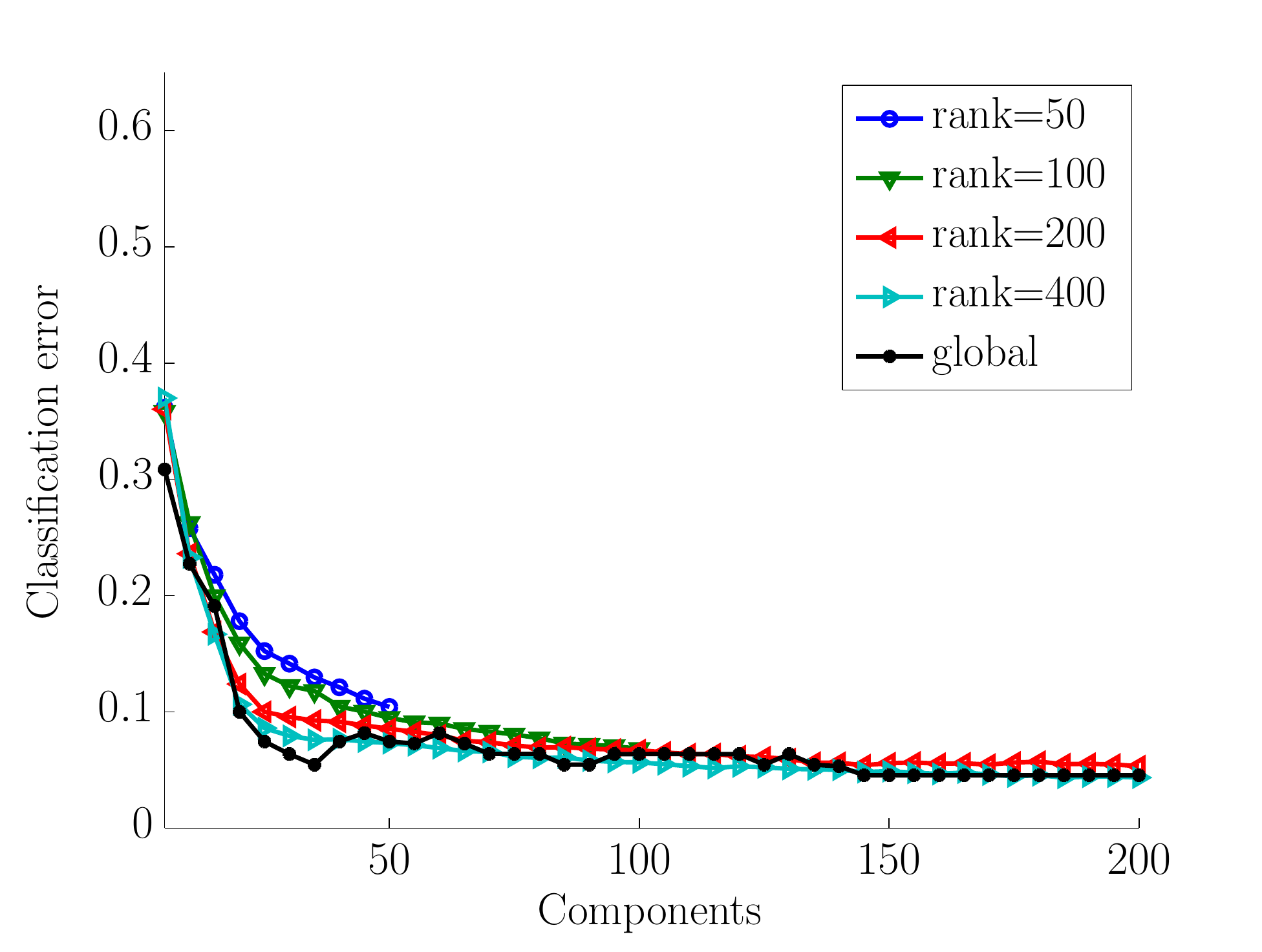}\\
\end{minipage}
\label{fig:subsample_thoroughA}
}
\subfigure[Global eigenvectors, $\sigma_i^2=200$]{
\begin{minipage}[b]{0.49\linewidth}
\centering
\includegraphics[scale=0.4]{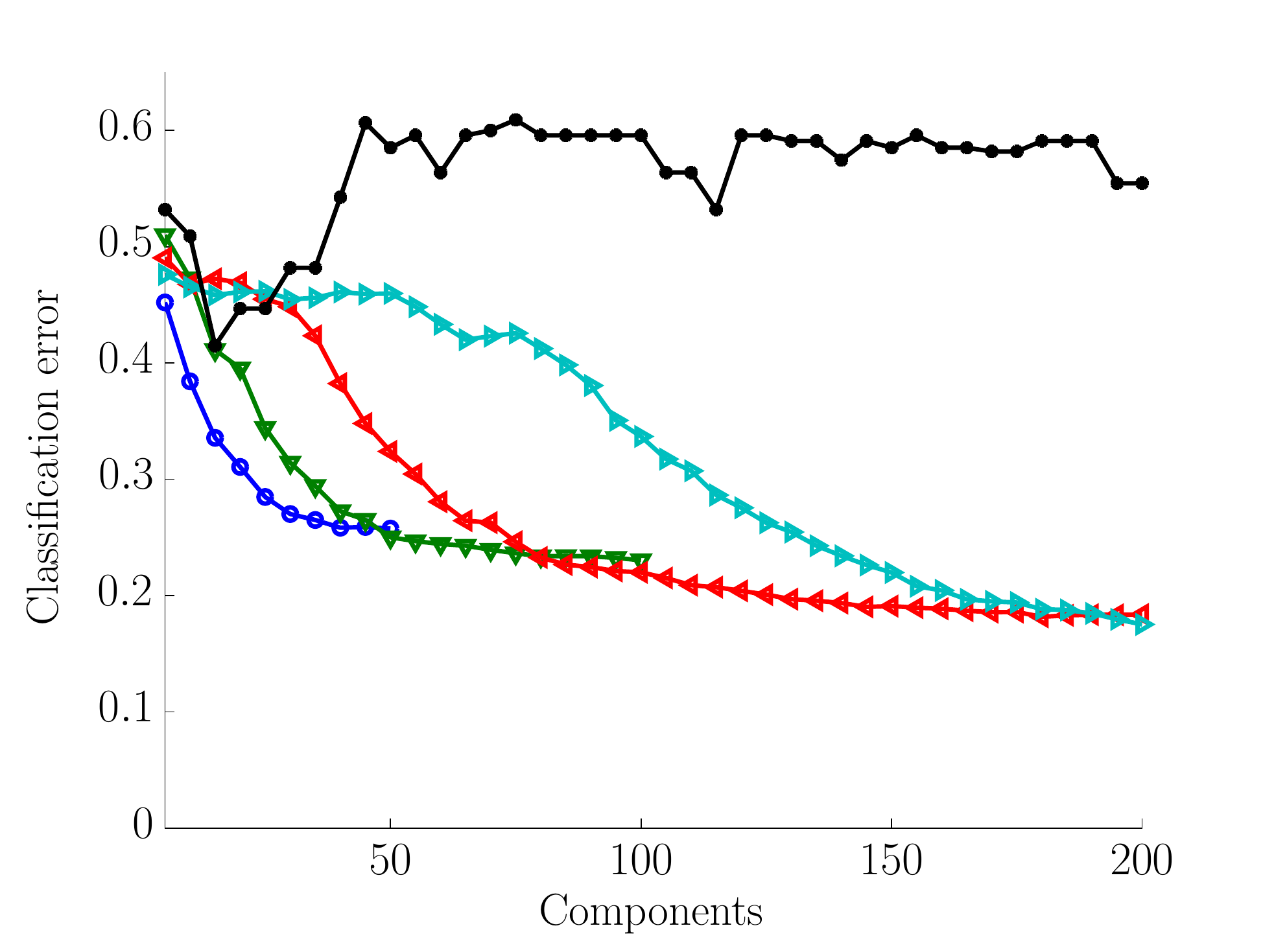}\\
\end{minipage}
\label{fig:subsample_thoroughB}
}
\subfigure[Semi-supervised eigenvectors, $\sigma_i^2=80$]{
\begin{minipage}[b]{0.49\linewidth}
\centering
\includegraphics[scale=0.4]{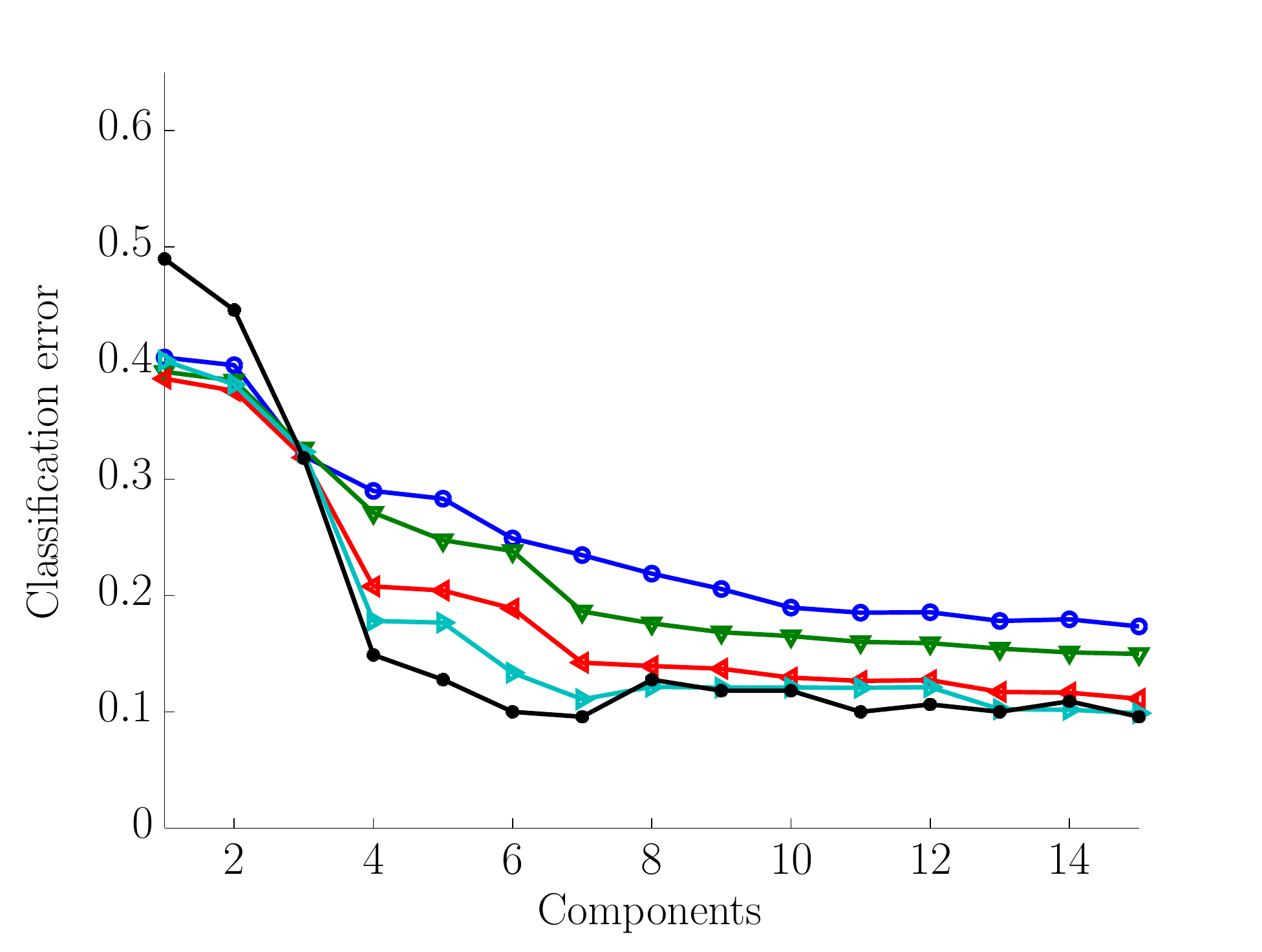}
\end{minipage}
\label{fig:subsample_thoroughC}
}
\subfigure[Semi-supervised eigenvectors, $\sigma_i^2=200$]{
\begin{minipage}[b]{0.49\linewidth}
\centering
\includegraphics[scale=0.4]{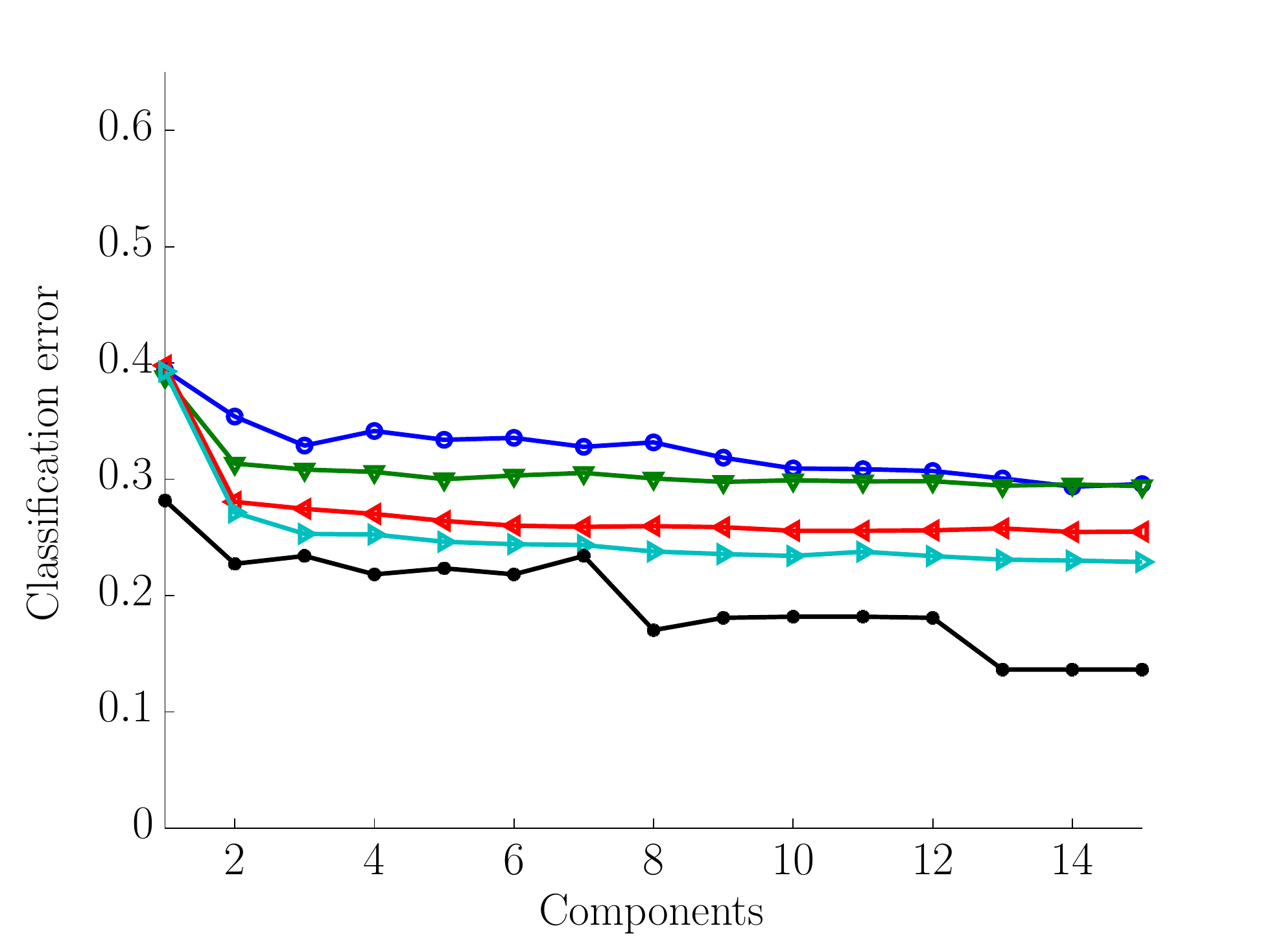}\\
\end{minipage}
\label{fig:subsample_thoroughD}
}
\caption{We consider $10\%$ of the MNIST training and 
test data and investigate the classification accuracy of a downstream SGT 
classifier for various approximations of the dense similarity matrix. 
\ref{fig:subsample_thoroughA} and \ref{fig:subsample_thoroughB}:
Classification error for the SGT evaluated directly on global eigenvectors, 
based on various Nystr\"{o}m approximations and the two choices of the 
kernel width parameter (respectively, $\sigma_i^2=80$ and $\sigma_i^2=200$).
\ref{fig:subsample_thoroughC} and \ref{fig:subsample_thoroughD}:
Classification error we have used the Nystr\"{o}m approximations as basis 
for computing semi-supervised eigenvectors that are then used in the 
downstream SGT classifier. 
All plots show the mean over 30 repetitions.
}
\label{fig:subsample_thorough}
\end{figure*}

\subsubsection{Implementation issues and running time considerations}
Here, we discuss implementation details and investigate the advantage of using the Graphics Processing Unit (GPU) for computing semi-supervised eigenvectors.
Although the computations underlying the construction of semi-supervised
eigenvectors could be performed in many computational environments, the GPU architecture fits well with the dense semi-supervised eigenvector computation in Eqn. (\ref{eq:lagrange_exact}); for each component, this expression will be executed $\mathcal{O}(\log_2((\lambda_2(G)+\vol(G)))/\epsilon)$  times within the binary search of Algorithm \ref{alg_new_1}.  


Compared to a Central Processing Unit (CPU), which is well-suited for processing code with a complex 
control flow, a GPU is much better suited for addressing problems that can 
be expressed as data-parallel computations with a high arithmetic intensity~\cite{kruger2003linear,bolz2003sparse,hansen2011}.
A GPU consists of a set of Multi Processors (MPs), each 
containing multiple Scalar Processors (SPs), as well as different types of local 
memories that the SPs may access.  
All MPs have also access to a large global memory that, compared to their 
internal memories, is much slower to access.
A computation task to be executed on such a device is usually setup in a 
grid, where each element in the grid gets assigned to a thread. 
The grid is then decomposed into blocks that are scheduled onto the MPs with 
available resources, and the assigned MP will schedule the elements 
of the block onto its SPs in warps with 32 threads. 
The best performance is obtained when all threads in a warp execute the same 
instruction and when the total number of threads in the grid is large, as this allows various latencies to be overlapped with 
arithmetic operations. 

\begin{figure*}[hbt!]
\subfigure[Single precision, the $25^\text{th}$ solution]{
\begin{minipage}[b]{0.49\linewidth}
\centering
\includegraphics[scale=0.4]{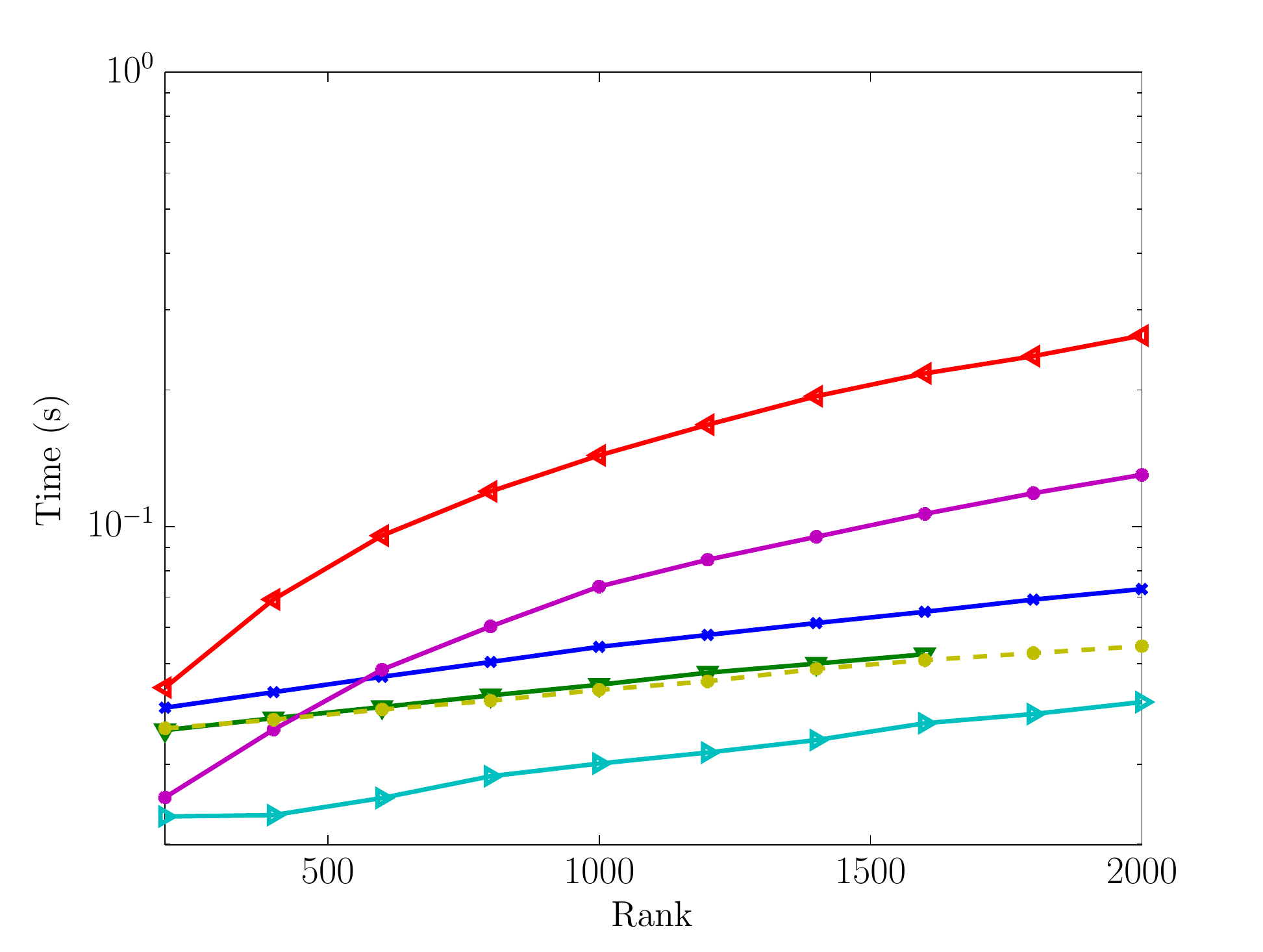}\\
\end{minipage}
\label{fig:new_gpuA}
}
\subfigure[Double precision, the $25^\text{th}$ solution]{
\begin{minipage}[b]{0.49\linewidth}
\centering
\includegraphics[scale=0.4]{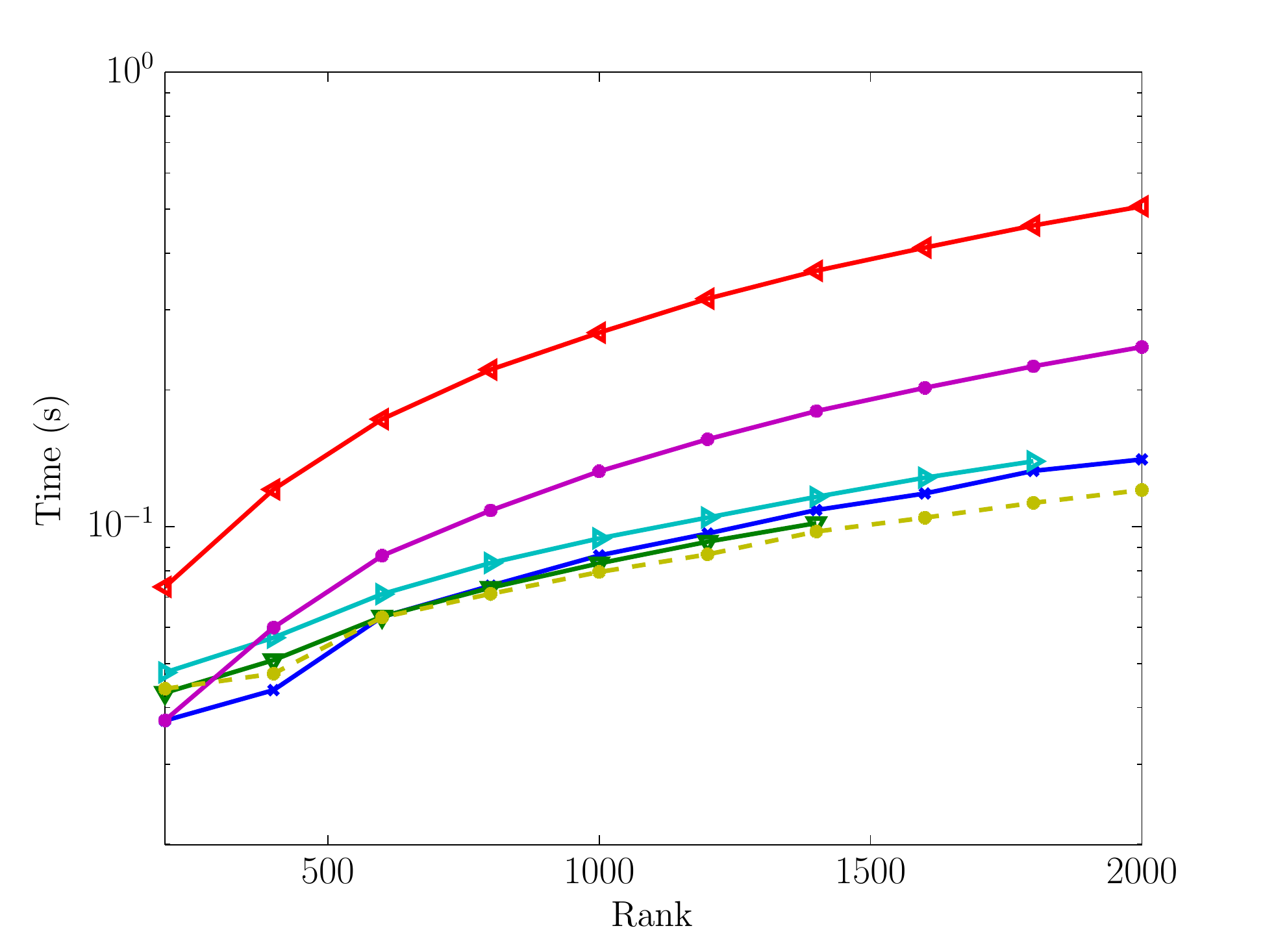}\\
\end{minipage}
\label{fig:new_gpuB}
}
\subfigure[Single precision, the $500^\text{th}$ solution]{
\begin{minipage}[b]{0.49\linewidth}
\centering
\includegraphics[scale=0.4]{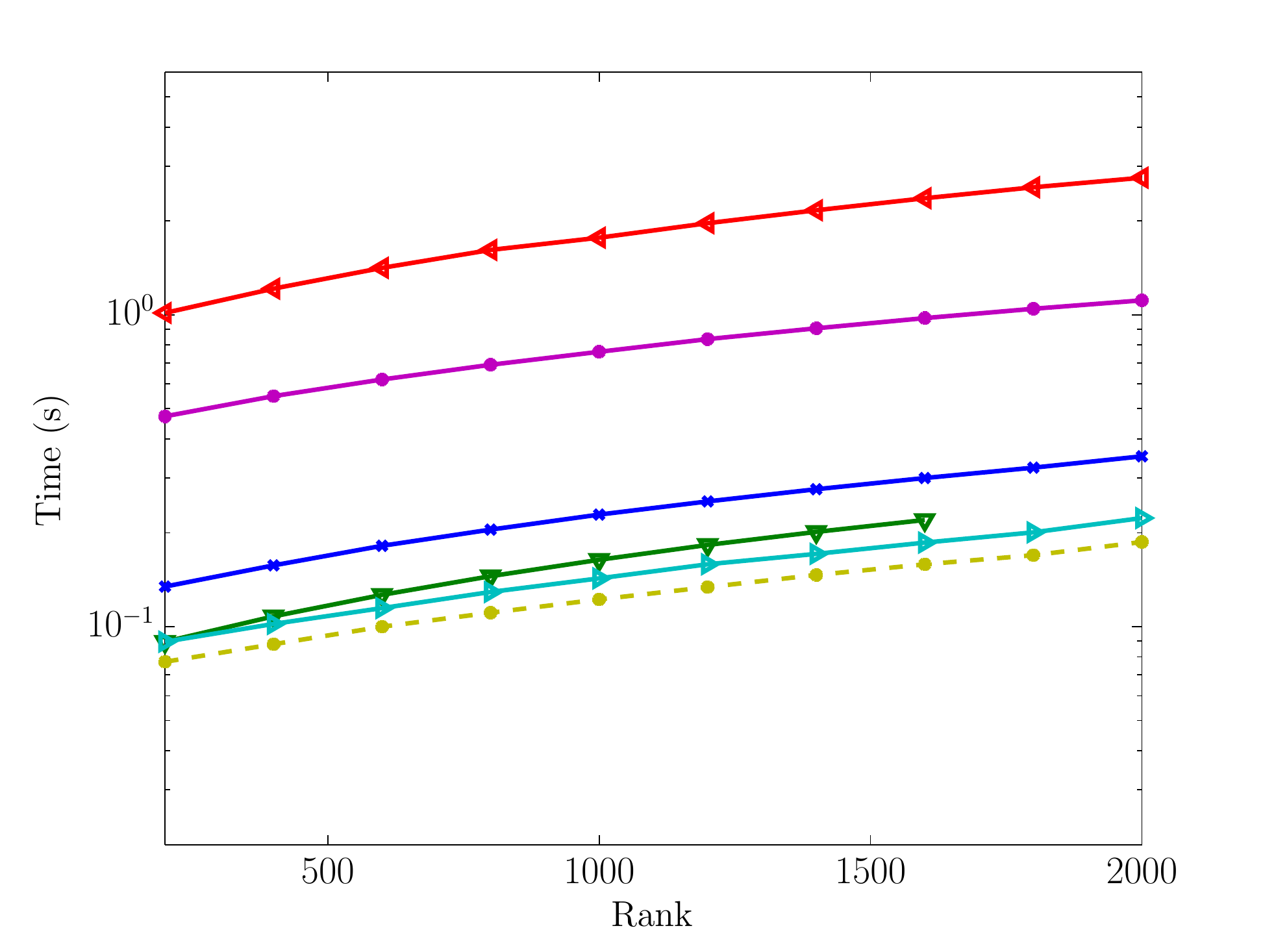}
\end{minipage}
\label{fig:new_gpuC}
}
\subfigure[Double precision, the $500^\text{th}$ solution]{
\begin{minipage}[b]{0.49\linewidth}
\centering
\includegraphics[scale=0.4]{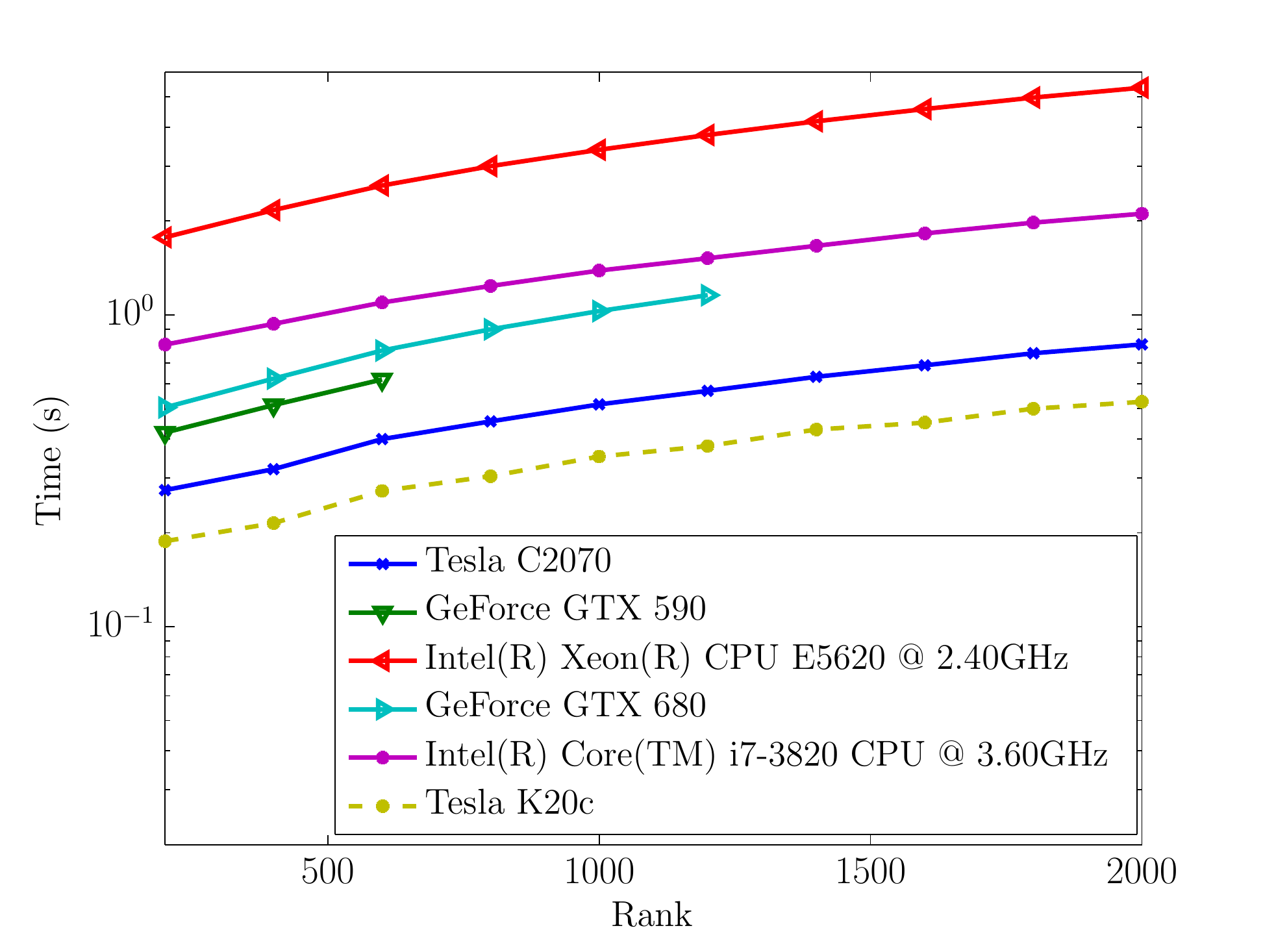}\\
\end{minipage}
\label{fig:new_gpuD}
}
\caption{Running time performance measurements for solving 
Eqn. (\ref{eq:lagrange_exact}), given a specific value of the parameter 
$\gamma$, on the entire MNIST data set consisting of $70,000$ samples, as 
a function of the rank parameter. Single and double precision arithmetic results are respectively shown in \ref{fig:new_gpuA} and \ref{fig:new_gpuB} for the task of computing the $25^\text{th}$ solution, \emph{i.e.}, constrained to be perpendicular to the previous $24$ solutions. Similar does \ref{fig:new_gpuC} and \ref{fig:new_gpuD} show performance results for computing the $500^\text{th}$ solution, and here the advantage of using recent GPU architectures become even more evident, as the operation is dominated by a high arithmetic intensity that fit well with such architectures.
}
\label{fig:new_gpu}
\end{figure*}


We compare most recent  CPU and GPU devices in computing the solution to Eqn. (\ref{eq:lagrange_exact}). In terms of the GPUs we test both consumer devices (GeForce) and professional devices (Tesla), where the latter provides enhanced performance for double-precision floating point arithmetic.  
For a fair comparison, we decided to rely on the BLAS\footnote{The BLAS implementation uses all physical CPU cores.} and CUBLAS implementations as used in \textsc{Matlab 2012b}, \emph{i.e.}, avoiding to favor specific architecturally dependent implementation optimizations, since BLAS and CUBLAS should be optimal in terms of the underlying architecture. Figure \ref{fig:new_gpu} shows performance measures (wall-clock-time as a function of the rank parameter) of CPU and GPU experiments. For single precision arithmetic the GTX 680 scales very well, and it ends up being more than three times faster than the i7-3820, as well as noticeably faster than the previous generation high-end Tesla C2070, and it even outperforms the latest generation Tesla K20c, as seen in Figure\ref{fig:new_gpuA}. 
As seen in Figure \ref{fig:new_gpuA} and \ref{fig:new_gpuB}, the GPUs perform poorly in the low-rank regime, and this is explained by the overhead of transferring data back and forth from the main memory and to the device. However, for the high-rank matrices the arithmetic intensity increases and the overhead is less dominant. Also evident is the performance improvement of the latest CPU generation (i7-3820), that for the considered operation ends up being more than twice as fast as a previous generation E5620, that primarily is due to the higher clock frequency.
For double precision arithmetic, the GTX 680 and GTX 590 are due to memory constraints stopped prematurely in the experiments, as they respectively are equipped with 2048MB and 1536MB (per GPU). Note that even though the GTX 590 is a dual GPU card, it is from the GPU computing perspective setup as two individual devices, and only one of these are used for the experiments.
Interestingly the older GTX 590 outperforms the recent GTX 680, which may be explained by a higher memory bandwidth. In Figure \ref{fig:new_gpuD} the Tesla K20c outperforms all other devices by a fair margin, being $\approx 1.5$ times as fast as the Tesla C2070, and four times faster than the i7-3820. 

Using GPU computing we are able to reduce the computation time considerably. Depending on the application of the semi-supervised eigenvectors, the advantage may be significant, for example if applied in time critical applications such as online and real-time applications or large-scale simulations.


\subsection{Functional Magnetic Resonance Imaging Data}
\label{sxn:empirical-fmri}
The next dataset we consider is from functional Magnetic Resonance Imaging (fMRI). Here data analysis usually considers the characterization of relations between cognitive variables and individual brain voxels, for instance using the mass-univariate General Linear Model (GLM), where statistical parametric maps are used to identify regions of gray matter that are significantly related to particular effects under study \cite{Friston1994}. 
Even though such a voxel-wise univariate approach has been tremendously productive, there are obvious limits on what can be learned about cognitive states by only examining isolated voxels \cite{NPDH06}. Multivariate methods have therefore paved the way for more advanced paradigms involving complex cognition, where the latent brain state cannot solely be determined from looking at individual voxel time series \cite{Eger2008,Kamitani2005,Bode2009}.
However, an immediate challenge for multivariate approaches is that weak signals carried by a sparse set of voxels can be very hard to detect, and for this reason multivariate approaches are often accompanied by spatial priors, to improve on the signal-to-noise ratio (SNR).


\begin{figure*}[hbt!]
\subfigure[Global eigenvector, $v_2$]{
\begin{minipage}[b]{0.49\linewidth}
\centering
\includegraphics[scale=0.4]{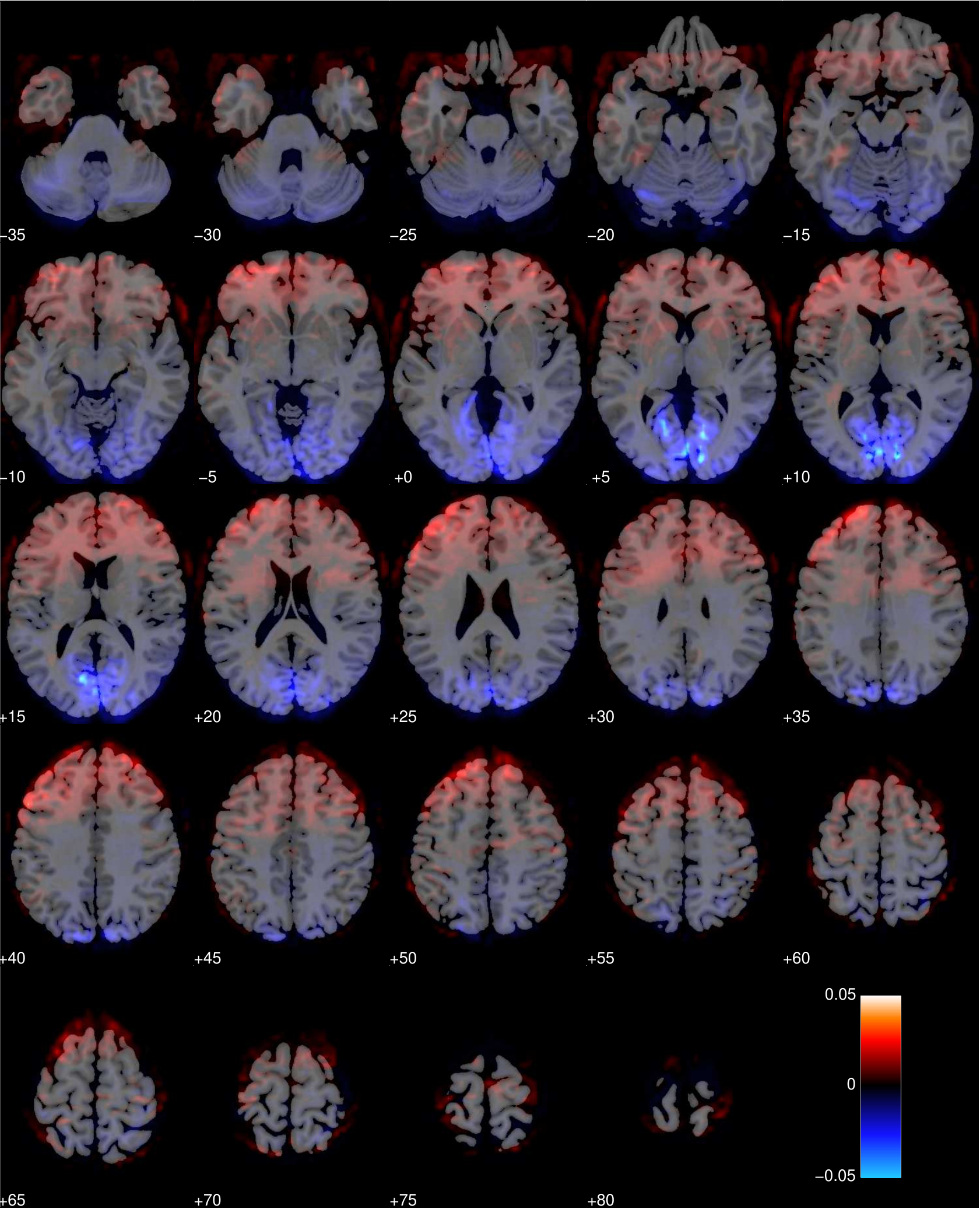}\\
\end{minipage}
\label{fig:global_new_2}
}
\subfigure[Global eigenvector, $v_3$]{
\begin{minipage}[b]{0.49\linewidth}
\centering
\includegraphics[scale=0.4]{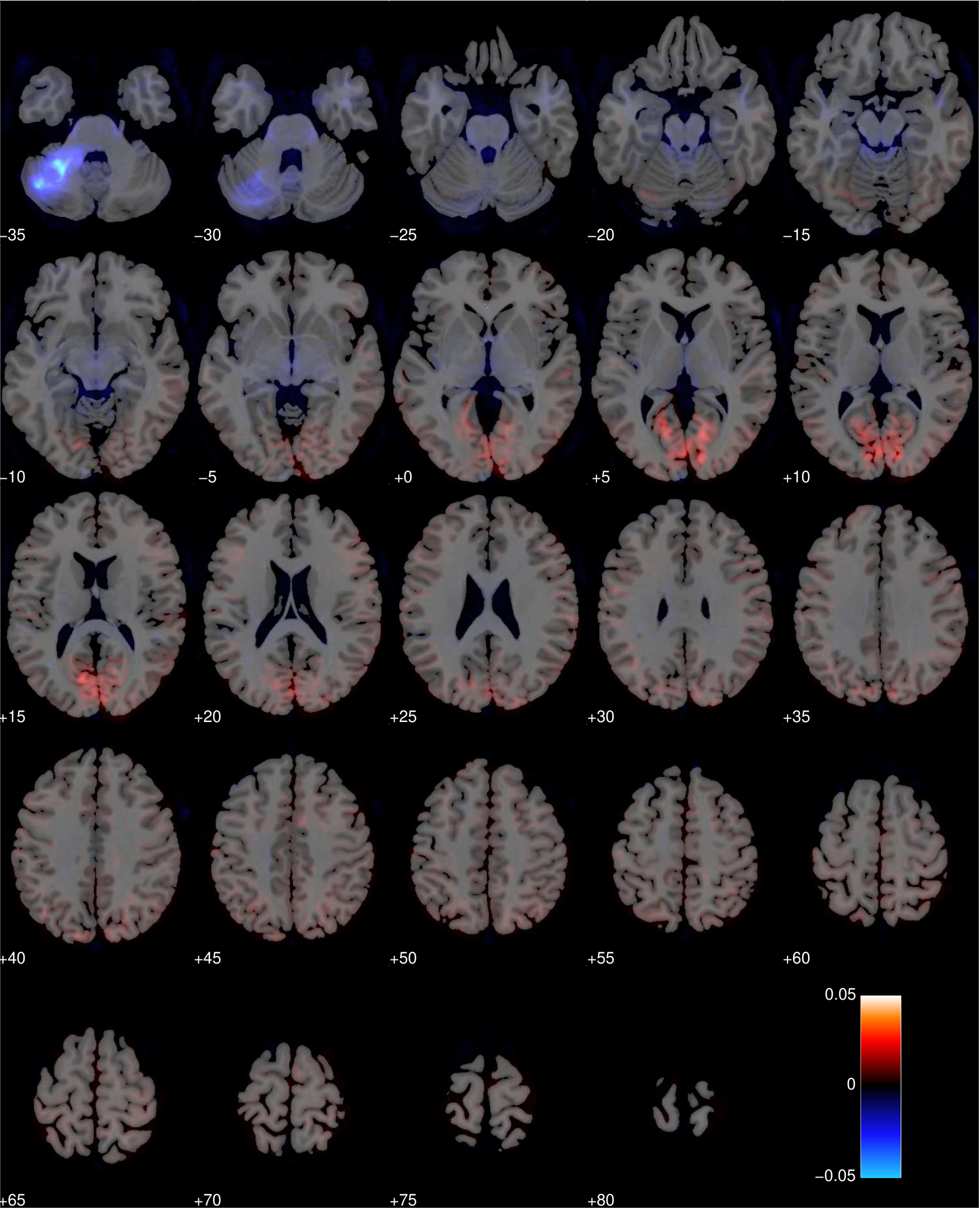}\\
\end{minipage}
\label{fig:global_new_3}
}
\subfigure[Global eigenvector, $v_4$]{
\begin{minipage}[b]{0.49\linewidth}
\centering
\includegraphics[scale=0.4]{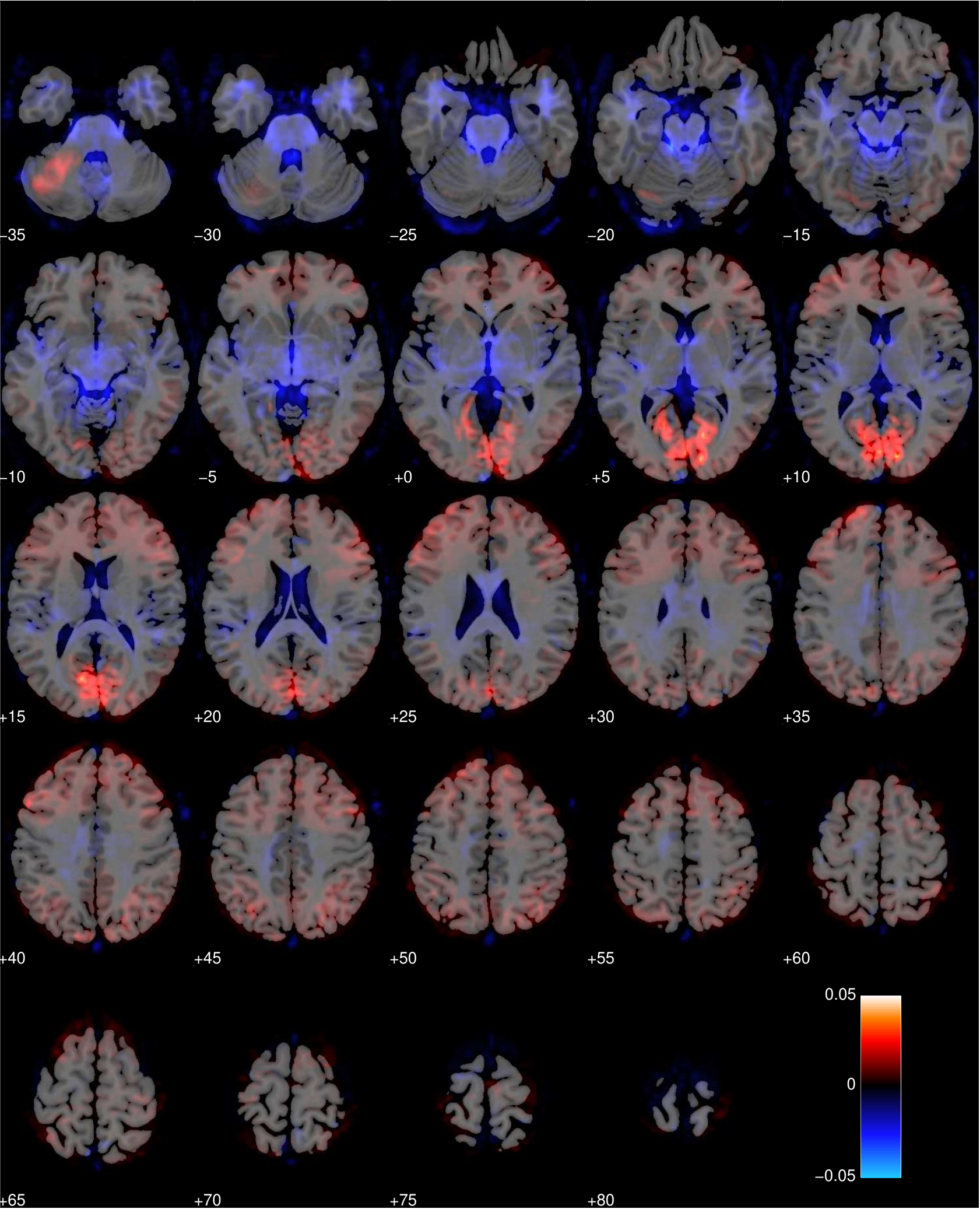}\\
\end{minipage}
\label{fig:global_new_4}
}
\subfigure[Global eigenvector, $v_5$]{
\begin{minipage}[b]{0.49\linewidth}
\centering
\includegraphics[scale=0.4]{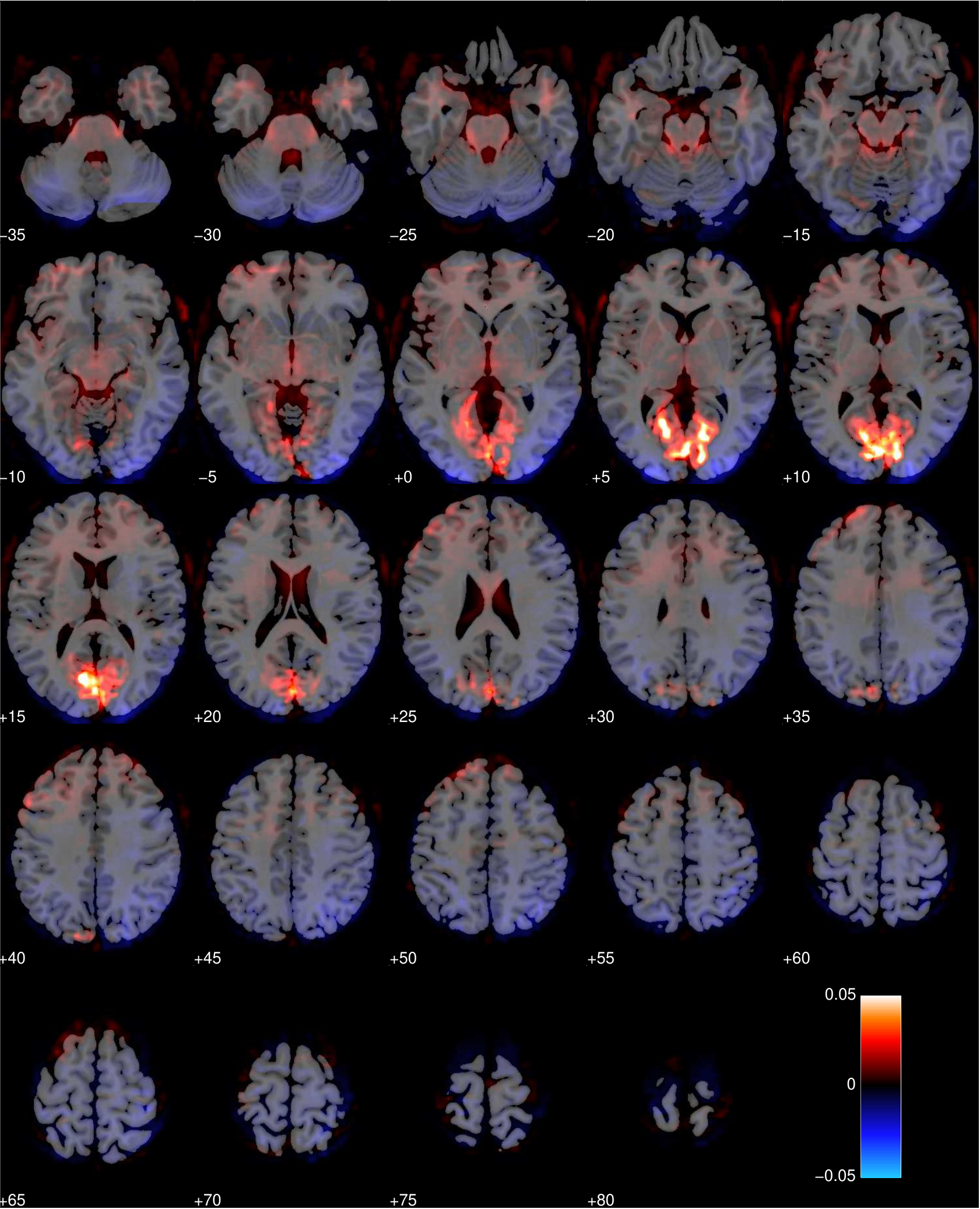}\\
\end{minipage}
\label{fig:global_new_5}
}
\caption{Visualization of the leading 4 nontrivial global eigenvectors.}
\label{fig:fmri_global_eigs}
\end{figure*}

\begin{figure*}[hbt!]
\subfigure[Primary Motor Cortex (PMC)]{
\begin{minipage}[b]{0.49\linewidth}
\centering
\includegraphics[scale=0.38]{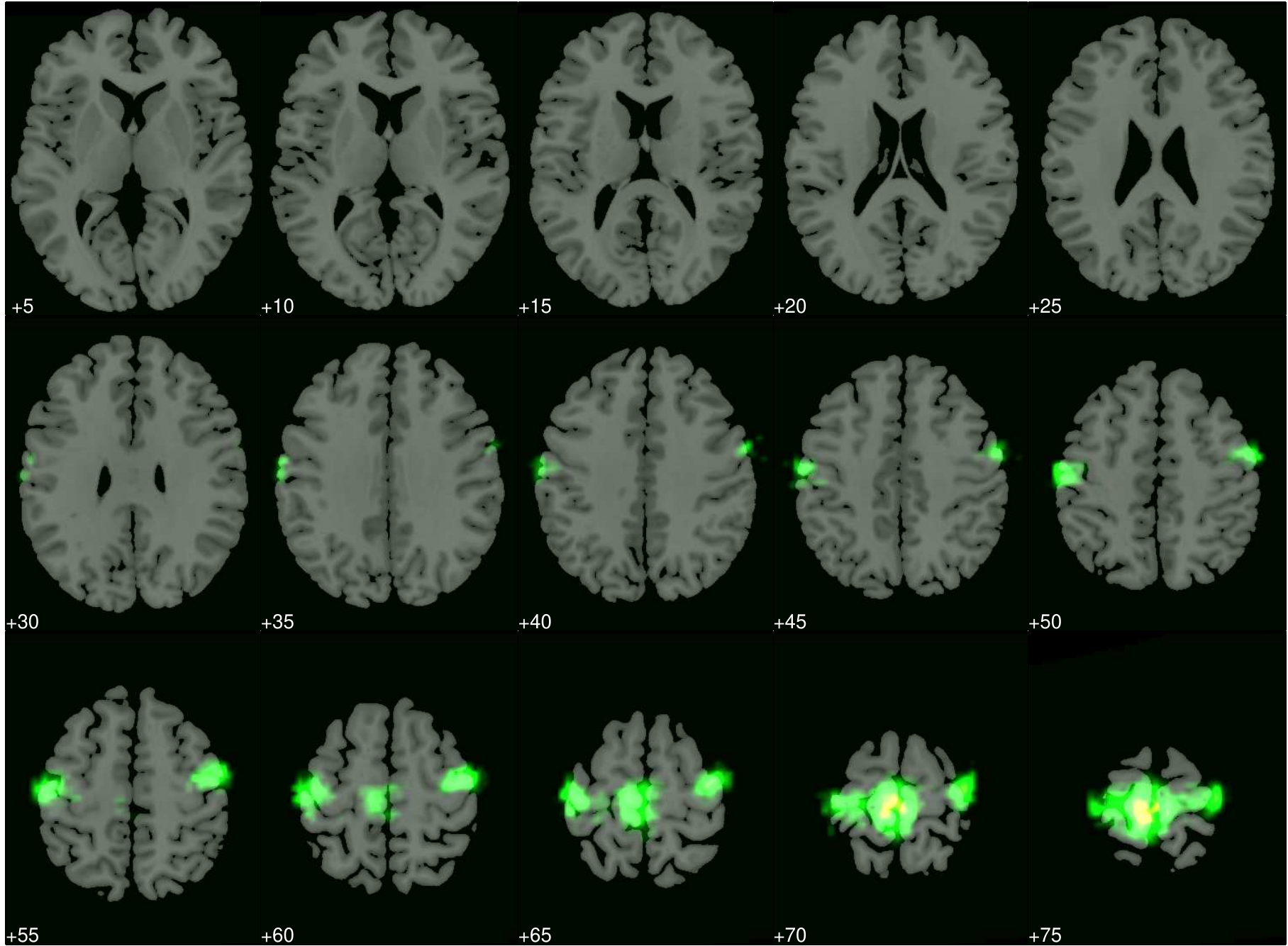}\\
\end{minipage}
\label{fig:fmriaccuracyPMC}
}
\subfigure[Primary Auditory Cortex (PAC)]{
\begin{minipage}[b]{0.49\linewidth}
\centering
\includegraphics[scale=0.38]{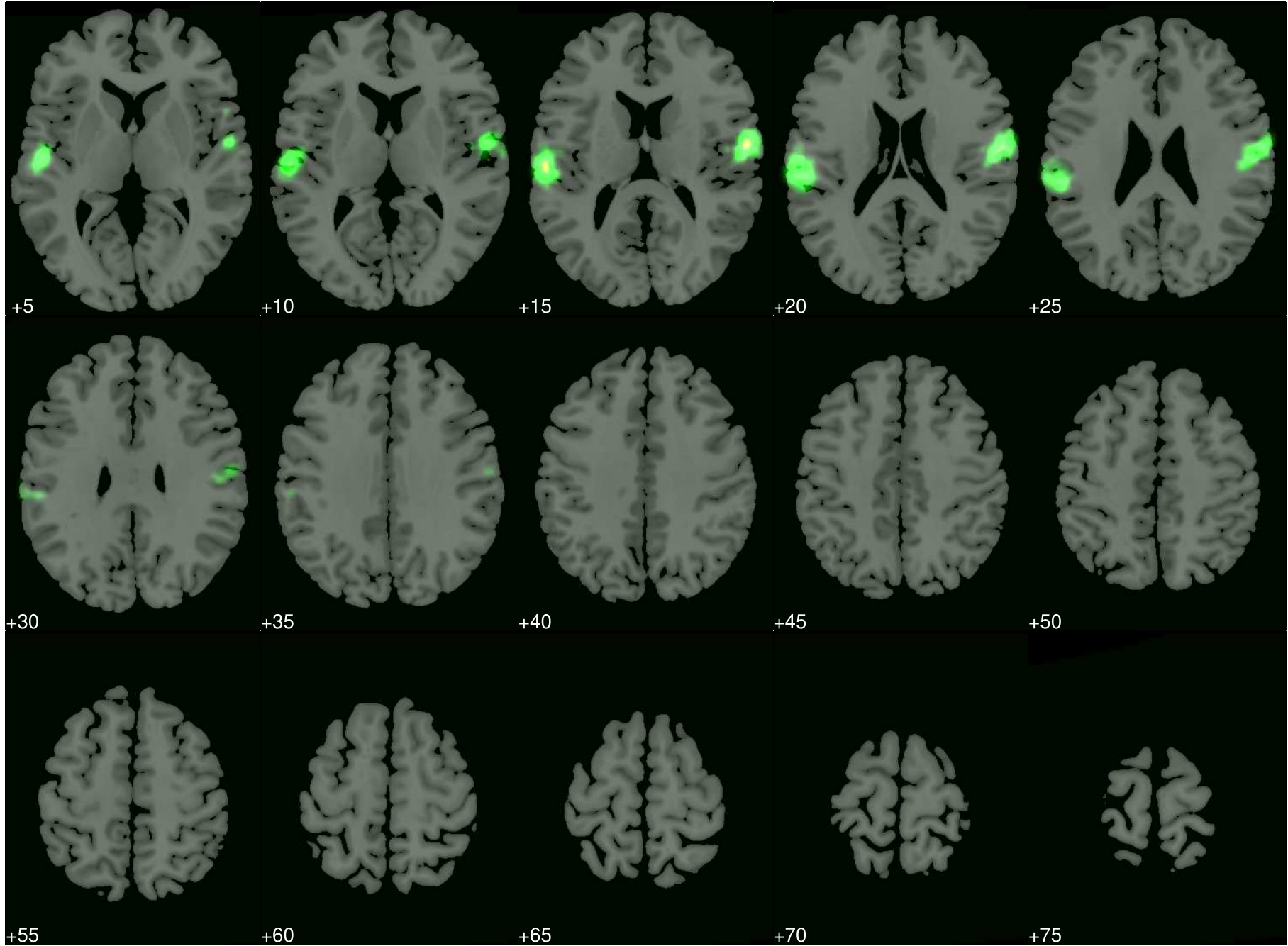}\\
\end{minipage}
\label{fig:fmriaccuracyPAC}
}
\subfigure[Classification accuracy]{
\begin{minipage}[b]{\linewidth}
\centering
\vspace{5mm}
\includegraphics[scale=0.45]{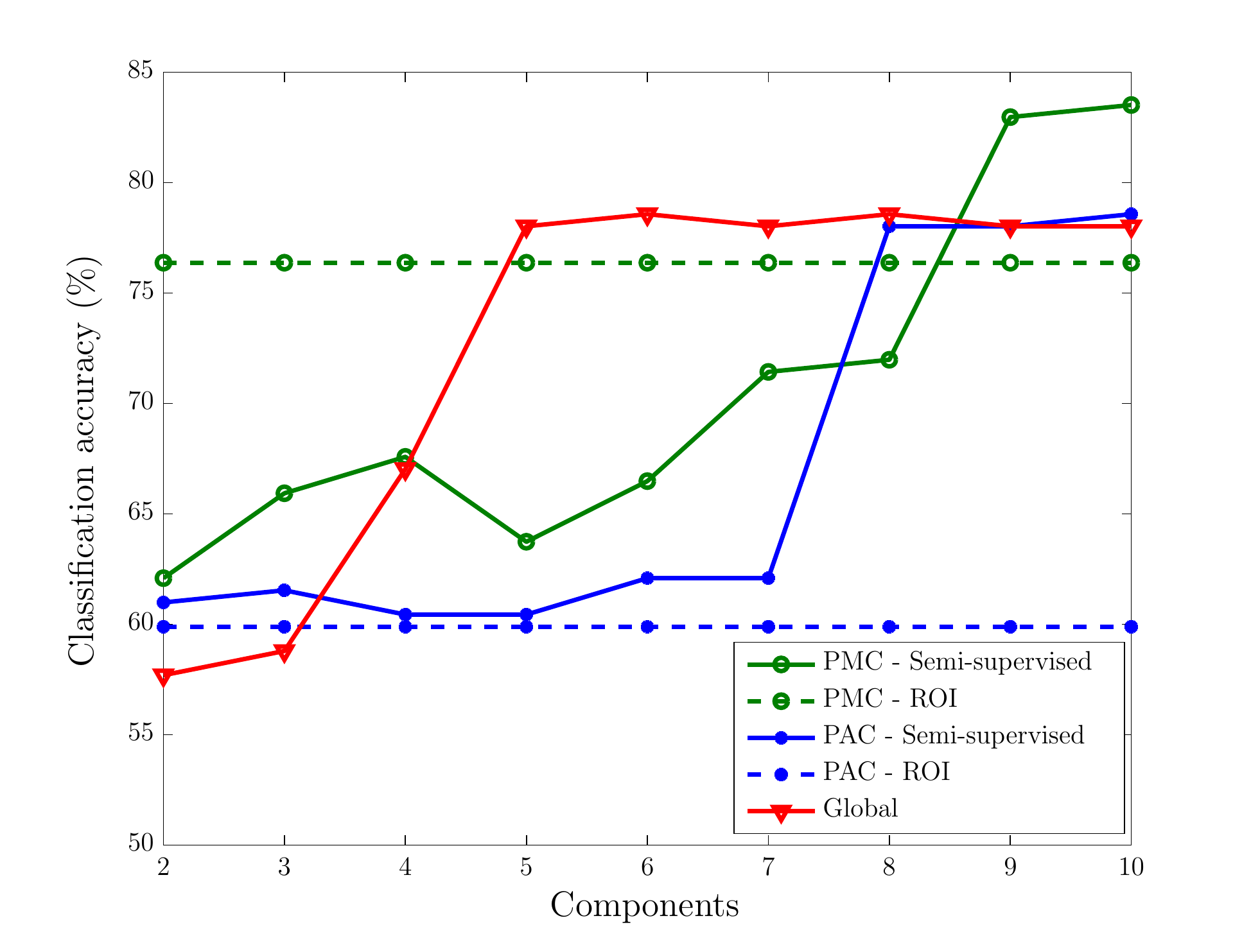}\; \\
\end{minipage}
\label{fig:fmriaccuracyA}
}
\caption{Figure \ref{fig:fmriaccuracyPMC} shows the seed region in PMC, and Figure \ref{fig:fmriaccuracyPAC} shows the seed region in PAC.
The plot in Figure \ref{fig:fmriaccuracyA} shows the classification accuracy for the 5 different features extraction approaches. The dashed lines mark the reference where all voxel time series, as covered by the seed, are used in the downstream classifier, and the solid ones correspond to the accuracy obtained from projecting the data onto the semi-supervised eigenvectors seeded in PAC and  PMC, as well as the global eigenvectors.}\label{fig:fmriaccuracy}
\end{figure*}

\begin{figure*}[hbt!]
\subfigure[Semi-supervised eigenvector (PMC), $x_1$]{
\begin{minipage}[b]{0.49\linewidth}
\centering
\includegraphics[scale=0.4]{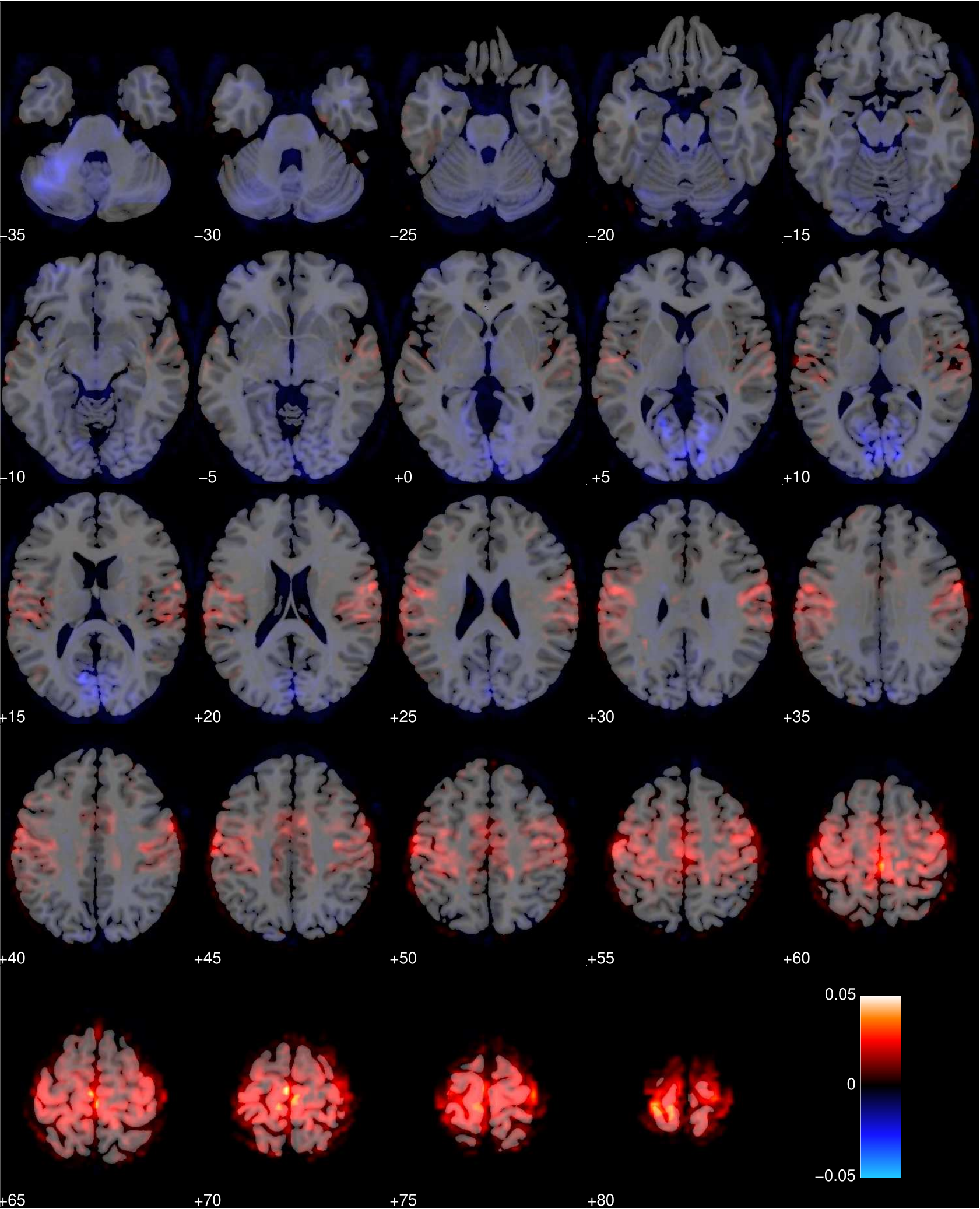}\\
\end{minipage}
\label{fig:fmri_sseigs_pmcA}
}
\subfigure[Semi-supervised eigenvectors (PMC), $x_2$]{
\begin{minipage}[b]{0.49\linewidth}
\centering
\includegraphics[scale=0.4]{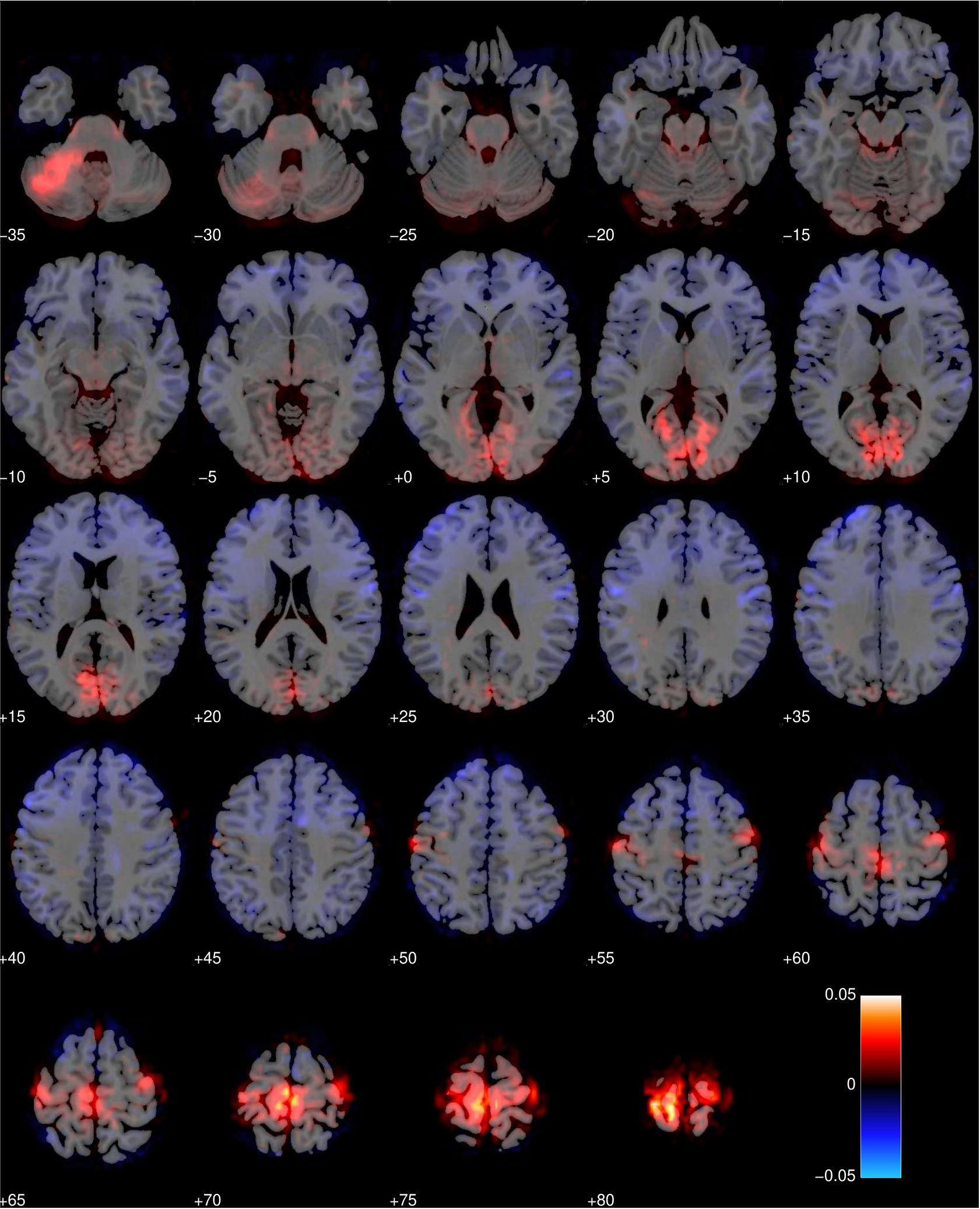}\\
\end{minipage}
\label{fig:fmri_sseigs_pmcB}
}
\subfigure[Semi-supervised eigenvectors (PMC), $x_3$]{
\begin{minipage}[b]{0.49\linewidth}
\centering
\includegraphics[scale=0.4]{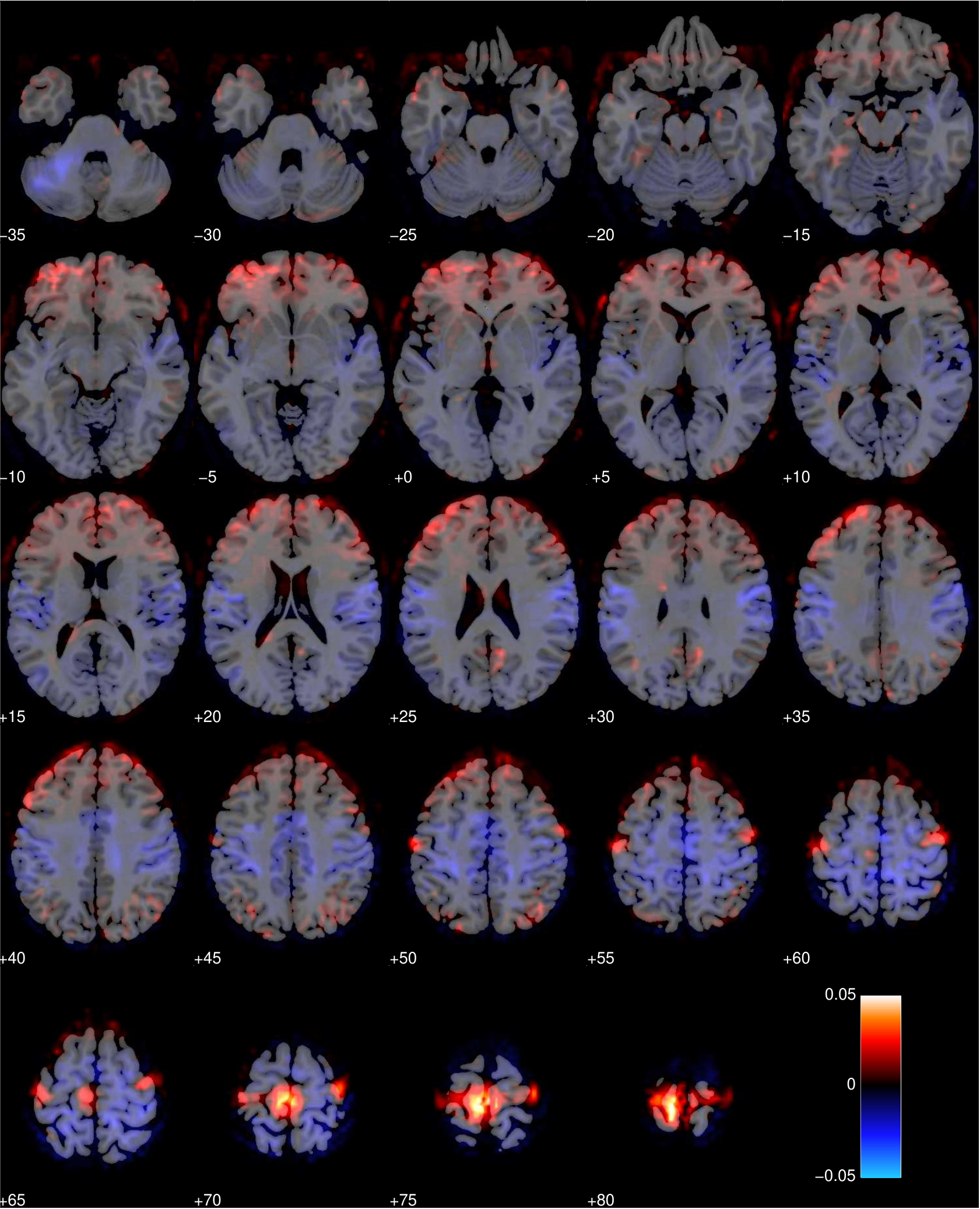}\\
\end{minipage}
\label{fig:fmri_sseigs_pmcC}
}
\subfigure[Semi-supervised eigenvectors (PMC), $x_4$]{
\begin{minipage}[b]{0.49\linewidth}
\centering
\includegraphics[scale=0.4]{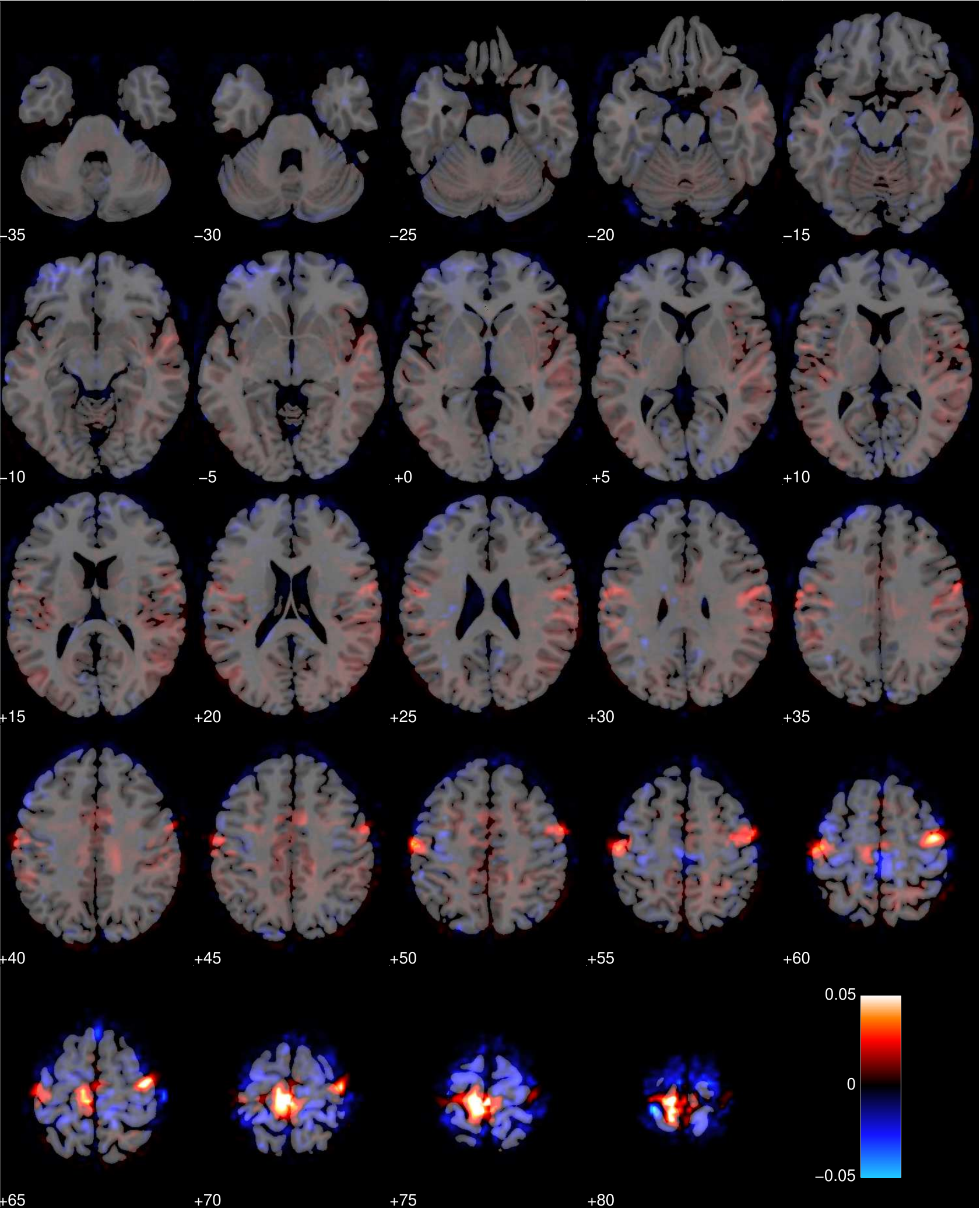}\\
\end{minipage}
\label{fig:fmri_sseigs_pmcD}
}
\caption{Visualization of the leading 4 semi-supervised eigenvectors seeded in PMC, each correlating $0.25$ with the seed, that is visualized in Figure \ref{fig:fmriaccuracyPMC}.}
\label{fig:fmri_sseigs_pmc}
\end{figure*}

\begin{figure*}[hbt!]
\subfigure[Semi-supervised eigenvector (PAC), $x_1$]{
\begin{minipage}[b]{0.49\linewidth}
\centering
\includegraphics[scale=0.4]{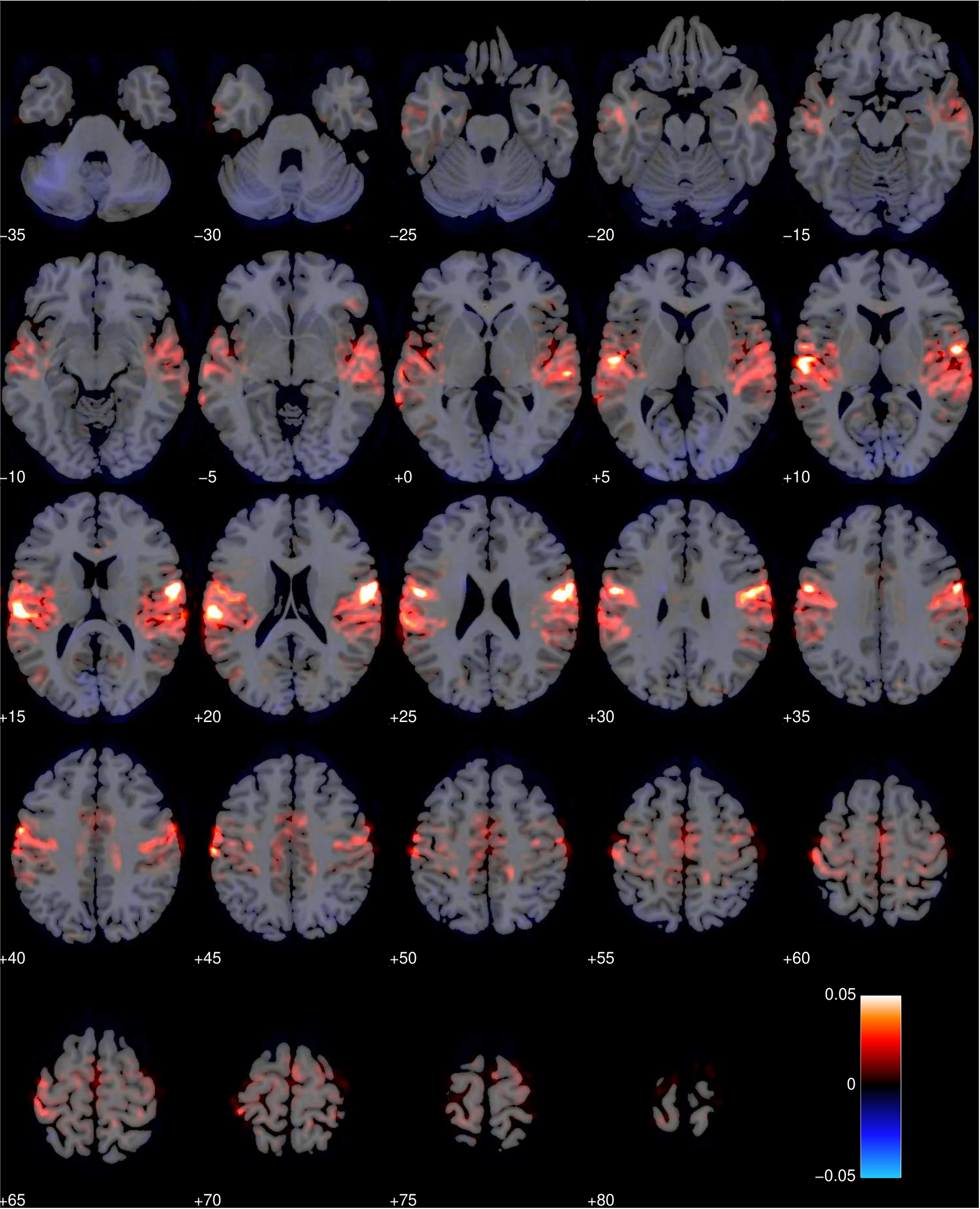}\\
\end{minipage}
\label{fig:fmri_sseigs_pacA}
}
\subfigure[Semi-supervised eigenvectors (PAC), $x_2$]{
\begin{minipage}[b]{0.49\linewidth}
\centering
\includegraphics[scale=0.4]{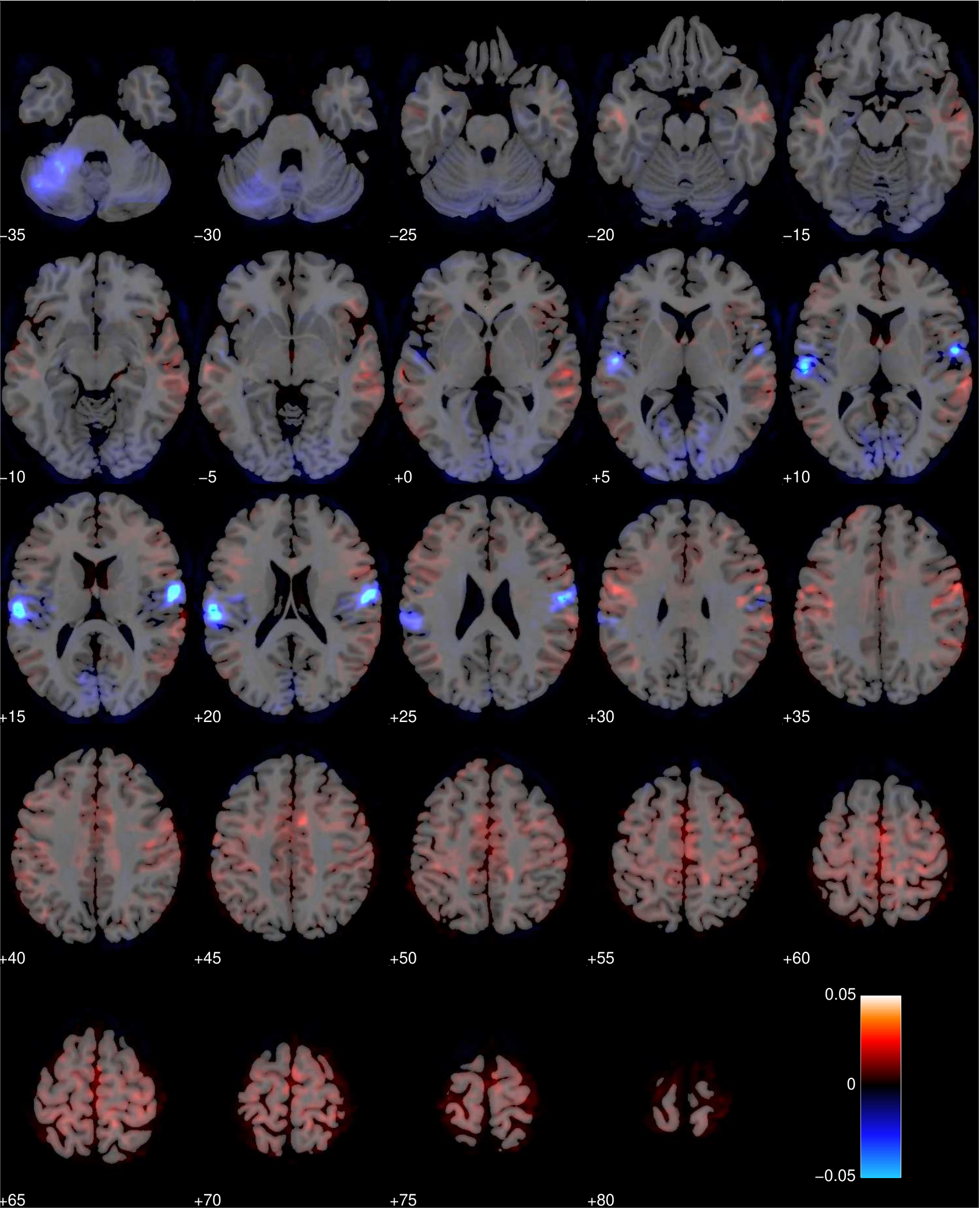}\\
\end{minipage}
\label{fig:fmri_sseigs_pacB}
}
\subfigure[Semi-supervised eigenvectors (PAC), $x_3$]{
\begin{minipage}[b]{0.49\linewidth}
\centering
\includegraphics[scale=0.4]{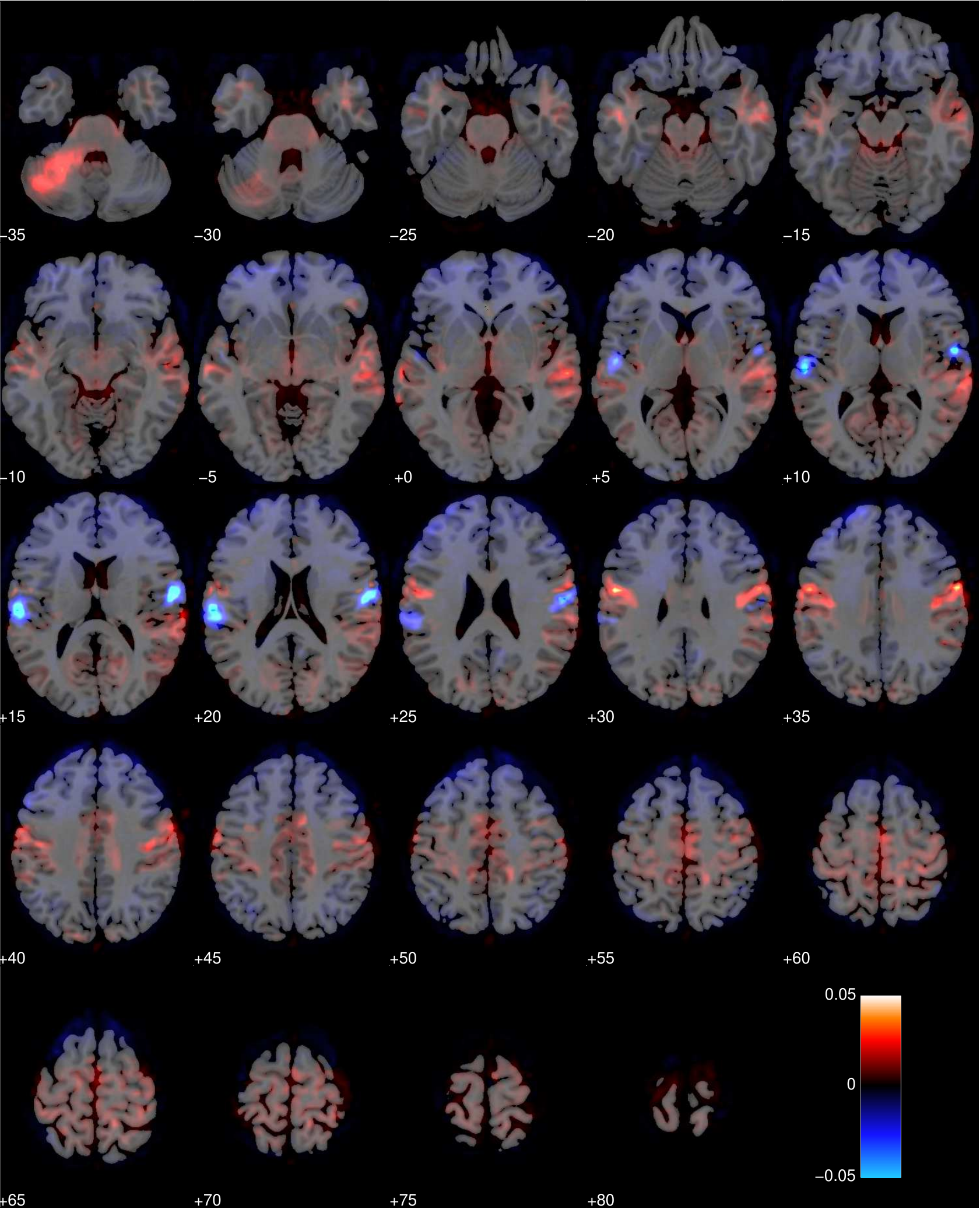}\\
\end{minipage}
\label{fig:fmri_sseigs_pacC}
}
\subfigure[Semi-supervised eigenvectors (PAC), $x_4$]{
\begin{minipage}[b]{0.49\linewidth}
\centering
\includegraphics[scale=0.4]{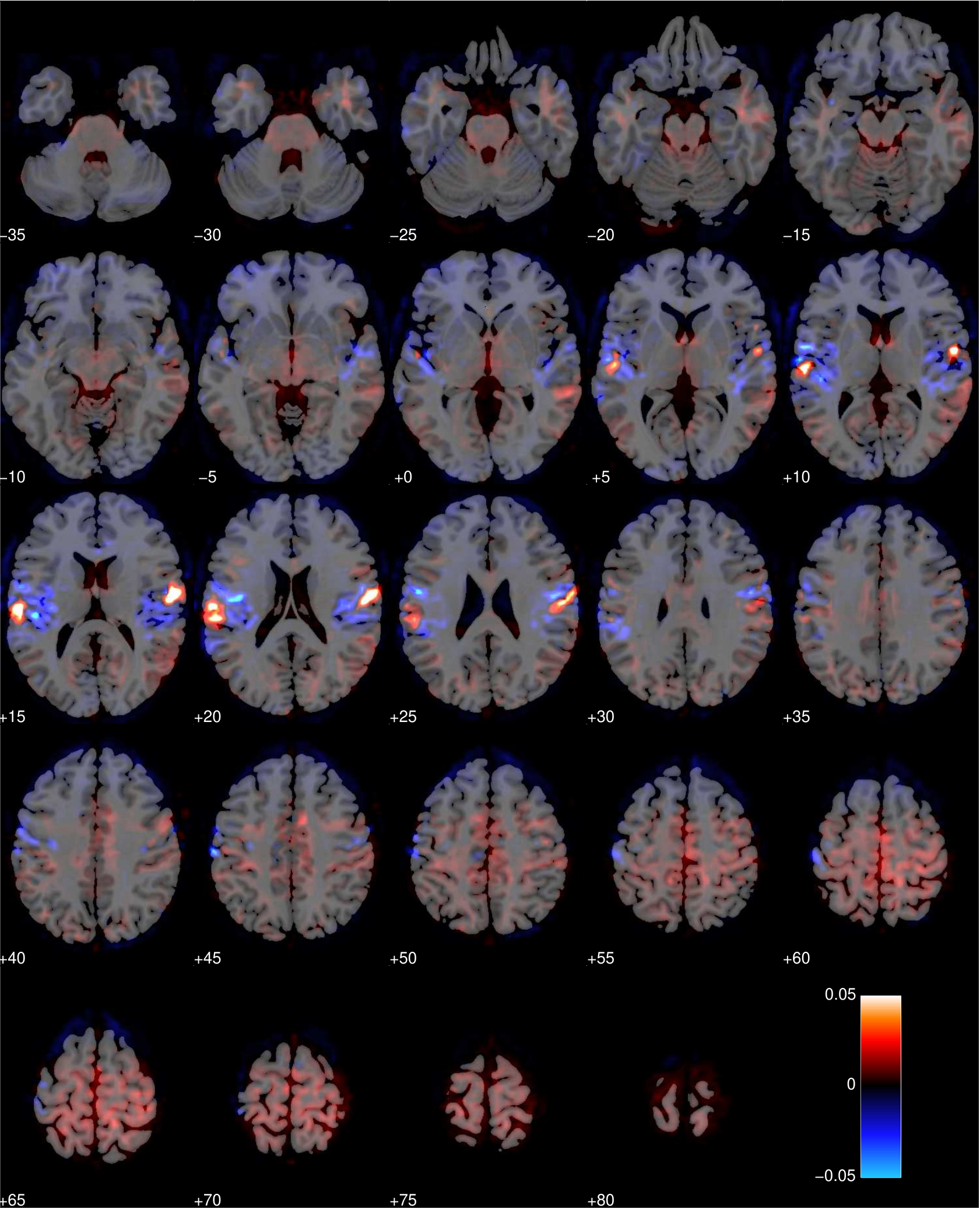}\\
\end{minipage}
\label{fig:fmri_sseigs_pacD}
}
\caption{Visualization of the leading 4 semi-supervised eigenvectors seeded in PAC, each correlating $0.25$ with the seed, that is visualized in Figure \ref{fig:fmriaccuracyPAC}.}
\label{fig:fmri_sseigs_pac}
\end{figure*}

Searchlight is an algorithm that scans through the whole brain by running multiple multivariate region-of-interest (ROI) analyses, measuring the respective generalization performance, and outputs a brain map showing which regions exhibit the best discriminative properties, for example measured by classification accuracy for a particular subject task~\cite{Kriegeskorte2006}. This approach was for instance applied by \cite{Haynes2007}, who used it to find regions in the brain that are predictive with respect to human intentions. 
Compared to a univariate approach, searchlight takes advantage of the power of multivariate techniques, with the caveat that it only performs well if the target signal is available within the area covered by the ROI. This limitation is indeed shared by the univariate approaches, but with searchlight we have the freedom to increase or decrease the ROI, depending on the structure of the considered problem. If the ROI is small we approach a univariate analysis, whereas if the ROI is large, the information localization becomes less specific. 
Thus, if the multivariate signal is spatially distributed the searchlight approach will fall short, and simply increasing the ROI may not be a solution as irrelevant time series will decrease the SNR.

The semi-supervised eigenvectors can be used to construct a spatially-guided basis that naturally allows for spatially distributed signal representations. This strategy shares many similarities with there searchlight approach, but it is not tied to a particular ROI, and it can span distributed voxel time series that are similar in terms of our graph representation. Using the semi-supervised eigenvectors on the $\text{voxel}\times \text{voxel}$ similarity graph in this way will yield a low dimensional representation that we can project the fMRI voxel time series onto, and in that projected space we can apply any suitable classification algorithm.

We tested the method on Blood Oxygenation Level Dependent (BOLD) sensitive fMRI data acquired on a 3 Tesla MR scanner (Siemens Magnetom Verio).
Additional sequence parameter were as follows: 25 interleaved echo planar imaging gradient echo slices, echo time 30 ms, repetition time 1390 ms, flip angle 90 degrees.
During the scanning session (1300 volumes) the subject was engaged in a simple motor paradigm in which the subject was asked to respond with key-presses when a visual cue was presented, and the classification task is to detect such key-presses. Pre-processing steps included: rigid body realignment, spatial smoothing (6 mm full width at half maximum isotropic Gaussian kernel), and high-pass filtering (cut-off frequency 1/128 Hz). See~\cite{strother2006evaluating} for more details. 

 
We construct a $\text{voxel} \times \text{voxel}$ $10$-nearest neighbor graph using the nonlinear affinity $w_{ij}=\text{exp}(- \|z_i-z_j \|^2)$.
Figure \ref{fig:fmri_global_eigs} shows the 4 leading non-trivial global eigenvectors projected onto a sliced brain. Note that the first slice (top left) in such an image corresponds to the bottom of the brain, whereas the last slice (bottom right) corresponds to the top of the brain. 
The non-trivial global eigenvectors aim to span the most dominant sources of variation in the data, which in this particular dataset appears to stem mainly from the primary visual cortex (V1), and a frontal/posterior contrast apparent in the second global eigenvector. Importantly, the global eigenvectors are typically not associated with the interesting features of the task but rather general signal variation, which may be due to visual presentation of the stimuli (visual cortex) and often physiological noise sources typically dominant in the lower slices of the brain near large arteries.

Using a probabilistic functional atlas created by averaging across multiple subjects~\cite{Eickhoff2005}, we carry out two experiments based on semi-supervised eigenvectors. 
Specifically, we construct  semi-supervised eigenvectors seeded in Primary Motor Cortex (PMC), known to be highly involved in the subject task~\cite{geyer1996two}, as well as  
semi-supervised eigenvectors seeded in Primary Auditory Cortex (PAC), that is not expected to carry much signal with respect to our target variable~\cite{morosan2001}.
The seed regions are highlighted in Figure \ref{fig:fmriaccuracyPMC} and \ref{fig:fmriaccuracyPAC}.

Figure \ref{fig:fmri_sseigs_pmc} and \ref{fig:fmri_sseigs_pac} shows respectively the leading 4 semi-supervised eigenvectors, each having a correlation of $0.25$ with the seed, and respectively seeded in PMC and PAC. As expected the semi-supervised eigenvectors are dominant near the seed region but are able to spread to related regions which carry information about important signal variation. For the PAC seed the first eigenvector appears to capture the general pattern of signal variation in part of the cortex that focus on auditory processing. The remaining three eigenvectors appear to span  specific signal variations in the PAC that are more specific to subregions with the auditory cortex.

Likewise the first semi-supervised eigenvector from the seed in the PMC reveals other dominant parts of the motor network including the remaining parts of the PMC (posterior part of Brodmann area 4), somatosensory cortex (Brodmann areas 1,2 and 3) and the premotor cortex (Brodmann area 5). The remaining semi-supervised eigenvectors again focus on more localized sources of signal within these areas as well as signal variation in the primary visual cortex (Brodmann area 17), which is to be expected as the visual presentation of stimuli is related to motor function in the present task.

For comparison in our classification task, we consider the leading global eigenvectors of the graph Laplacian, as well as simply extracting the time series as specified by the seed regions. For all of the considered feature extraction approaches we use either the projected or extracted time series as data for a linear SVM that is responsible for the downstream classification task.
Figure \ref{fig:fmriaccuracyA} summarizes the classification accuracies obtained by performing leave-one-out cross validation as a function of the number of components. For each semi-supervised eigenvector we fix $\kappa=\frac{1}{k}$ where $k$ is the number of components. Hence, for two components, each correlates $0.5$ with the seed, and so forth.
In the same plot, the dashed blue line corresponds to classifying the brain state using only voxel time series in the region as defined by PAC. Unsurprisingly, for the dashed green line, corresponding to PMC, it is evident that the primary motor cortex is a much better proxy for predicting motor responses. Due to inter-subject variability there is no guarantee that the rigid body realignment will align the seed perfectly with the physical region, which explains why the data-driven global eigenvectors are able to yield an even higher accuracy than the PAC time series. 
Also seen is the ``bump'' in classification accuracy for the global eigenvectors, when we reach 4-5 components. Thus, for this particular dataset, relevant parts of the are signal are captured in this regime. 
 
In the regime of few semi-supervised eigenvectors, the solutions are too localized to explain relevant local heterogeneities both near and within the seed set. As we increase the number of components they become less localized, and the semi-supervised eigenvectors seeded in PMC eventually surpasses the accuracy of global approach. As we consider more and more components, while distributing the correlation evenly across the semi-supervised eigenvectors, they will eventually converge to the global eigenvectors. Complementary, in the limit of a single component, the projection onto the leading trivial global eigenvector will simply correspond to the average time series, whereas for a leading semi-supervised eigenvector the solution is simply the seed itself, \emph{i.e.}, the projection onto this corresponds to a weighted average in the region defined by the seed. Hence, as seen in Figure \ref{fig:fmriaccuracyA} there exists a regime in which the semi-supervised approach performs better as we are able to pickup the relevant local heterogeneities at that particular scale, given that the seed is relevant with respect to the subject task.

\subsection{Large-scale Network Data}
\label{sec:largescalenetworkdata}
The final datasets we consider are from a collection of large sparse networks~\cite{BCSU3,BRSLLP,BoVWFI}. On these data, we demonstrate that the Push-peeling Heuristic introduced in Section \ref{sec:reid} is attractive due to an improved running time, as compared to solving a system of linear equations. Moreover, we also show that the ability to obtain multiple semi-supervised eigenvectors depends on the degree heterogeneity near the seed. Finally, we empirically evaluate the influence of the $\epsilon$ parameter of the Push algorithm that implicitly determines how many nodes the algorithm will touch. This parameter can be interpreted as a regularization parameter (different from $\gamma$ parameter), and setting it too large means we fail to distribute mass in the network, so that a few semi-supervised eigenvectors will consume all of the correlation. In particular, this behavior was investigated on the MNIST digits in Section \ref{sec:mnistpeeling}.
The basic properties for the networks considered in this section are shown in Table \ref{tab:networks}.

We start by considering the moderately sized networks from the DIMACS implementation challenge, as these networks are commonly used for the purpose of measuring realistic algorithm performance. Figure \ref{fig:DIMACS} shows analysis results for 6 networks from this collection, where we evaluate the performance and feasibility of the Push algorithm for approximating the leading semi-supervised eigenvector. 

 \begin{table*}[ht]
\centering
\begin{tabular}{lrr}
\hline
Network name & Number of nodes & Number of edges\\ 
\hline 
DIMACS10/de2010 & 24,115 & 116,056 \\
DIMACS10/ct2010 & 67,578 & 336,352\\
DIMACS10/il2010 & 451,554 & 2,164,464\\
DIMACS10/smallworld & 100,000 & 999,996 \\
DIMACS10/333SP & 3,712,815 & 22,217,266 \\
DIMACS10/AS365 & 3,799,275 & 22,736,152\\
LAW/arabic-2005 & 22,744,080 & 1,107,806,146 \\
LAW/indochina-2004 & 7,414,866 & 301,969,638 \\
LAW/it-2004 & 41,291,594 & 2,054,949,894 \\
LAW/sk-2005 & 50,636,154 & 3,620,126,660 \\
LAW/uk-2002 & 18,520,486 & 523,574,516 \\
LAW/uk-2005 & 39,459,925 & 1,566,054,250 \\
\hline
\end{tabular} 
\label{tab:auc_irm}
\caption{Summary of the networks considered in this section. Some of these networks are directed and have been symmetrized for the purpose of this analysis, \emph{i.e.}, the number edges in this table refer to the number of edges in the undirected graph.}
\label{tab:networks}
\end{table*} 

As stated in Section \ref{sec:discussion}, diffusion based procedures such as the Push algorithm can be used to solve our objective for $\gamma<0$. The impact of the reduced search range is that such procedures may not be able to produce a uniform correlation distribution for a set of semi-supervised eigenvectors. Hence, the leading solution(s) will instead pickup too much correlation, because sufficient mass cannot to diffuse away from the seed set.
However, the effect of a non-uniform correlation distribution was analyzed on the MNIST data in Section \ref{sxn:empirical-digits}, where we found that the performance of a downstream classifier is fairly robust to such non-uniformities, as seen by the simplex in Figure \ref{fig:dirichlet_kappa}.
Consequently, we emphasize that in a large-scale setting such side effects of diffusion based procedures is offset by the advantage of a greatly improved time complexity as compared to solving the system of linear equations, that implicitly touch every node.

\begin{figure*}[hbt!]
\subfigure[]{
\begin{minipage}[b]{0.31\linewidth}
\centering
\includegraphics[scale=0.38]{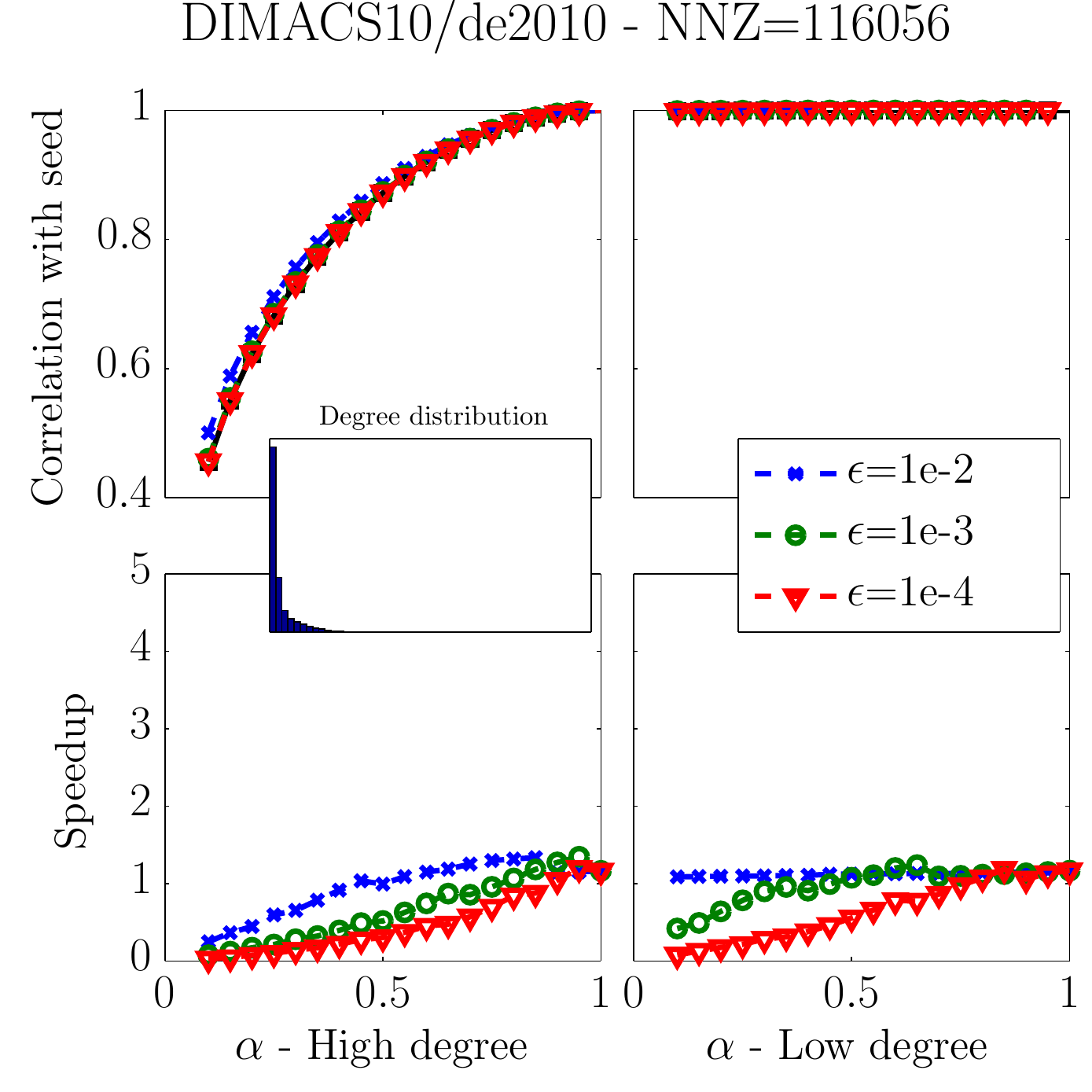}
\end{minipage}
\label{fig:DIMACSA}
}
\subfigure[]{
\begin{minipage}[b]{0.31\linewidth}
\centering
\includegraphics[scale=0.38]{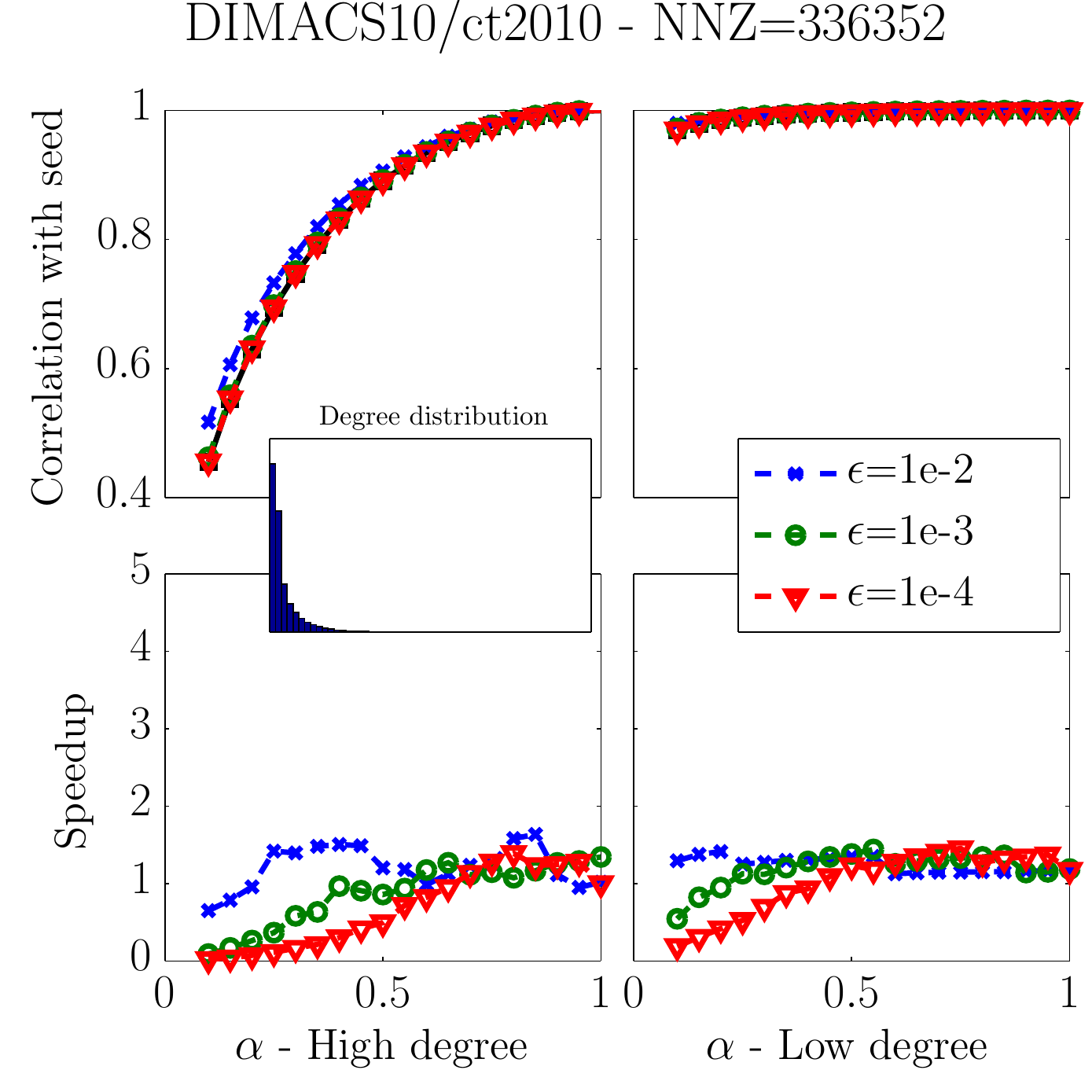}
\end{minipage}
\label{fig:DIMACSB}
}
\subfigure[]{
\begin{minipage}[b]{0.31\linewidth}
\centering
\includegraphics[scale=0.38]{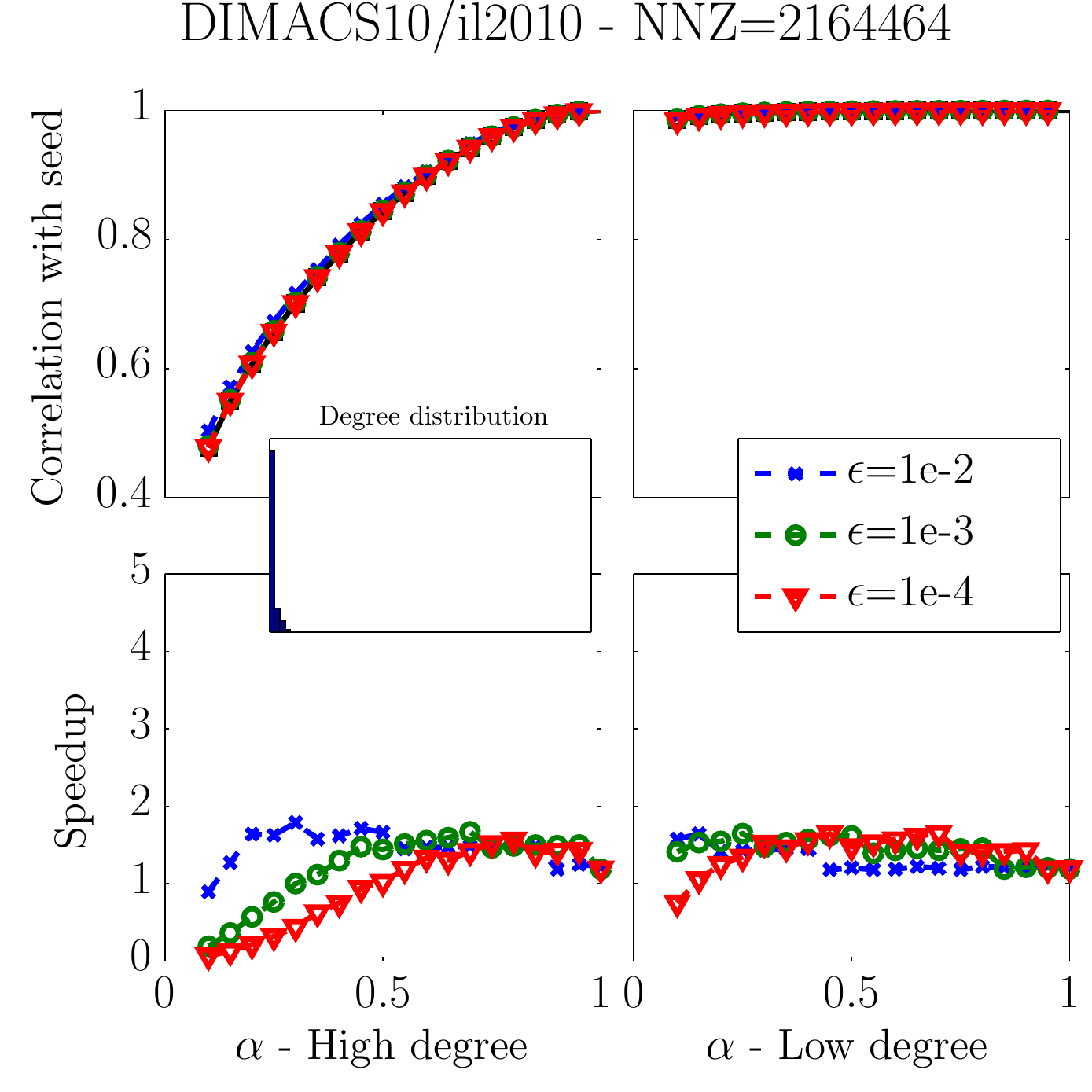}
\end{minipage}
\label{fig:DIMACSC}
}
\subfigure[]{
\begin{minipage}[b]{0.31\linewidth}
\centering
\includegraphics[scale=0.38]{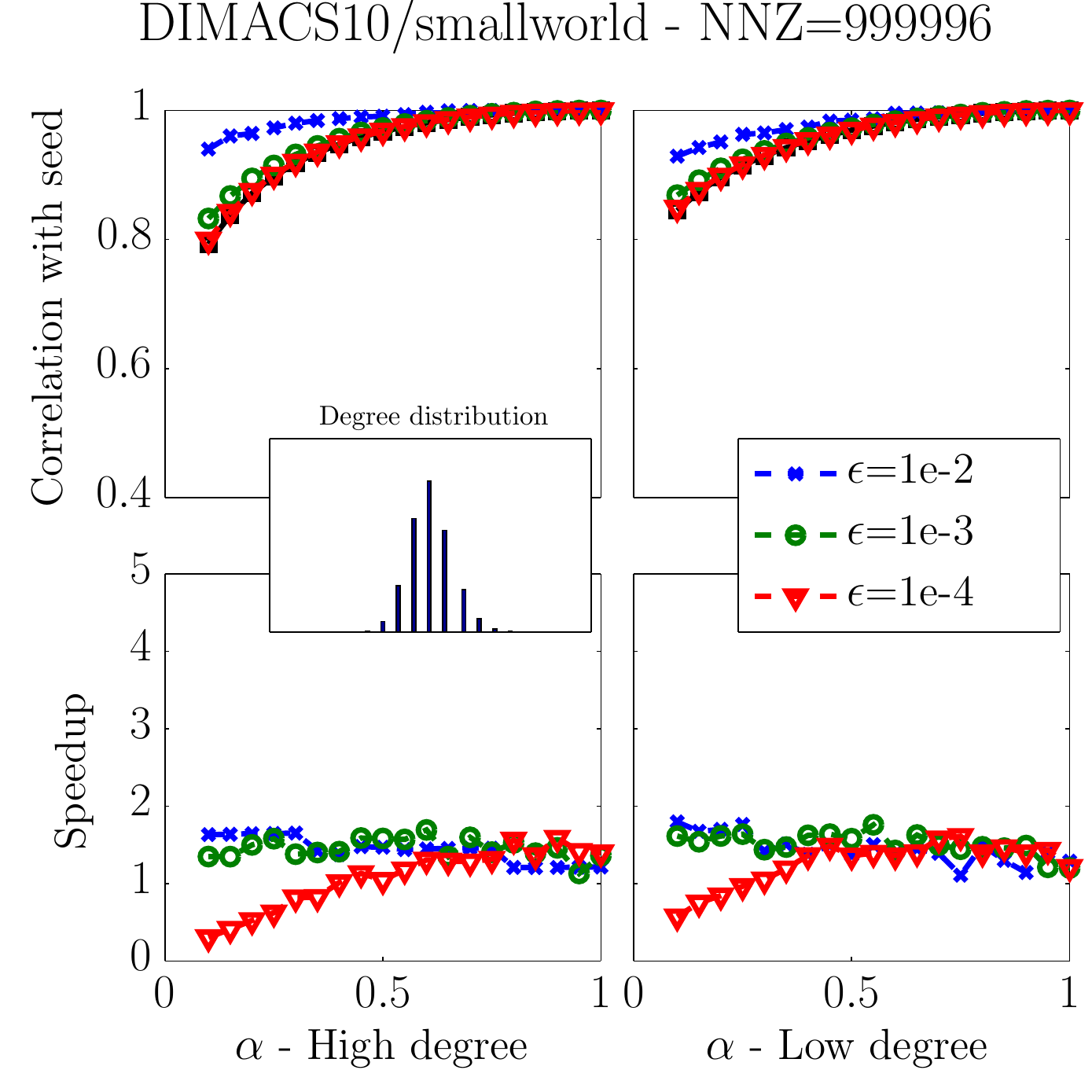}
\end{minipage}
\label{fig:DIMACSD}
}
\subfigure[]{
\begin{minipage}[b]{0.31\linewidth}
\centering
\includegraphics[scale=0.38]{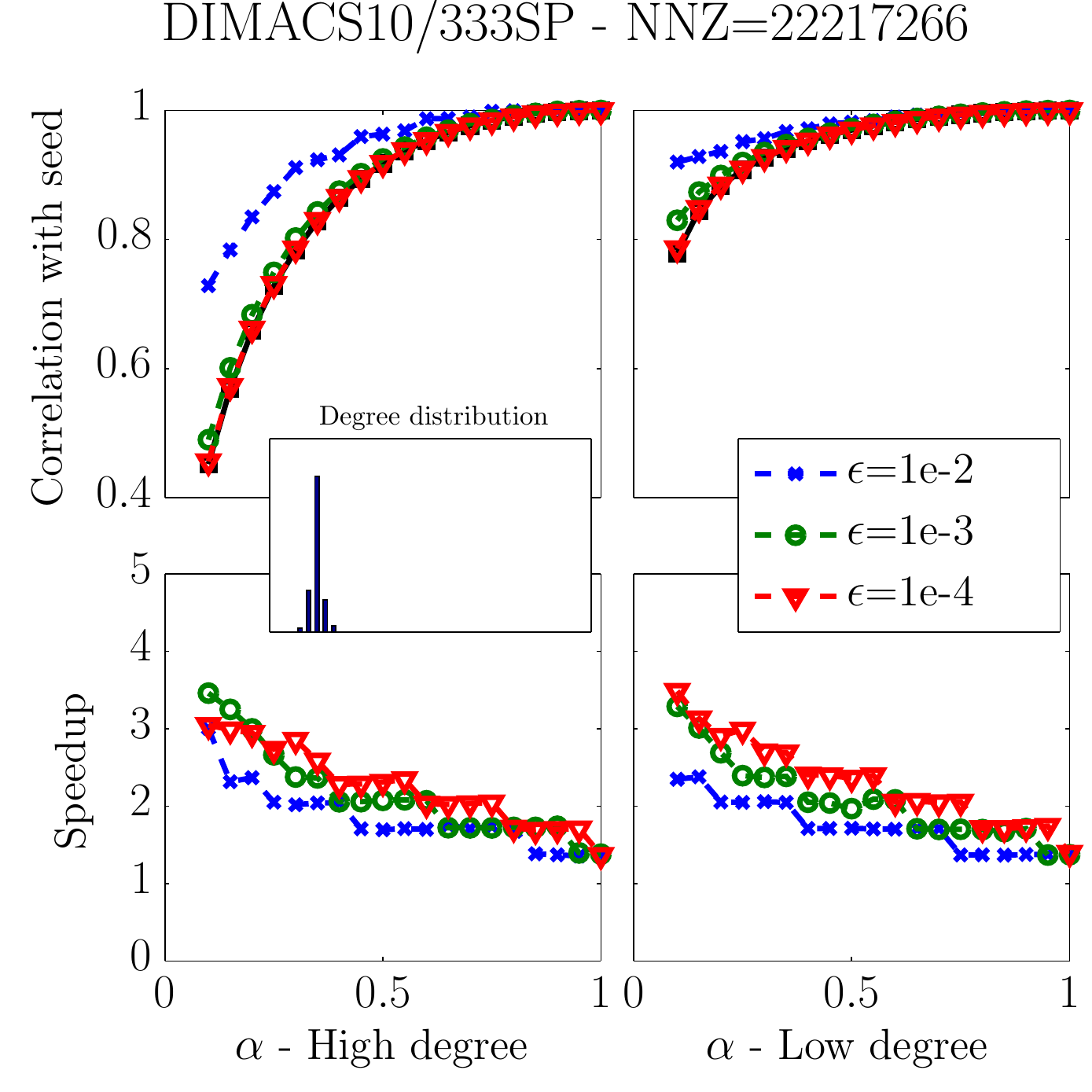}
\end{minipage}
\label{fig:DIMACSE}
}
\subfigure[]{
\begin{minipage}[b]{0.31\linewidth}
\centering
\includegraphics[scale=0.38]{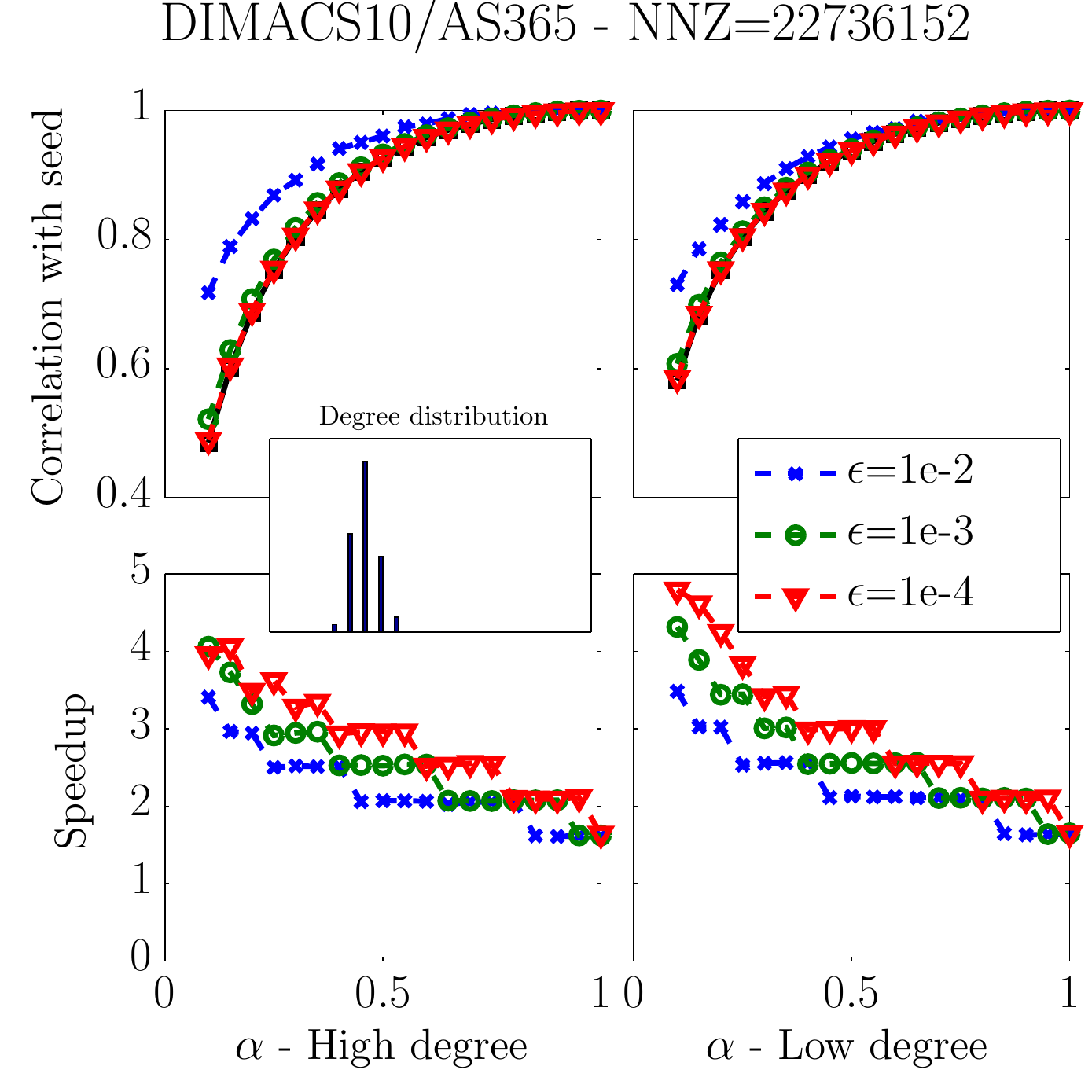}
\end{minipage}
\label{fig:DIMACSF}
}
\caption{For each network the first row depicts how the correlation decays as $\alpha$ tends towards 0, whereas the bottom row shows the speedup relative to the standard approach using conjugate gradient with a tolerance of $1\text{e-6}$, that is the default approach in our software distribution. Besides the three considered values of $\epsilon$ the correlation plots also illustrate the decay based on conjugate gradient (black curve), however this may be difficult to see, as the Push algorithm for $\epsilon=1\text{e-4}$ coincides with that solution.
Finally, seeds based on a high degree and low degree node are presented in respectively the first and last column, and the degree distribution for the network is visualized in a minor overlapping plot. 
}
\label{fig:DIMACS}
\end{figure*}

For each of the 6 analyzed networks in Figure \ref{fig:DIMACS}, we run two experiments considering different seeds, using respectively a high degree and low degree single seed node.
Figure \ref{fig:DIMACSA}-\ref{fig:DIMACSC} considers census block networks characterized by heavy-tailed degree distributions, whereas Figure \ref{fig:DIMACSD}-\ref{fig:DIMACSF} considers more densely connected synthetic networks.  For each of these 6 networks the speedup is measured by comparing with a standard conjugate gradient implementation using a tolerance of $1\text{e-6}$, and we stress that this tolerance cannot be directly compared with $\epsilon$ in the Push algorithm. Moreover, we test three different settings of the $\epsilon$ parameter, and we emphasize that for $\epsilon=1\text{e-4}$, the Push algorithm produces a similar result as the conjugate gradient algorithm. In Figure \ref{fig:DIMACS} this can be seen by the red curve ($\epsilon=1\text{e-4}$) in the correlation decay plots (see the figure caption) being on top of the black curve (conjugate gradient).

Common for Figure \ref{fig:DIMACSA}-\ref{fig:DIMACSC} are that low degree seed nodes yield very localized solutions for the entire range of $\alpha$, opposed to the high degree nodes that all succeed in gradually reducing the correlation when $\alpha$ is reduced. Also, the choice of $\epsilon$ is obviously very important, \emph{i.e.}, choosing it too large results in a solution that correlates too much with the seed, whereas choosing it too small means that we will be touching more nodes than necessary, resulting in a performance penalty. In general the networks analyzed in Figure \ref{fig:DIMACSA}-\ref{fig:DIMACSC} are too small to yield significant performance improvements over the conjugate gradient algorithm, and the Push algorithm is only competitive for large values of $\alpha$.

For the network in Figure \ref{fig:DIMACSD}, we see similar performance characteristics as the networks analyzed in Figure \ref{fig:DIMACSA}-\ref{fig:DIMACSC} due to its small size. However, the two final networks analyzed in Figure \ref{fig:DIMACSE}-\ref{fig:DIMACSF} share similar characteristics in terms of the degree distribution, but due to a much larger size they show significant performance improvements over the conjugate gradient algorithm. Interestingly, the Push algorithm instantiated with $\epsilon=1\text{e-4}$ yields a greater speedup in some settings, which may be explained by faster convergence, caused by a reduced threshold for distributing mass. Hence, the running time of the Push algorithm may not always decrease monotonically as $\epsilon$ increases.


In general it seems that seeding in a sparsely connected region of a network results in a solution having a large correlation with the seed for most values of $\alpha$. This is obviously a limiting factor if we are interested in using the peeling procedure to find multiple semi-supervised eigenvectors in that particular region. However, for large networks and more densely connected regions the benefit of the Push algorithm is immediate. 

\begin{figure*}[hbt!]
\subfigure[]{
\begin{minipage}[b]{0.31\linewidth}
\centering
\includegraphics[scale=0.38]{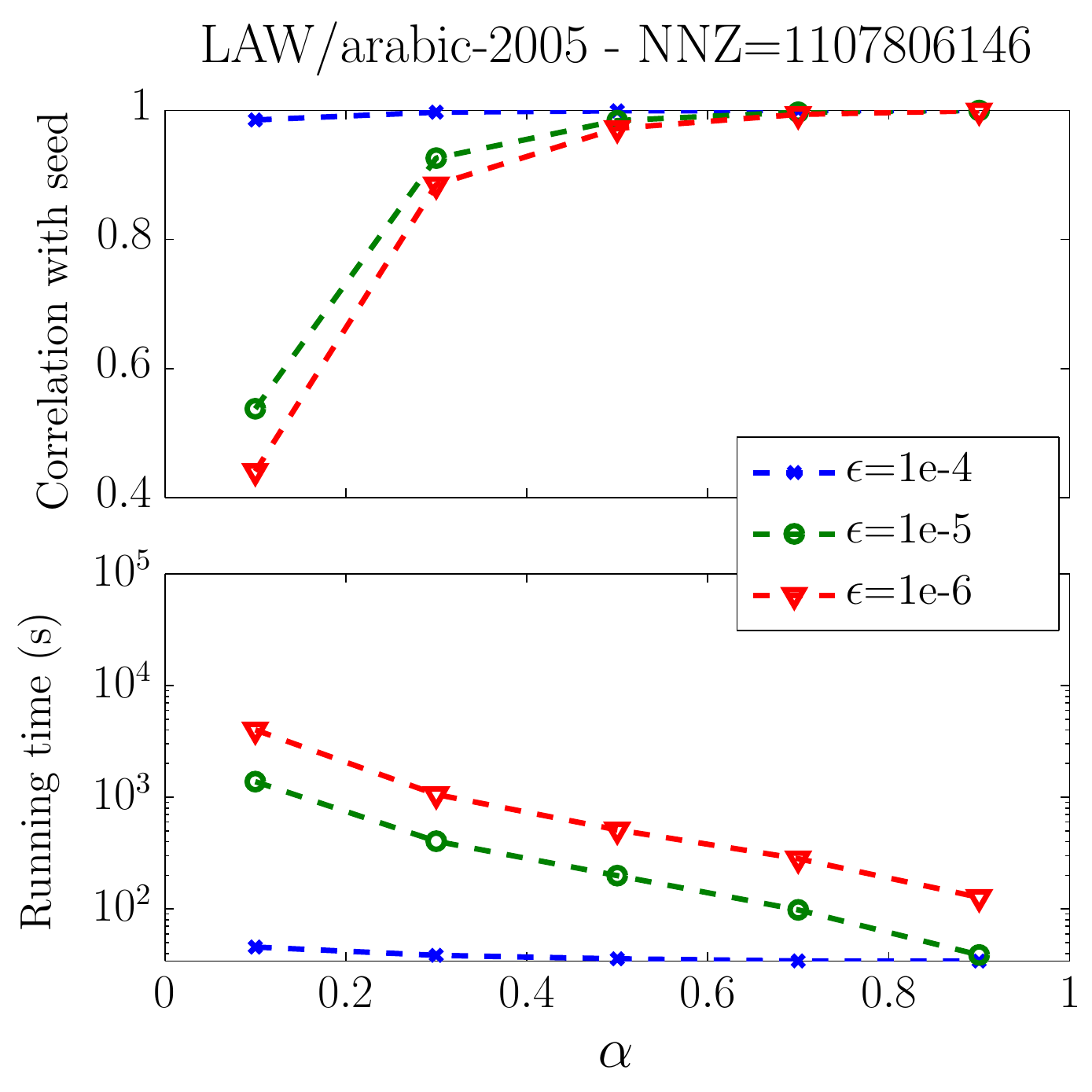}
\end{minipage}
\label{fig:LAWA}
}
\subfigure[]{
\begin{minipage}[b]{0.31\linewidth}
\centering
\includegraphics[scale=0.38]{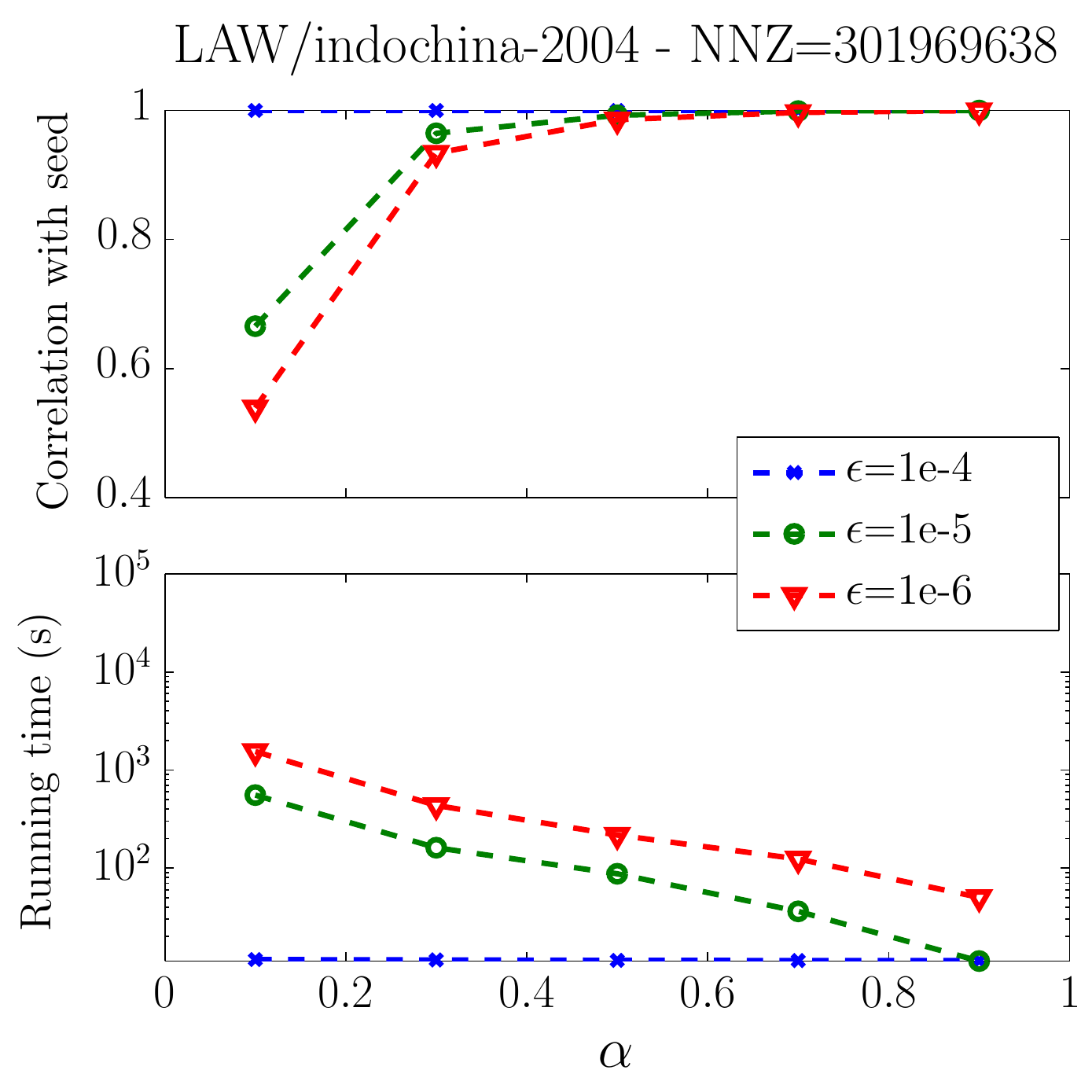}
\end{minipage}
\label{fig:LAWB}
}
\subfigure[]{
\begin{minipage}[b]{0.31\linewidth}
\centering
\includegraphics[scale=0.38]{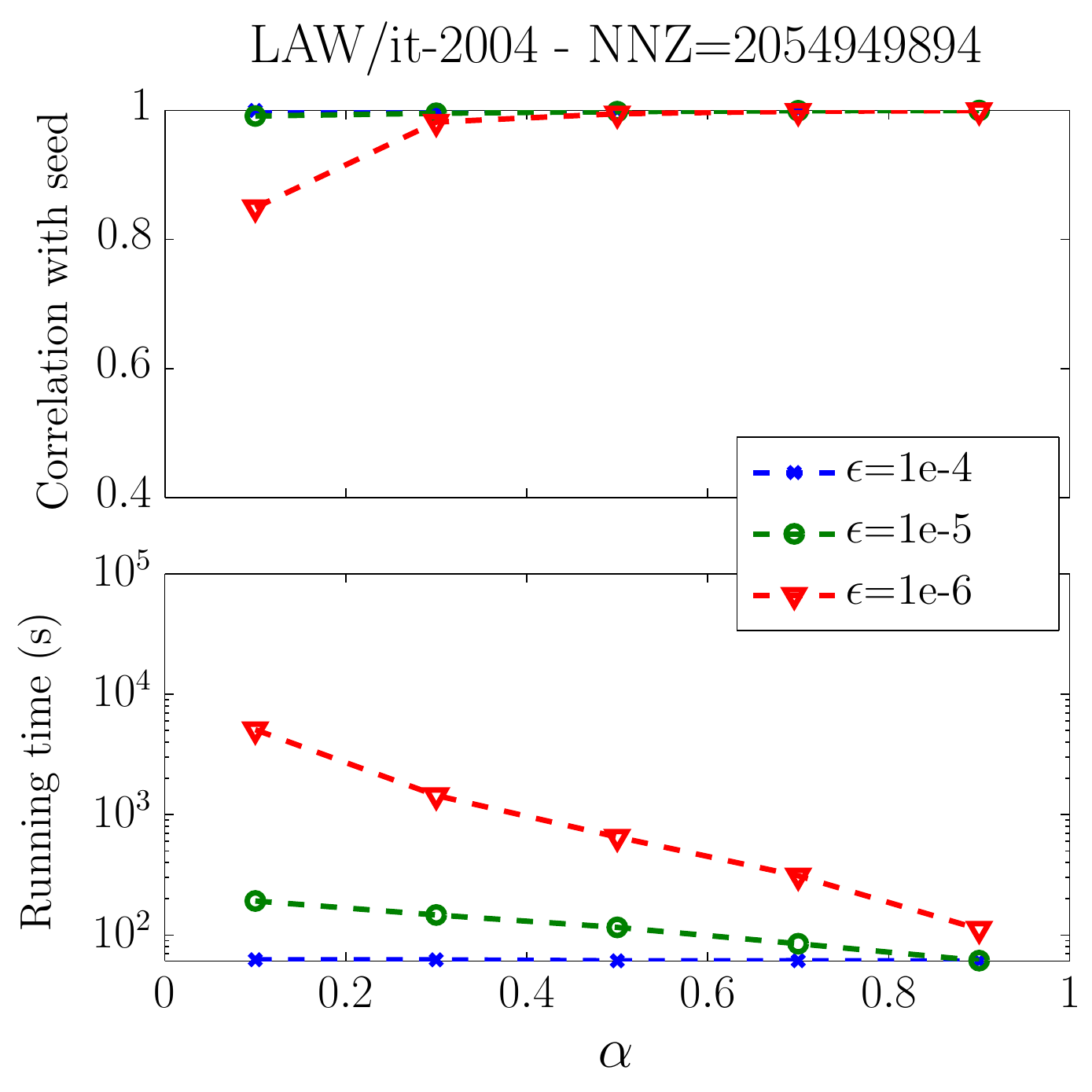}
\end{minipage}
\label{fig:LAWC}
}
\subfigure[]{
\begin{minipage}[b]{0.31\linewidth}
\centering
\includegraphics[scale=0.38]{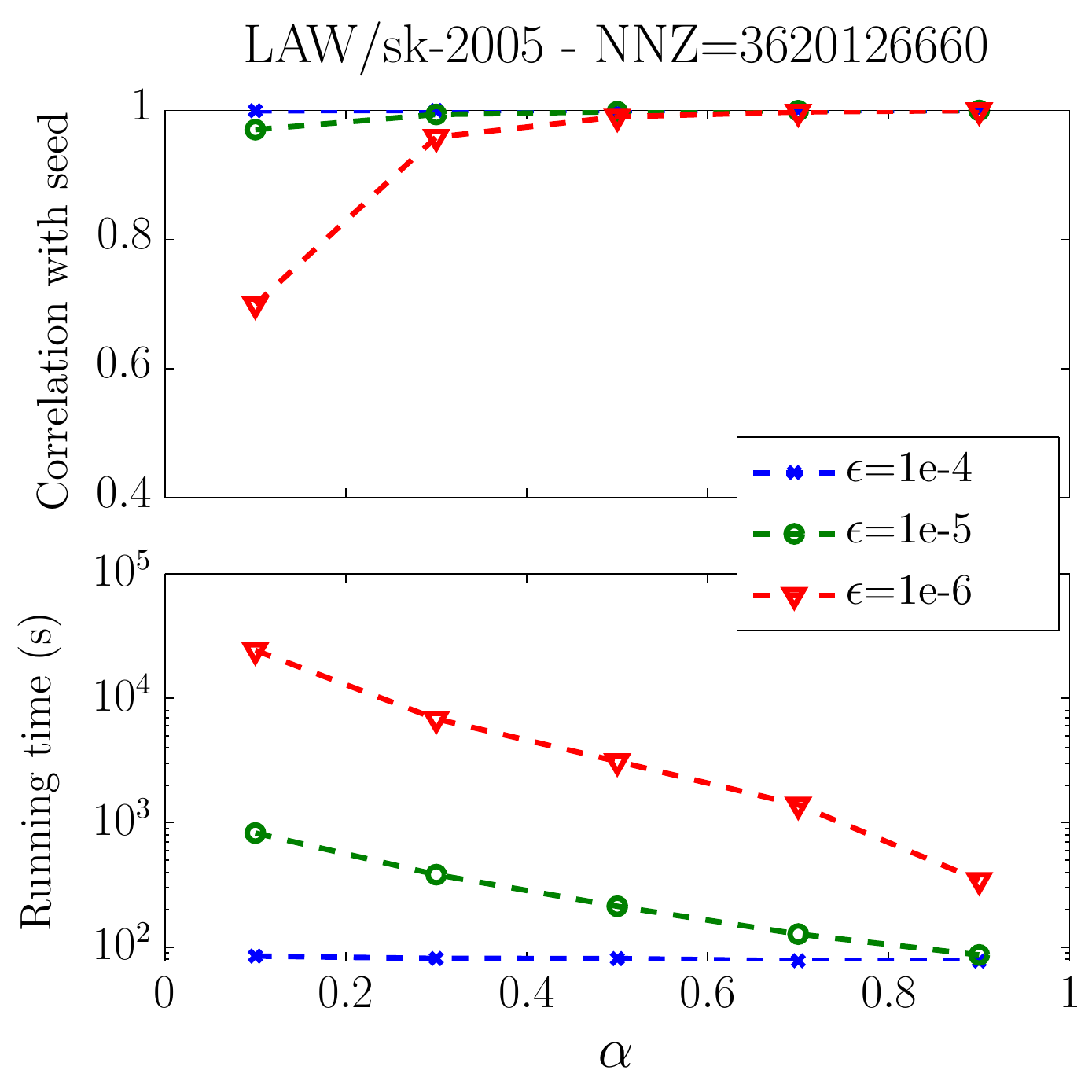}
\end{minipage}
\label{fig:LAWD}
}
\subfigure[]{
\begin{minipage}[b]{0.31\linewidth}
\centering
\includegraphics[scale=0.38]{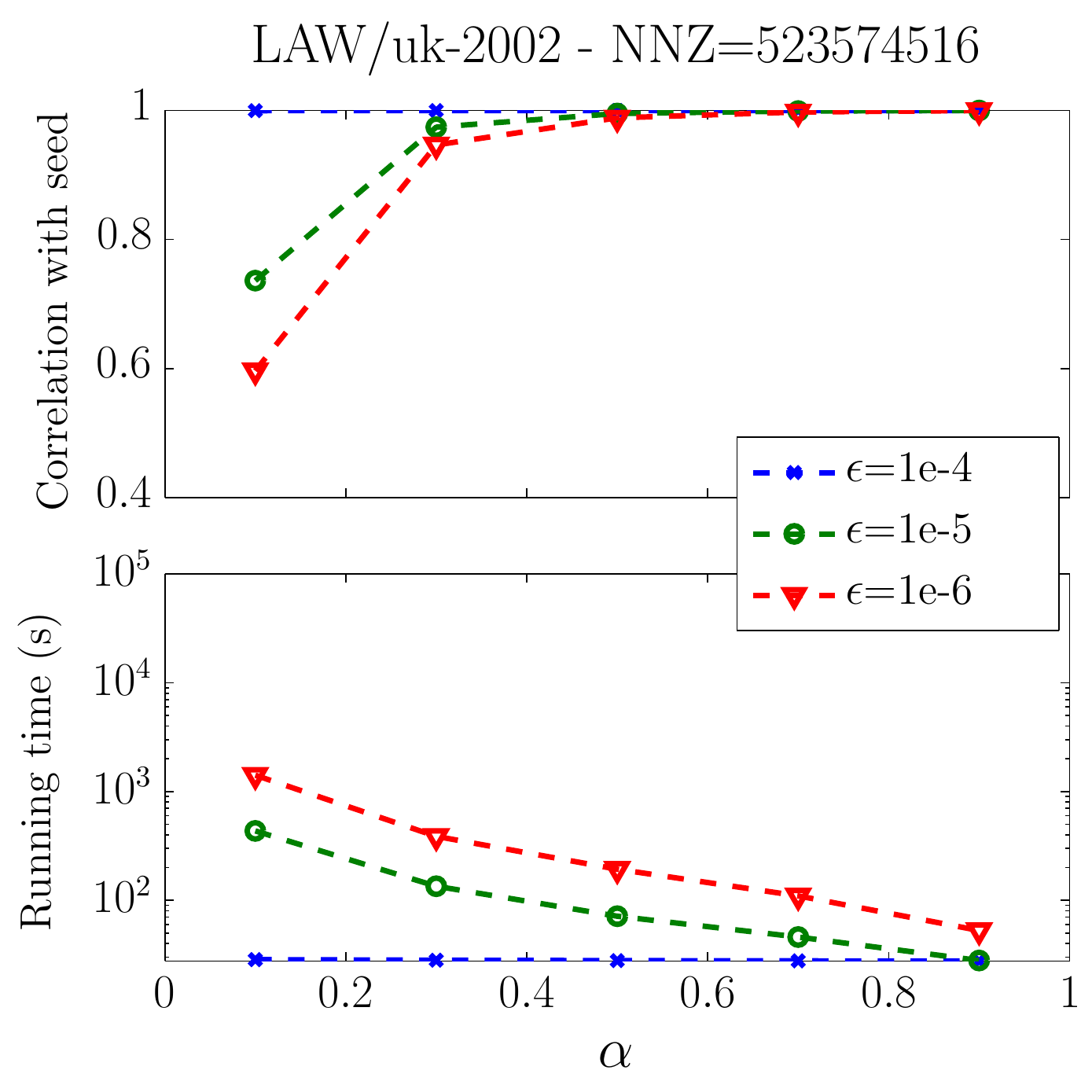}
\end{minipage}
\label{fig:LAWE}
}
\subfigure[]{
\begin{minipage}[b]{0.31\linewidth}
\centering
\includegraphics[scale=0.38]{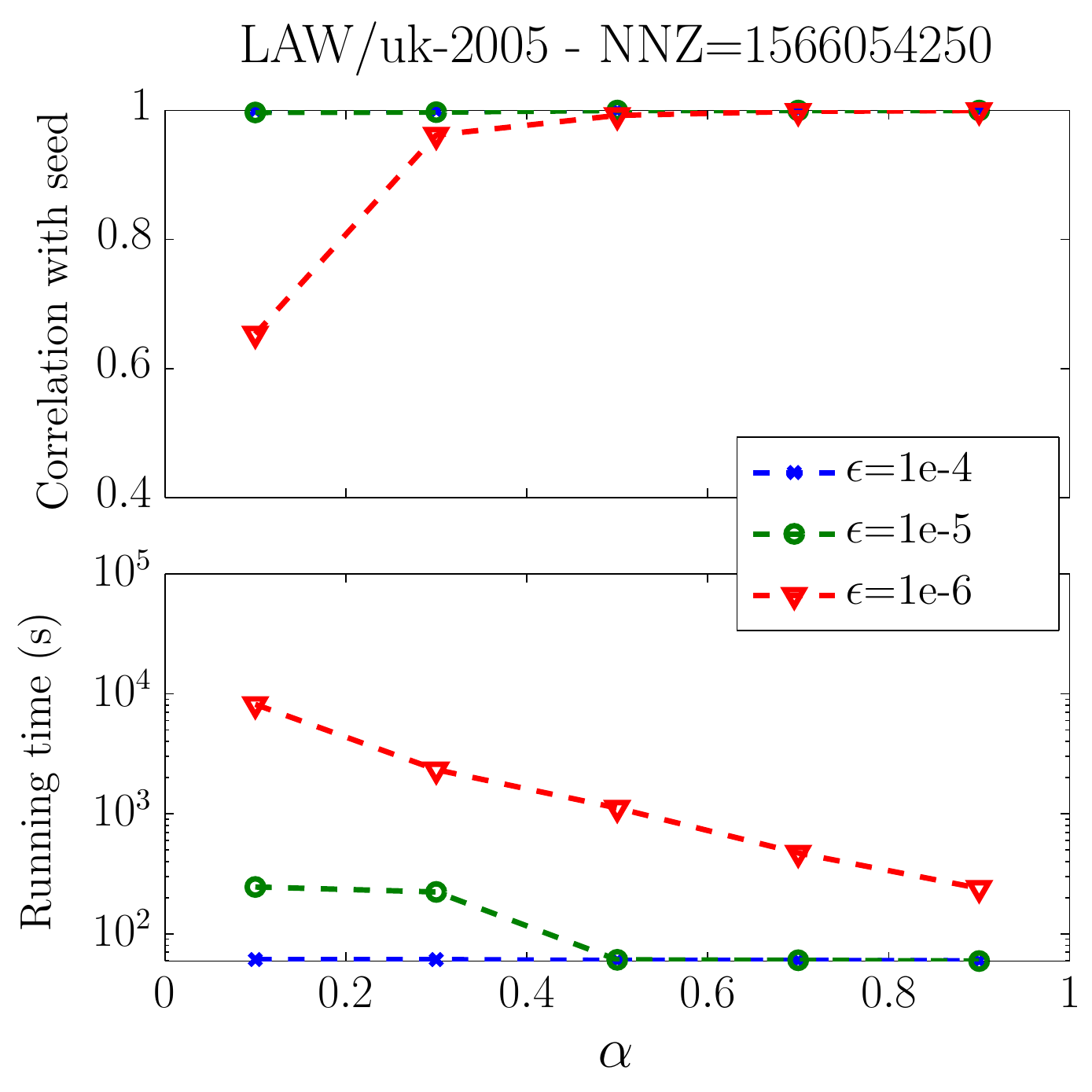}
\end{minipage}
\label{fig:LAWF}
}
\caption{Visualizes results for applying the Push algorithm to 6 very large web-crawl networks. For all networks we seed in the node with the highest degree. The top plot in each subfigure shows the correlation decay as a function of $\alpha$, whereas in the bottom plot we resort to absolute timings as the conjugate gradient algorithm is not feasible in this setting, as opposed to showing speedups as in Figure \ref{fig:DIMACS}.}
\label{fig:LAW}
\end{figure*}

Finally, we scale up to demonstrate that we can adapt the notion of semi-supervised eigenvectors to large datasets, and we do so by analyzing 6 large web-crawl networks.  These networks are large enough that touching all nodes is infeasible, \emph{i.e.}, conjugate gradient is not a feasible option, so in Figure \ref{fig:LAW} we resort to absolute timings. 
For the analysis results shown in Figure \ref{fig:LAW}, we are solely interested in giving the reader some intuition about the running time in a large-scale setting, as well as an idea on how the parameters interplay. Hence, we only consider experiments where we seed in a high degree node, as these are likely yield the worst running times, but also succeed in reducing the correlation the most. This will make the peeling procedure described in Section \ref{sec:reid} applicable, allowing us to obtain multiple semi-supervised eigenvectors.
As seen for all networks analyzed in Figure \ref{fig:LAWA}-\ref{fig:LAWF} the solution is highly sensitive to the choice of $\epsilon$, but for all networks we are able to reduce the correlation when $\alpha$ tends towards $0$ in case of $\epsilon=1\text{e-6}$. We emphasize that the reason for $\epsilon$ being smaller for these experiments, as compared to the previous is that the seed is normalized to have unit norm, implicitly requiring a lower $\epsilon$ when the network increases in size. 

For diffusion based procedures to be useful with respect to the computation of semi-supervised eigenvectors, mass must be able ``bleed'' away from the seed set and into the surrounding network. Otherwise only few semi-supervised eigenvectors can be found as the leading solution(s) become too correlated with the seed set. For moderately sized problems conjugate gradient performs very well, but in a large-scale setting, as considered here, the presented approach proves very efficient, allowing us to compute approximations to semi-supervised eigenvectors in networks consuming more than 30GB of working memory.
Obtaining an improved understanding of how the method of semi-supervised 
eigenvectors can be used to perform common machine learning tasks on 
graphs of that size is an obvious direction raised by our work.

\section{Conclusion}
\label{sxn:conclusion}

We have introduced the concept of semi-supervised eigenvectors as local 
analogues of the global eigenvectors of a graph Laplacian that have proven 
so useful in a wide range of machine learning and data analysis applications.
These vectors are biased toward prespecified local regions of interest in a 
large data graph; and we have shown that since they inherit many of the nice 
properties of the usual global eigenvectors, except in a locally-biased 
context, they can be used to perform locally-biased machine learning.
The basic method is conceptually simple and involves solving a sequence of
linear equation problems; we have also presented several extensions of the 
basic method that have improved scaling properties; and we have illustrated
the behavior of the method.
Due to the speed, simplicity, stability, and intuitive appeal of the method, 
as well as the range of applications in which local regions of a large data
set are of interest, we expect that the method of semi-supervised 
eigenvectors can prove useful in a wide range of machine learning and data 
analysis applications.

\section*{Ackowledgements}
We acknowledge Kristoffer H. Madsen, researcher at the Danish Research Centre for Magnetic Resonance at the University Hospital in Hvidovre, for providing the analyzed fMRI data.


\bibliographystyle{plain}

%

\bibliography{paper}


\clearpage
\appendix

\section{Supplementary Proofs} \label{app:proofs}
\begin{claimenv}\label{proof1}
Given an SPSD matrix $ M$ and some vector $ x$ where $ x^\top  x=1$, it holds that
\begin{align}
\displaystyle \lim_{\omega \to \infty} \left ( M + \omega  x  x^\top \right )^+ = \left ( \left ( I -  x  x^\top \right)  M \left ( I -  x  x^\top \right ) \right )^+.\label{eq:equivalence}
\end{align}
\end{claimenv}
\begin{proof}
Prior to applying the pseudo inverse, $x$ is clearly an eigenvector with eigenvalue $\lambda=0$ on the right hand side, and for left hand side $x$ is an  eigenvector with eigenvalue $\lambda=\infty$. Hence, without loss of generalizability we can decompose $ M =  \alpha  x  x^\top + X_\perp  \Lambda  X_\perp^\top$, where $ \Lambda$ is a diagonal matrix, such that $ M^+=\frac{1}{\alpha}  x  x^\top + X_\perp  \Lambda^+  X_\perp^\top$. First we consider the expansion of the left hand side of Eqn.~(\ref{eq:equivalence}) 
\begin{align*}
\displaystyle \lim_{\omega \to \infty} \left ( \left (\alpha+\omega \right )  x  x^\top + X_\perp  \Lambda  X_\perp^\top \right )^+ = \displaystyle \lim_{\omega \to \infty} \frac{1}{\alpha+\omega}  x  x^\top + X_\perp  \Lambda^+  X_\perp^\top
&=  X_\perp  \Lambda^+  X_\perp^\top.
\intertext{Similar, by expanding the right hand side we get}
 \left (  \left ( I -  x  x^\top \right) \left (\alpha  x  x^\top + X_\perp  \Lambda  X_\perp^\top \right ) \left ( I -  x  x^\top \right ) \right )^+ =  \left( X_\perp  \Lambda  X_\perp^\top \right )^+
 &=  X_\perp  \Lambda^+  X_\perp^\top. 
\end{align*}
\end{proof}

\begin{claimenv}\label{proof2}
For $ M'= M +\omega \sum_{i} x_i  x_i^\top $ where $\omega \geq 0$ it holds that $\lambda_k( M') \geq \lambda_k( M)$.
\end{claimenv}

\begin{proof}
All eigenvalues of the sum of rank-1 perturbations are non-negative 
\begin{align*}
\omega \sum_{i}  x_i  x_i^\top \succeq 0 \Rightarrow  M' \succeq  M. &
\end{align*}
\end{proof}

\begin{claimenv}\label{claim:corrsum}
Given an orthonormal basis, $ X = \left [ x_1,\ldots, x_{n-1} \right ]$, \emph{i.e.}, $ X^\top  D_G  X =  I$, and unit length seed $ s^\top  D_G  s = 1$. Then, any unit length vector $ x_n^\top  D_G  x_n=1$, perpendicular to the subspace $ X^\top  D_G  x_n =  0$, will have a correlation with the seed bounded by
\begin{align*}
0 \leq ( x_n^\top  D_G  s)^2 \leq 1-\sum_{i=1}^{n-1}( x_i^\top  D_G  s)^2.
\end{align*}
\end{claimenv}

\begin{proof}
The proof follows directly from the Pythagorean theorem. Let $ X = \left [ x_1,\ldots, x_N \right ]$ be the orthonormal basis of $\mathbb{R}^N$, \emph{i.e.}, spanning $ s$. Then
\begin{align*}
\sum_{i=1}^{N}( x_i^\top  D_G  s)^2 = ( s^\top  D_G  s)^2 = 1.
\end{align*}   
\end{proof}

\begin{claimenv}\label{claim:structural}
For the matrix $P_\gamma=\mathcal{L}_G-\gamma I$ it holds that 
\begin{align}
P_\gamma^+ -P_{\hat \gamma}^+  =   ( \gamma - \hat \gamma)  P_{\hat \gamma}^+ P_\gamma^+,\label{eq:identity1}
\end{align} 
given that neither $\gamma$ nor $\hat \gamma$ coincides with an eigenvalue of $\mathcal{L}_G$.
\end{claimenv}
\begin{proof}
The proof follows directly by plain algebra. Simply substitute the SVD $P_\gamma=V\Lambda_\gamma V^T$, where $\Lambda_\gamma$ is a diagonal matrix with the eigenvalues shifted by $\gamma$, into Eqn. (\ref{eq:identity1})
  \begin{align*}
V\Lambda_\gamma^+V^T-V\Lambda_{\hat \gamma}^+V^T&=(\gamma - \hat \gamma)V\Lambda_{\hat \gamma}^+V^TV\Lambda_\gamma^+V^T\\
V\Lambda_\gamma^+V^T-V\Lambda_{\hat \gamma}^+V^T&=(\gamma - \hat \gamma)V\Lambda_{\hat \gamma}^+\Lambda_\gamma^+V^T\\
\Rightarrow \Lambda_\gamma^+-\Lambda_{\hat \gamma}^+&=(\gamma - \hat \gamma)\Lambda_{\hat \gamma}^+\Lambda_\gamma^+.
\intertext{The system is decoupled so it will be sufficient to consider a single eigenvalue}
\frac{1}{\lambda_i-\gamma}-\frac{1}{\lambda_i-\hat \gamma}&=\frac{\gamma -\hat \gamma}{(\lambda_i-\hat \gamma)(\lambda_i-\gamma)}\\
\frac{\lambda_i-\hat \gamma}{(\lambda_i-\hat \gamma)(\lambda_i-\gamma)}-\frac{\lambda_i-\gamma}{(\lambda_i-\hat \gamma)(\lambda_i-\gamma)}&=\frac{\gamma -\hat \gamma}{(\lambda_i-\hat \gamma)(\lambda_i-\gamma)}\\
\frac{\gamma-\hat \gamma}{(\lambda_i-\hat \gamma)(\lambda_i-\gamma)}&=\frac{\gamma -\hat \gamma}{(\lambda_i-\hat \gamma)(\lambda_i-\gamma)}.
\end{align*}
Also, this trivially holds for the rank deficient case, \emph{i.e.}, $0=0$.
\end{proof}

\begin{claimenv}\label{claim:peeling}
As pointed out in Section \ref{sec:discussion}, it is already immediate that the initial semi-supervised eigenvector can be computed using a diffusion-based procedure, such as the Push algorithm. However, from that discussion it remains unclear how the approach can be generalized for the consecutive $k-1$ semi-supervised eigenvectors. It turns out that the $k^{th}$ solution is approximated by
\begin{align}
x_k &\approx  c(I-XX^TD_G)(L_G-\gamma_k D_G)^+D_Gs, \label{eq:peeling}
\end{align}
given that $(L_G-\gamma_k D_G)^+D_Gs$ is linearly independent with respect to the previous $k-1$ solutions contained in $X$.
\end{claimenv}
\begin{proof}
By Eqn. (\ref{eq:lagrange_exact}) the solution for the second semi-supervised eigenvector can be expressed as
\newcommand{\PI}{{P_{\gamma_1}}}
\newcommand{\PII}{{P_{\gamma_2}}}
\newcommand{\rab}{{\rho_{\gamma_1 \gamma_2}}}
\newcommand{\raa}{{\rho_{\gamma_1 \gamma_1}}}
\begin{align}
y_2&=c \left (\PII^+- \PII^+y_1 (y_1^T\PII^+y_1)^+ y_1^T\PII^+ \right )D_G^{1/2}s, \notag
\intertext{where $(y_1^T\PII^+y_1)^+$ is a constant. For notational convenience we start by substituting $b=D_G^{1/2}s$ together with the explicit solution $y_1\propto  \PI^+b$}
y_2&=c \PII^+b- \frac{c\PII^+\PI^+b  b^T\PI^+ \PII^+ b}{b^T \PI^+\PII^+\PI^+b},\notag\\
\intertext{and for the same reason we also introduce $\rho_{\gamma_1 \gamma_2}=b^\top \PI^+   \PII^+b$}
y_2&=c  \PII^+b- \frac{c \rab\PII^+\PI^+b }{b^T \PI^+\PII^+\PI^+b}.\notag\\
 \intertext{We can approximate this expression by exploiting the structural result of Lemma \ref{claim:structural}, namely that $\PI^+ -\PII^+  =   ( \gamma_1 - \gamma_2)  \PII^+ \PI^+$ }
y_2 &\approx c  \PII^+b- \frac{c \rab(\PI^+-\PII^+)b }{b^T \PI^+( \PI^+-\PII^+)b}\notag\\
 &=c  \PII^+b- \frac{c \rab(\PI^+-\PII^+)b }{\raa-\rab}.\notag\\
\intertext{We emphasize that this approximation is exact whenever $\PI^+ -\PII^+$ is well-conditioned, and singular for $\gamma_1=\gamma_2$. Then, substitute $c=\frac{\raa-\rab}{\raa}$}
y_2 &\approx \frac{\raa \PII^+b-\rab \PII^+b}{\raa}- \frac{ \rab(\PI^+-\PII^+)b }{\raa}\notag\\
 &=\frac{\raa \PII^+b-\rab \PII^+b- \rab\PI^+b+\rab \PII^+b }{\raa}\notag\\
 &=\frac{\raa \PII^+b- \rab\PI^+b}{\raa}\notag\\
 &=\PII^+b-\frac{ \rab\PI^+b}{\raa}\notag.
 \intertext{By resubstituting for the auxiliary variables we obtain the desired result}
 y_2&\approx c(I-y_1y_1^T)(\mathcal{L}_G-\gamma I)^+D_G^{1/2}s,\notag
 \intertext{and by applying this procedure recursively it follows that}
 y_k&\approx c(I-YY^T)(\mathcal{L}_G-\gamma_k I)^+D_G^{1/2}s.\notag
\intertext{Finally, we can relate this result to the combinatorial graph Laplacian as follows}
 y_k&\approx c (I-D_G^{1/2}XX^TD_G^{1/2})D_G^{1/2}(L_G-\gamma_k D_G)^+D_Gs\notag\\
 &=c(D_G^{1/2}-D_G^{1/2}XX^TD_G)(L_G-\gamma_k D_G)^+D_Gs\notag\\
  &= cD_G^{1/2} (I-XX^TD_G)(L_G-\gamma_k D_G)^+D_Gs,\notag
\intertext{and due to the relationship $x_k=D_G^{-1/2}y_k$ it follows that}
 x_k &\approx  c (I-XX^TD_G)(L_G-\gamma_k D_G)^+D_Gs. \notag
 \end{align}
\end{proof}



\section{Derivation of sparse graph diffusions.}\label{app:sparsediffusion}
To allow efficient computation of semi-supervised eigenvectors by graph diffusions, we must make the relationship with the sparse seed vector explicit. Here we specifically consider the derivation of Eqn. (\ref{eq:pushfast}).
Given a sparse seed indicator $s_0$, we can write the seed vector $s$ as $s\propto D_G^{-1/2}(I-v_0v_0^T)s_0$, where $v_0\propto \text{diag}(D^{1/2})$ is the leading eigenvector of the normalized graph Laplacian (corresponding to the all-one vector of the combinatorial graph Laplacian).
Using this explicit form of $s$ we can rewrite the leading solution as
\begin{align*}
x_1&=c (L_G-\gamma D_G)^+D_Gs\\
&= cD_G^{-1/2}(\mathcal{L}_G-\gamma I)^+D_G^{1/2}s\\
&= cD_G^{-1/2}(\mathcal{L}_G-\gamma I)^+D_G^{1/2} D_G^{-1/2}(I-v_0v_0^T)s_0 \\
&= cD_G^{-1/2}\left ( (\mathcal{L}_G-\gamma I)^+s_0 - (\mathcal{L}_G-\gamma I)^+ v_0v_0^Ts_0 \right).
\intertext{Since $\mathcal{L}_G-\gamma I$ simply shifts the eigenvalues of $\mathcal{L}_G$ by $-\gamma$, the latter term simplifies to}
x_1&= cD_G^{-1/2}\left ( (\mathcal{L}_G-\gamma I)^+s_0 - \left (\frac{1}{-\gamma}v_0v_0^T \right) v_0v_0^Ts_0 \right)\\
&= cD_G^{-1/2}\left ( (\mathcal{L}_G-\gamma I)^+s_0 +\frac{1}{\gamma}v_0 v_0^Ts_0 \right)\\ 
&= cD_G^{-1/2}\left ( \frac{1}{-\gamma} D_G^{-1/2}  \text{pr}_\epsilon \left(\frac{\gamma}{\gamma-2},D_G^{1/2}s_0 \right)   +\frac{1}{\gamma}v_0 v_0^Ts_0 \right ).
\end{align*}
Finally, by exploiting the peeling result in Eqn. (\ref{eq:peeling}), we can use the Push algorithm to approximate the sequence of semi-supervised eigenvectors in an extremely efficient manner
\begin{align*}
x_t^{*}&\approx c \left (I-XX^TD_G \right )\left ( D_G^{-1}    \text{pr}_\epsilon \left(\frac{\gamma_t}{\gamma_t-2},D_G^{1/2}s_0 \right)  - D_G^{-1/2} v_0 v_0^Ts_0 \right ),
\end{align*}
as the Push algorithm is only applied on the sparse seed set.

\section{Nystr{\"{o}}m Approximation for the Normalized Graph Laplacian}\label{app:nystrom}
The vanilla procedure is as follows; we choose $m$ samples at random from the full data set, and for notational simplicity we reorder the samples so that these $m$ samples are followed by the remaining $n=N-m$ samples, \emph{i.e.}, we can partition the adjacency matrix as
\begin{align*}
A_G=\left( \begin{array}{cc}
A & B \\
B^T & C \end{array} \right),
\end{align*}
where $A \in \mathbb{R}^{m \times m}$, $B \in \mathbb{R}^{m \times n}$, and $C\in \mathbb{R}^{n\times n}$, with $N=m+n$ and $m\ll n$. The Nystr{\"{o}}m extension then approximates the huge $C$ matrix in terms of $A$ and $B$, so the resulting approximation to weight matrix becomes
\begin{align*}
A_G \approx \hat A_G = \left( \begin{array}{cc}
A & B \\
B^T & B^TA^{-1} B \end{array} \right).
\end{align*}
Hence, rather than encoding only each nodes k-nearest-neighbors into the weight matrix, the Nystr{\"{o}}m methods provides a low-rank approximation to the entire dense weight matrix. Since the leading eigenvectors of $D_G^{-1/2}A_GD_G^{-1/2}$ correspond to the smallest of $\mathcal{L}_G$, our goal is to diagonalize $D_G^{-1/2}A_GD_G^{-1/2}$.
At the risk of washing out the local hetrogeneties the Nystr{\"{o}}m procedure approximates the largest eigenvectors of $D_G^{-1/2}A_GD_G^{-1/2}$ using the normalized matrices $\tilde A$ and $\tilde B$
\begin{align*}
\tilde A_{ij}&=\frac{A_{ij}}{\sqrt{\hat d_i \hat d_j}}, \quad i,j=1,\ldots m\\
\tilde B_{ij}&=\frac{B_{ij}}{\sqrt{\hat d_i \hat d_{j+m}}}, \quad i=1,\ldots m,\quad j=1,\ldots,n.
\end{align*}
Finally, let $U \Lambda U^T$ be the SVD of $\tilde A + \tilde A^{-1/2} \tilde B \tilde B^T \tilde A^{-1/2}$, then the $m$ leading eigenvectors are approximated by
\begin{align*}
V=\left( \begin{array}{c}
\tilde A \\
\tilde B^T \end{array} \right) \tilde A^{-1/2}U \Lambda^{-1/2},
\end{align*}
and the normalized graph Laplacian by $\mathcal{L}_G \approx I - V \Lambda V^T$.

\end{document}